\documentclass{amsart}
\usepackage{amsmath,amssymb,latexsym, geometry}
\usepackage{enumerate}
\usepackage{enumitem}
\usepackage{graphicx}
\usepackage[ruled,vlined,linesnumbered]{algorithm2e}
\usepackage{comment}
\usepackage{color}
\usepackage{bbm}
\usepackage{algorithmic} 
\usepackage{pgfplotstable}
\usepackage{booktabs}
\usepackage{array}
\usepackage{colortbl}
\usepackage{mathtools}
\usepackage{figsize}
\usepackage{subfigure}
\usepackage{multirow}
\usepackage{mathptmx}
\usepackage{url,bm,overpic}

\newtheorem{assumption}{Assumption}[section]
\newtheorem{theorem}{Theorem}[section]
\newtheorem{proposition}{Proposition}[section]
\newtheorem{definition}{Definition}[section]
\newtheorem{lemma}{Lemma}[section]

\newtheorem{remark}{Remark}[section]

\numberwithin{equation}{section}

\begin{document}

\title{Boundary detection algorithm inspired by locally linear embedding}

\author{Pei-Cheng Kuo}
\address{Department of Mathematics\\
National Taiwan University, Taipei, 10617, Taiwan}
\email{r12221017@ntu.edu.tw}

\author{Nan Wu}
\address{Department of Mathematical Sciences\\
The University of Texas at Dallas, Richardson, TX 75080, United States}
\email{nan.wu@utdallas.edu}

\begin{abstract}
In the study of high-dimensional data, it is often assumed that the data set possesses an underlying lower-dimensional structure. A practical model for this structure is an embedded compact manifold with boundary. Since the underlying manifold structure is typically unknown, identifying boundary points from the data distributed on the manifold is crucial for various applications.  In this work, we propose a method for detecting boundary points inspired by the widely used locally linear embedding algorithm. We implement this method using two nearest neighborhood search schemes: the $\epsilon$-radius ball scheme and the $K$-nearest neighbor scheme. This algorithm incorporates the geometric information of the data structure, particularly through its close relation with the local covariance matrix. We analyze the algorithm by exploring the spectral properties of the local covariance matrix, with the findings guiding the selection of key parameters.  In the presence of high-dimensional noise, we propose a framework aimed at enhancing boundary detection in noisy data. Furthermore, we demonstrate the algorithm’s performance with simulated examples.
\end{abstract}

\keywords
{Manifold Learning;  Manifold with boundary;  Boundary detection;  Local covariance matrix;  Locally linear embedding;  Nearest neighborhood search scheme.}

\maketitle

\section{Introduction}
In modern data analysis,  it is common to assume that high-dimensional data concentrates around a low-dimensional structure. A typical model for this low-dimensional structure is an unknown manifold, which motivates manifold learning techniques. However, many approaches in manifold learning assume the underlying manifold of the data to be closed (compact and without boundary), which does not always align with realistic scenarios where the manifold may have boundaries.  This work aims to address the identification of boundary points for data distributed on an unknown compact manifold with boundary.

Due to the distinct geometric properties of boundary points compared to interior points, boundary detection is crucial for developing non-parametric statistical methods (see \cite{ding2019bdrygp} for Gaussian process regression and see \cite{berry2017density} and Proposition \ref{strong uniform consistency of KDE} in the Supplementary Material for kernel density estimation) and dimension reduction techniques. Moreover, it plays an important role in kernel-based methods for approximating differential operators on manifolds under various boundary conditions (\cite{peoples2021spectral, jiang2023ghost, vaughn2024diffusion}). However, identifying boundary points on an unknown manifold presents significant challenges. Traditional methods for boundary detection may struggle due to the extrinsic geometric properties of the manifold and the non-uniform distribution of data. 

Locally Linear Embedding(LLE)\cite{Roweis_Saul:2000} is a widely applied nonlinear dimension reduction technique. In this work, we propose a \textbf{B}oundary \textbf{D}etection algorithm inspired by \textbf{L}ocally \textbf{L}inear \textbf{E}mbedding (BD-LLE). BD-LLE leverages barycentric coordinates within the {framework of LLE}, implemented via either an $\epsilon$-radius ball scheme or a $K$-nearest neighbor (KNN) scheme. It incorporates {the geometry of the manifold} through its relation with the local covariance matrix. Across sampled data points, BD-LLE approximates a bump function that concentrates at the boundary of the manifold, exhibiting a constant value on the boundary and zero within the interior. This characteristic remains consistent regardless of extrinsic curvature and data distribution. The clear distinction in bump function values between the boundary and the interior facilitates straightforward identification of points in a small neighborhood of the boundary by applying a simple threshold. Particularly in the $\epsilon$-radius scheme, BD-LLE identifies points within a narrow, uniform collar region of the boundary. { In practice, data is often contaminated by high-dimensional noise. We propose a framework that combines BD-LLE with Diffusion Maps \cite{coifman2006diffusion} to enhance boundary detection in noisy data.}

We outline our theoretical contributions in this work. {We present the bias and variance analyses of the local covariance matrix constructed under the KNN scheme for data sampled from manifolds with boundary.  Leveraging the spectral properties of the local covariance matrix, we provide an analysis of the BD-LLE algorithm in both the $\epsilon$-radius ball scheme and the KNN scheme, offering guidance on selecting key parameters for the algorithm.} Previous theoretical analyses of manifold learning algorithms have predominantly focused on the $\epsilon$-radius ball search scheme, with fewer results available for the KNN scheme. (Refer to \cite{calder2022improved, cheng2022convergence} for analyses of the Diffusion Maps (Graph Laplacians) on a closed manifold in the KNN scheme.) The framework developed in this paper provides useful tools for analyzing kernel-based manifold learning algorithms in the KNN scheme applicable to manifolds with  boundary. 

We provide a brief overview of the related literature concerning the local covariance matrix constructed from samples on embedded manifolds in Euclidean space.  Most research focuses on closed manifolds. \cite{alvarez2020local} explicitly calculates a higher order expansion in the bias analysis of the local covariance matrix using the $\epsilon$-radius ball scheme.  \cite{Tyagi2013} studies the spectral properties of  the local covariance matrix in the $\epsilon$-radius ball scheme for data under specific distributions on a closed manifold. \cite{Wu_Wu:2017} presents the bias and variance analyses of the local covariance matrix in the $\epsilon$-radius ball scheme on a closed manifold and studies its spectral properties. These analyses are further extended for samples on a manifold with boundary under the $\epsilon$-radius ball scheme in \cite{cheng2013local, wu2018locally}.  Notably, \cite{singer2012vector} provides the bias and variance analyses of the local covariance matrix constructed using a smooth kernel for samples on manifolds with and without boundary. Recent studies include considerations of noise. \cite{kaslovsky2014non} investigates the local covariance matrix in the KNN scheme,  focusing on data sampled from a specific class of closed manifolds contaminated by Gaussian noise. \cite{little2017multiscale, dunson2021inferring} explore the spectral properties of  the local covariance matrix constructed in the $\epsilon$-radius ball scheme, for data sampled on closed manifolds with Gaussian noise.

We further review the boundary detection methods developed in recent {decades}. The $\alpha$-shape algorithm and its variations \cite{Edelsbrunner:1983,Edelsbrunner:1994, chintakunta2013distributed} are widely applied in boundary detection. Other approaches \cite{dibakar1999computational, Bo:2008,  cholaquidis2016set} utilize convexity and concavity relative to the inward normal direction of the boundary. However, these methods except \cite{Bo:2008} are most effective when applied to data on manifolds with boundary of the same dimension as the ambient Euclidean space.  Several methods \cite{xue2009boundary, Xia:2006, Qiu:2007, cao2018multidimensional, qiu2016clustering} are developed based on the asymmetry and the volume variation near the boundary; e.g. as a point moves from the boundary to the interior, its neighborhood should encompass more points. Nevertheless, the performance of these algorithms is sensitive to the manifold's extrinsic geometry and data distribution. Recently, \cite{aamari2023minimax} proposes identifying boundary points using Voronoi tessellations over projections of neighbors onto estimated tangent spaces. \cite{calder2022boundary} detects boundary points by directly estimating the distance to the boundary function for points near the boundary.

The remainder of the paper is structured as follows. We review the LLE algorithm and its relation with the local covariance matrix in Section 2. In Section 3.1, we define the detected boundary points based on the manifold with boundary setup and introduce the BD-LLE algorithm in both the $\epsilon$-radius ball and KNN schemes. Section 3.2 discusses the relation between BD-LLE and the local covariance matrix. In Section 3.3, we propose the selection of the regularizer parameter for BD-LLE based on the spectral properties of the local covariance matrix. {In Section 3.4, we propose the selection of the scale parameters $\epsilon$ and $K$ in both nearest neighborhood search schemes.  Section 4 presents the bias and variance analyses of the local covariance matrix in the KNN scheme and the BD-LLE algorithm in both the $\epsilon$-radius ball and KNN schemes.} Section 5 provides numerical simulations comparing the performance of BD-LLE with different boundary detection algorithms. {Section 6 presents a framework designed to improve boundary detection in noisy data.} Table \ref{Table:Notations} summarizes the commonly used notations.

\begin{table}\label{Table:Notations}
\caption{Commonly used notations in this paper.}
\begin{tabular}{|l|l|} 
\hline $Symbol$ & $Meaning$\\ 
\hline\hline 
$M$ &  $d$-dimensional compact smooth manifold with smooth boundary \\ 
$\partial M$ &  The boundary of $M$ \\
$\iota$ & An isometric embedding of $M$ in $\mathbb{R}^p$\\
$P$ & P.d.f. on $M$ with lower and upper bounds $P_m$ and $P_M$ respectively\\
$\{x_i\}_{i=1}^n$ & Points sampled based on $P$ from $M$ \\ 
$\mathcal{X}=\{z_i=\iota(x_i)\}_{i=1}^n$ & The point cloud \\ 
$\epsilon$, $K$ & The scale parameters \\
$\mathcal{O}_k$ & The nearest neighbors of $z_k$  \\
$B_k$ & The value of the boundary indicator at $z_k$ \\
$B(x)$, $\tilde{B}(x)$ & The bump functions in the analyses of BD-LLE in different \\
& nearest neighborhood search schemes \\
\hline 
\end{tabular}
\end{table}

\section{Review of the Locally Linear Embedding}
Recall the definitions of the following two nearest neighborhood search schemes. 
\begin{definition}\label{nn definition}
Suppose $\mathcal{X}=\{z_i\}_{i=1}^n \subset \mathbb{R}^p$.  Denote the nearest {neighbors} of $z_k\in \mathcal{X}$ as $\mathcal{O}_k=\{z_{k,i}\}_{i=1}^{N_k}$ with $N_k$ to be the number of points in $\mathcal{O}_k$.  

In the $\epsilon$-radius ball scheme with $\epsilon>0$, 
$$\mathcal{O}_k= \{z_i \in \mathcal{X}| 0<\|z_i-z_k\|_{\mathbb{R}^p} \leq \epsilon \}.$$ 

In the KNN scheme with $1 \leq K \leq n-1$,  for any $z_k \in \mathcal{X}$, we rearrange $\mathcal{X} \setminus  \{z_k\}=\{z'_i\}_{i=1}^{n-1}$ based on their distances to $z_k$, i.e. $\|z'_1-z_k\|_{\mathbb{R}^p} \leq \cdots \leq \|z'_K-z_k\|_{\mathbb{R}^p} \leq \cdots \leq \|z'_{n-1}-z_k\|_{\mathbb{R}^p} $. {Then, $\|z'_K-z_k\|_{\mathbb{R}^p}$ is called the $K$-distance of $z_k$, and}
$$\mathcal{O}_k=\{z_i \in \mathcal{X}| 0<\|z_i-z_k\|_{\mathbb{R}^p} \leq \|z'_K-z_k\|_{\mathbb{R}^p} \}. $$ 
\end{definition}

\begin{remark}
{For the above definitions of $\mathcal{O}_k$ and $N_k$ in the KNN scheme, if there is $z'_j \in \{z'_i\}_{i=1}^{n-1}$ with $j>K$ and $\|z'_j-z_k\|_{\mathbb{R}^p}=\|z'_K-z_k\|_{\mathbb{R}^p}$, then $N_k > K$. Otherwise $N_k=K$.}
\end{remark}

We summarize the essential details of the LLE and {direct the reader} to \cite{Roweis_Saul:2000,Wu_Wu:2017, wu2018locally} for an in-depth discussion. {The LLE algorithm is based on the regularized barycentric coordinates. For any $z_k$ in the point cloud $\mathcal{X}=\{z_i\}_{i=1}^n \subset \mathbb{R}^p$, the regularized barycentric coordinates of $z_k$ are constructed through the following steps. First, either the  $\epsilon$-radius ball scheme or the KNN scheme is applied to determine the nearest neighbors $\mathcal{O}_k$ of $z_k$.  Let $\mathcal{O}_k=\{z_{k,i}\}_{i=1}^{N_k}$ denote those neighbor points. Second, using $\mathcal{O}_k$, we construct the {\em local data matrix} $G_{n,k}\in \mathbb{R}^{p \times {N_k}}$, where the $j$-th column is the vector $z_{k,j}-z_k$. Finally, the regularized barycentric coordinates are the weights $w_k \in \mathbb{R}^{N_k}$ assigned to the points in $\mathcal{O}_k$ and are the solutions of the following equation:
\begin{align}
(G_{n,k}^\top G_{n,k}+c I_{N_k\times N_k})y_k=\boldsymbol{1}_{N_k}\,, \quad w_k=\frac{y_k}{y_k^\top \boldsymbol{1}_{N_k}},  \label{Section2:wk}
\end{align}
where $\boldsymbol{1}_{N_k}$ is the vector in $\mathbb{R}^{N_k}$ with all entries equal to $1$ and $c>0$ is the regularizer chosen by the user. The regularized barycentric coordinates are extended to the LLE matrix $W \in \mathbb{R}^{n \times n}$, where the entry $W_{ki}$ equals $w_k(j)$ if $z_i=z_{k,j}$ in $\mathcal{O}_k$, and $0$ otherwise. Dimension reduction  is performed using the eigenvectors of the matrix $(I-W)^\top (I-W)$.  Our boundary detection algorithm is  developed based on the regularized barycentric coordinates, i.e. the solutions of \eqref{Section2:wk}.

The regularized barycentric coordinates can be expressed through the eigenpairs of the local covariance matrix.  Define
\begin{align}\label{local covariance matrix}
C_{n,k}=G_{n,k}G_{n,k}^\top =\sum_{i=1}^{N_k}(z_{k,i}-z_k)(z_{k,i}-z_k)^\top \in \mathbb{R}^{p \times p}.
\end{align}
Then $\frac{1}{n}C_{n,k}$ represents the local covariance matrix associated with $\mathcal{O}_k$ at $z_k$. Consider the orthonormal eigen-decomposition of the matrix $C_{n,k}=U_{n,k} \Lambda_{n,k} U_{n,k}^\top$ , where $U_{n,k} \in O(p)$ and $\Lambda_{n,k}$ is the eigenvalue matrix with the $i$-th diagonal entry denoted as $\lambda_{n,i}(z_k)$. We assume the eigenvalues are arranged in the decreasing order, i.e. ${\lambda}_{n,1}(z_k) \geq {\lambda}_{n,2}(z_k) \geq \cdots \geq {\lambda}_{n,p}(z_k)\geq 0$. Let $r_n=\texttt{rank}(C_{n,k}) \leq p$, and define $I_{p,r_n}:=\begin{bmatrix}
I_{r_n} &  0 \\
0 & 0  \\
\end{bmatrix}\in \mathbb{R}^{p\times p}$.  The regularized pseudo inverse of $C_{n,k}$ is given by
\begin{equation}\label{Definition:Irho:Soft}
\mathcal{I}_c(C_{n,k}):=U_{n,k}I_{p,r_n} (\Lambda_{n,k}+c I_{p\times p})^{-1} U_{n,k}^\top,
\end{equation}
where $c$ is the regularizer of the LLE. Note that $\mathcal{I}_c(C_{n,k})$ is a symmetric matrix.  It is shown in \cite[Section 2]{Wu_Wu:2017} that the solution to (\ref{Section2:wk}) is 
\begin{align}
y_{k} =&\,c^{-1}(\boldsymbol{1}_{N_k}-  G_{n,k}^\top \mathcal{I}_c(C_{n,k}) G_{n,k}\boldsymbol{1}_{N_k})  \label{solution y_n},
\end{align}
and hence 
\begin{align}
w_k&\,=\frac{\boldsymbol{1}_{N_k}-  G_{n,k}^\top \mathcal{I}_c(C_{n,k}) G_{n,k}\boldsymbol{1}_{N_k}}
{N_k- \boldsymbol{1}_{N_k} ^\top G_{n,k}^\top \mathcal{I}_c(C_{n,k}) G_{n,k}\boldsymbol{1}_{N_k}}\,.\label{Expansion:LLEweightedKernel}
\end{align}
}

\section{Boundary detection algorithm inspired by the LLE}
\subsection{Setup of the problem and the main idea of the algorithm}
In this section, we focus on the identification of boundary points distributed on a manifold with boundary. Prior to delving into the BD-LLE algorithm, we introduce the following model of a manifold with boundary.
\begin{assumption}\label{main assumption 1}
Let $(M,g)$ be a d-dimensional compact, smooth Riemannian manifold with boundary isometrically embedded in $\mathbb{R}^p$ via $\iota:M \hookrightarrow \mathbb{R}^p$. We assume the boundary of $M$, denoted as $\partial M$, is smooth.  {Denote the pushforward of $\iota$ as $\iota_*$}.
\end{assumption}
Since $\iota$ is an embedding, the boundary of $\iota(M)$ satisfies $\partial \iota(M)=\iota(\partial M)$.  Next, we provide the following assumption about the samples on the manifold  with boundary $M$.
\begin{assumption}\label{main assumption 2}
{Suppose $(\Omega, \mathcal{F}, \mathbb{P})$ is a probability space, where $\mathbb{P}$ is a  probability measure defined on the Borel sigma algebra $\mathcal{F}$ on $\Omega$. Let  $X$ be a random variable on $(\Omega, \mathcal{F}, \mathbb{P})$ with the range on $(M,g)$.  We assume $\mathbb{P}_X:=X_*\mathbb{P}$ is absolutely continuous with respect to the volume measure $\mathfrak{m}$ on $M$ associated with $g$ so that $d\mathbb{P}_X=P d\mathfrak{m}$ by the Radon-Nikodym theorem, where $P$ is a non-negative function defined on $M$. We call $P$ the probability density function (p.d.f.) associated with $X$. We further assume $P \in C^2(M)$  and $0 <P_m \leq P(x) \leq P_M$ for all $x \in M$.  We assume $\{x_1 \cdots, x_n\} \subset M$ are i.i.d. sampled from $P$.}
\end{assumption} 

Under Assumptions \ref{main assumption 1} and \ref{main assumption 2}, we consider the point cloud $\mathcal{X}=\{z_i=\iota(x_i)\}_{i=1}^n$.  In this work, the geometric information of $\iota(M)$ is not accessible, and we propose the BD-LLE algorithm to detect the boundary points on $\iota(M)$ through the Euclidean coordinates of $\mathcal{X}$. Since $\partial M$ is a measure $0$ subset of $M$ and {$\mathbb{P}_X$} is absolutely continuous with respect to the volume measure on $M$,  the probability for a sample in $\mathcal{X}$ to lie on $\partial \iota(M)$ is $0$.  The best we can do is finding all points $\partial \mathcal{X}$ from $\mathcal{X}$ lying in a small neighborhood of $\partial  \iota(M)$ in $\iota(M)$. We call $\partial \mathcal{X}$  the {\em{detected boundary points}} from $\mathcal{X}$. 

The BD-LLE algorithm includes the construction of a {\em boundary indicator} (BI) over $\mathcal{X}$ using the barycentric coordinates of each sample point $z_k$. We describe the intuition behind this construction. Specifically, the value of the BI at $z_k$,  $B_k$,  approximates the value of a function $B(x)$ on $M$ at $x=x_k$.  The function $B(x)$ is constant on $\partial M$ and {attains its maximum} on $\partial  M$. It decreases rapidly along the normal direction of $\partial  M$ towards the interior of $M$. The value of $B(x)$ at a point $x$ is $0$ whenever $x$ is away from $\partial M$. If we choose a threshold, then the preimage of the values larger than the threshold under $B(x)$ is a neighborhood of  $\partial M$. In the context of the $\epsilon$ radius ball scheme, this preimage is a narrow, uniform collar region of $\partial M$.  Refer to (6) in Theorem \ref{prop BI 1} for a precise description. Therefore, points $z_k$  corresponding  to $B_k$ greater than the threshold are identified as boundary points.  

We summarize the steps of BD-LLE in Algorithm \ref{Algorithm inspired by LLE}. The inputs of the algorithm include the point cloud $\mathcal{X}$, the scale parameters $\epsilon$ or $K$, and the regularizer $c$. The outputs of the algorithm are the detected boundary points $\partial \mathcal{X} \subset \mathcal{X}$.

%The function  decays faster when $\epsilon$ is smaller.  Thus, if $\epsilon$ is smaller, then we can find a smaller neigborhood of $\partial  \iota(M)$ to have a more accurate approximation of the boundary. Later, we will provide the relation between the sample size $n$ and the bandwidth $\epsilon$ so that  the BI gives a good approximation of the function. We will also show that the choice of  $c=n \epsilon^{d+3}$ in Step3 of the algorithm guarantees such approximation regardless the distribution of the samples on the manifold.  

\begin{algorithm}[h]
%\SetAlgoLined
%\SetKwInOut{KwIn}{Input}
%
%\KwIn{$\mathcal{X}$,  $c$, $\epsilon$ or K, and $0<q<100$} 
\begin{algorithmic}[1]
\STATE{Inputs: $\mathcal{X}$, $\epsilon$ or $K$, and $c$}

\STATE{For each $k$, find the neighborhood $\mathcal{O}_k=\{z_{k,i}\}_{i=1}^{N_k} \subset \mathcal{X}$ of $z_k$ through either the $\epsilon$-ball scheme or the KNN scheme.} 

\STATE{Construct $G_{n,k}=\begin{bmatrix}
 | & | & |\\
\ldots & z_{k,j}-z_k & \ldots
\\   |  &  |  & | 
\end{bmatrix}. \quad  z_{k,j} \in \mathcal{O}_k$.\;}

%Choose $c=n \epsilon^{d+3}$ for  the $\epsilon$-ball scheme and choose $c=n (\frac{K}{n})^{\frac{d+3}{d}}$ for  the KNN scheme. 
\STATE{Let $y_k=(G_{n,k}^\top G_{n,k}+c I_{N_k\times N_k})^{-1} \boldsymbol{1}_{N_k}$.  Let $B_k=\frac{N_k-c y_k^\top \boldsymbol{1}_{N_k}}{N_k}$.}

\STATE{Set $\partial \mathcal{X}:=\{x_k|\, B_k\geq \frac{1}{2} \max_{i}{B_i}\}$.\; }
\caption{BD-LLE algorithm} \label{Algorithm inspired by LLE}
\end{algorithmic}
\end{algorithm}

\subsection{Relation between the local covariance matrix and the boundary indicator}
{We express BI explicitly through the local data matrix $G_{n,k}$ and the local covariance matrix  $\frac{1}{n}C_{n,k}$.}  Since $G_{n,k}$ is invariant under translation, and $G_{n,k}^\top G_{n,k}$ is invariant under orthogonal transformation on $\mathbb{R}^p$, $y_k$ in Step 3 of Algorithm \ref{Algorithm inspired by LLE} is invariant under orthogonal transformation and translation.  Hence, $B_k$ is invariant under translation and orthogonal transformation on $\mathbb{R}^p$. Moreover, $B_k$ is related to the local covariance matrix $\frac{1}{n}C_{n,k}$ through \eqref{solution y_n}.  We summarize {this fact} as the following proposition.

\begin{proposition}\label{invariant B_k}
The value of the BI in Algorithm \ref{Algorithm inspired by LLE} at each $z_k$, denoted $B_k$, is invariant under translation of $\mathcal{X}$ and orthogonal transformation on $\mathbb{R}^p$. {Moreover,
\begin{align}\label{proof prop1 main}
B_k=\frac{N_k-c y_k^\top \boldsymbol{1}_{N_k}}{N_k}=\frac{ \boldsymbol{1}_{N_k} ^\top G_{n,k}^\top \mathcal{I}_c(C_{n,k}) G_{n,k}\boldsymbol{1}_{N_k}}{N_k},\,
\end{align}
where $\mathcal{I}_c(C_{n,k})$ is defined in \eqref{Definition:Irho:Soft}.}
\end{proposition}

\subsection{Selection of the regularizer} \label{Selection of the regularizer}
{The choice of the regularizer $c$ is crucial for the LLE algorithm. In \cite{Wu_Wu:2017, wu2018locally}, when the point cloud is distributed on an embedded $d$-dimensional manifold with or without boundary, for dimension reduction purpose, the selection of $c$ is discussed for the LLE matrix to recover the Laplace-Beltrami operator of the manifold. Under the $\epsilon$-radius ball scheme, given certain relations between $\epsilon$ and the size of the point cloud $n$, these studies demonstrate that the $d$ largest eigenvalues of $C_{n,k}$ are of order $n\epsilon^{d+2}$, while the remaining smaller eigenvalues encode the local extrinsic geometric information of the data and are of order $O(n\epsilon^{d+4})$. For a review of the results regarding the spectral properties of the local covariance matrix in the $\epsilon$-radius ball scheme, refer to Section \ref{local covariance matrix eps scheme} of the Supplementary Material.  Since the Laplace-Beltrami operator depends only on the intrinsic geometry of the manifold, it is suggested that $c=n \epsilon^{d+3}$ should be used to dominate the smaller eigenvalues of $C_{n,k}$ and eliminate the impact the extrinsic geometry. }

From the discussion in the previous subsection, we observe that an ideal BI should satisfy two main criteria: (1) it should be smaller within the interior of the manifold to distinguish interior points from boundary points, and (2) it should approximate a constant near the boundary to facilitate straightforward threshold selection.  Proposition \ref{proof prop1 main} establishes that the BI depends on the regularized pseudo inverse $\mathcal{I}_c(C_{n,k})$. According to \eqref{Definition:Irho:Soft}, if the regularizer $c$ outweighs the eigenvalues of $C_{n,k}$, then $\mathcal{I}_c(C_{n,k})$ is dominated by $c^{-1} U_{n,k}I_{p,r_n} U_{n,k}^\top$ causing $B_k$ to lose the geometric information of the manifold. As a result, the values of $B_k$ over the interior and near the boundary may not be distinguishable. Refer to the right panel in Figure \ref{fig: BI}. Conversely, the inversion of $G_{n,k}^\top G_{n,k}+c I_{N_k\times N_k}$ in Step 3 of the algorithm implies that if $c$ is too small, the BI becomes unstable. Moreover, choosing c too small contaminates the values of $B_k$ near the boundary with the extrinsic geometric information. 

{In section \ref{Theoretical Analysis section}, we show that, within the KNN scheme, given certain relations between $K$ and $n$, the $d$ largest eigenvalues of $C_{n,k}$ are of order  $n(\frac{K}{n})^{\frac{d+2}{d}}$, while the remaining smaller eigenvalues are of order  $O(n(\frac{K}{n})^{\frac{d+4}{d}})$. Therefore, motivated by \cite{Wu_Wu:2017, wu2018locally}, we propose the regularizer $c=n\epsilon^{d+3}$ or $n(\frac{K}{n})^{\frac{d+3}{d}}$ in our theoretical analysis of the BI. This choice of $c$ is smaller than the first $d$ eigenvalues of  $C_{n,k}$ thereby enabling the BI to preserve the intrinsic geometric information of the manifold. Meanwhile, it dominates the remaining eigenvalues of  $C_{n,k}$, eliminating the influence of both the sample distribution and extrinsic geometric information. As a result, the BI will satisfy the desired properties.  As shown in Section \ref{local covariance matrix eps scheme} of the Supplementary Material and Section \ref{Theoretical Analysis section},  the eigenvalues of $C_{n,k}$ depend on the p.d.f. of the data. Thus, we propose the following more practical choice of $c$, which incorporates the distribution of the samples on the manifold. Recall that ${\lambda}_{n,1}(z_k) \geq {\lambda}_{n,2}(z_k) \geq \cdots \geq {\lambda}_{n,p}(z_k)\geq 0$ are the eigenvalues of $C_{n,k}$. If $d<p$ and $\lambda_{n,d+1}(z_i) \not=0$ for some $i$,
\begin{align}\label{selection of the regularizer 0}
c=\frac{1}{n}\sqrt{ \big(\sum_{i=1}^n \lambda_{n,d}(z_i)\big)\big(\sum_{i=1}^n \lambda_{n,d+1}(z_i)\big)}.
\end{align}
Otherwise, $C_{n,k}$ has only $d$ nonzero eigenvalues for all $k$. Then, we choose 
\begin{align}\label{selection of the regularizer 1}
c=\frac{\mathfrak{s}}{n}\big(\sum_{i=1}^n \lambda_{n,d}(z_i)\big),
\end{align}
where $\mathfrak{s}<1$ is smaller than $\epsilon$ or $(\frac{K}{n})^{\frac{1}{d}}$. Note that based on our analysis of the eigenvalues of $C_{n,k}$, the above choices of $c$ are of order $O(n\epsilon^{d+3})$ or $O\big(n(\frac{K}{n})^{\frac{d+3}{d}}\big)$) and exceed $\frac{1}{n}\sum_{i=1}^n \lambda_{n,d+1}(z_i)$.  We illustrate the performance of the BI on a unit disk based on the suggested $c$ in Figure \ref{fig: BI}.}

\subsection{Selection of the scale parameters} \label{selection of eps K}
{In Section \ref{Theoretical Analysis section}, we provide the bias and variance analyses of the BI under the KNN scheme. Suppose $c=n(\frac{K}{n})^{\frac{d+3}{d}}$.  The bias and variance errors of the BI are $(\frac{K}{n})^{\frac{1}{d}}$ and $\sqrt{\frac{\log(n)}{K}}$ respectively. Ignoring the $\log(n)$ factor, we propose to choose $K$ by balancing these errors, i.e. $K=\lceil n^{\frac{1}{1+d/2}} \rceil$. This choice of $K$ incorporates information about the distribution of $\mathcal{X}$ and the geometry of the underlying manifold $\iota(M)$. 

Under the $\epsilon$-radius ball scheme, suppose  $c=n\epsilon^{d+3}$. Then, the bias and variance errors of the BI are $\epsilon$ and $\frac{\sqrt{\log (n)}}{n^{1/2}\epsilon^{d/2}}$ respectively. Balancing these errors leads to a choice of $\epsilon$ that depends only on $n$. However,  unlike the KNN scheme, such choice of $\epsilon$ does not vary with respect to the distribution of $\mathcal{X}$ and the geometry of $\iota(M)$.  For example, for a fixed $n$, this choice of $\epsilon$ remains unchanged when the size of the underlying manifold is scaled. Therefore, we propose selecting $\epsilon$ based on $K$ in the KNN scheme. For any $z_k \in \mathcal{X}=\{z_i\}_{i=1}^n$, let $K=\lceil n^{\frac{1}{1+d/2}} \rceil$, and let $R_k$ represent the $K$-distance of $z_k$, as defined in Definition \ref{nn definition}. Based on the analysis in Section \ref{Theoretical Analysis section}, $R_k$ is small when $z_k$ is far from the boundary of $\iota(M)$ and points are densely distributed around $z_k$. In contrast, $R_k$ is large when $z_k$ is near the boundary of $\iota(M)$ or points are sparse around $z_k$. Let $\epsilon_{min} = \text{median}_{k=1, \ldots, n} R_k$ and $\epsilon_{max} = \max_{k=1, \ldots, n} R_k$. Since we expect $\epsilon$ to be sufficiently large such that the $\epsilon$ neighborhood of $z_k$ captures enough geometric information of the manifold, especially when $z_k$ is near the boundary of $\iota(M)$ or points are sparse around $z_k$, we propose selecting $\epsilon$ within the range between $\epsilon_{min}$ and $\epsilon_{max}$.  In Figure \ref{fig: BI}, we illustrate the performance of the algorithm on a unit disk, using the proposed scale parameters and the suggested regularizer from the previous section. In Figure \ref{fig: BI 2}, we demonstrate the performance of the algorithm on a unit disk, using various values of  $\epsilon$ between $\epsilon_{min}$ and $\epsilon_{max}$ along with the regularizer recommended in the previous section. For each choice of $\epsilon$,  the algorithm successfully identifies the points within a narrow, uniform collar region of the boundary, with the width of the region increasing as $\epsilon$ increases.}
\begin{figure}[ht]
 \centering
 \includegraphics[scale=0.48]{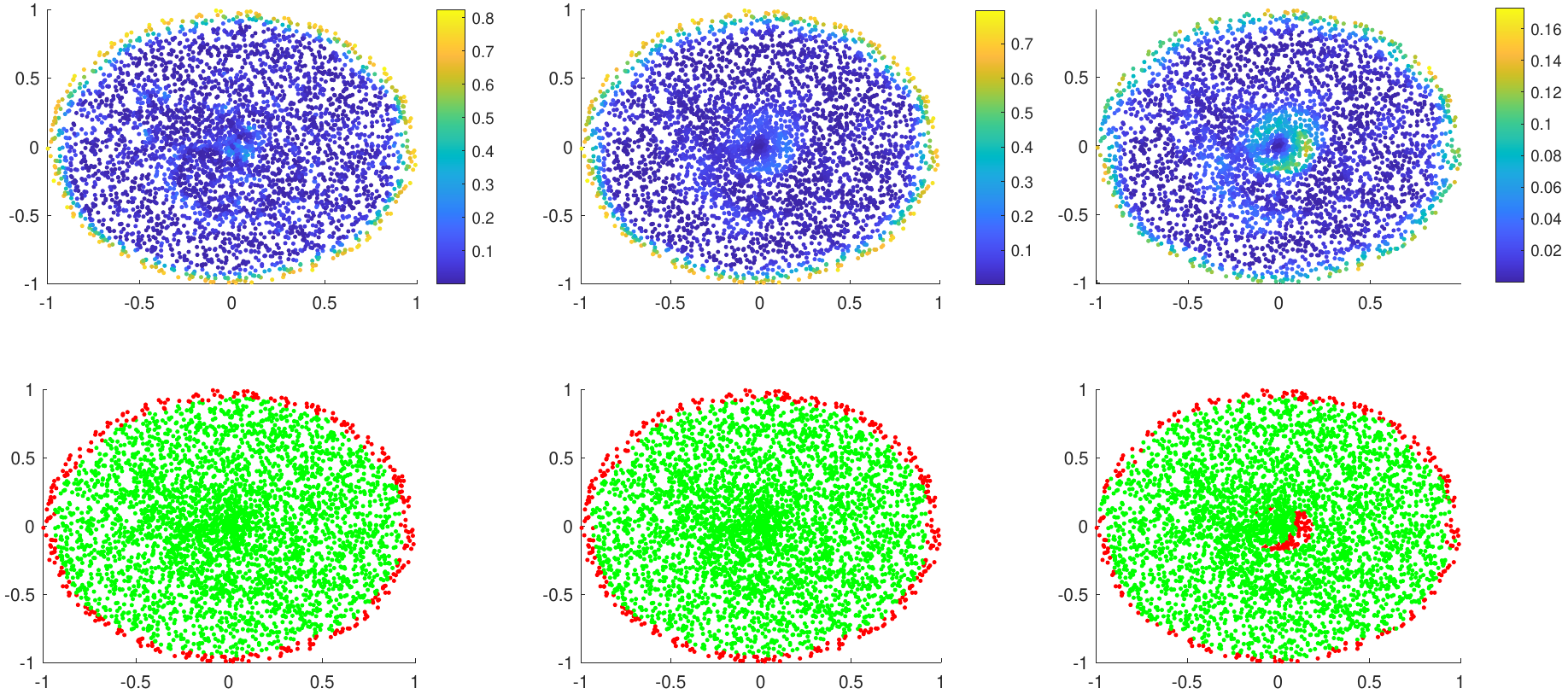}
 \caption{{Top Panels: The plots of the BIs constructed over $n=4000$ non-uniformly sampled points on the unit disk. Bottom Panels: The plots of the detected boundary points through the BIs in the corresponding top panels.  Left Panels: The BI is constructed in the KNN scheme with $K=\lceil \sqrt{4000} \rceil=64$, as proposed in Section \ref{selection of eps K}. Since $d=p=2$,  the regularizer is $c=\frac{1}{100n}\big(\sum_{i=1}^n \lambda_{n,2}(z_i)\big)$,  as specified in \eqref{selection of the regularizer 1}. Middle Panels: The BI is constructed in the $\epsilon$-radius ball scheme.  $\epsilon_{min}=0.130$ and $\epsilon_{max}=0.231$ with $\epsilon=0.180$ selected. The regularizer is $c=\frac{1}{100n}\big(\sum_{i=1}^n \lambda_{n,2}(z_i)\big)$.  Right Panels: The BI is constructed with $\epsilon=0.180$ and a large regularizer $c=\frac{1}{n}\sum_{i=1}^n \lambda_{n,2}(z_i)$ is selected. Then, the values of the BI are not distinguishable between the boundary and interior points, resulting in some boundary points not being identified and some interior points being incorrectly identified as boundary points.  For the same $\epsilon$, when the regularizer is too small, for example when $c$ is extremely close to $0$, the BI is not stable.}}
 \label{fig: BI}
\end{figure}
\begin{figure}[ht]
 \centering
 \includegraphics[scale=0.52]{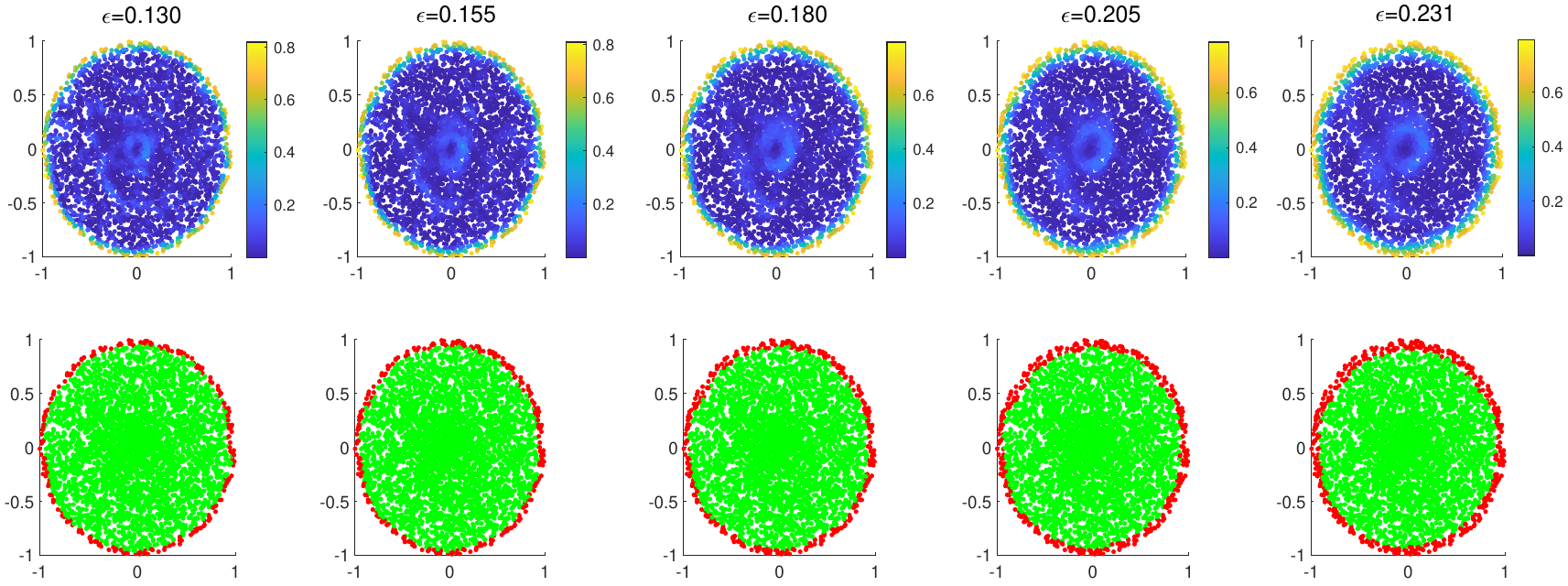}
 \caption{{Top Panels: The plots of the BIs constructed in the $\epsilon$-radius ball scheme over $n=4000$ non-uniformly sampled points on the unit disk.  As outlined in Section \ref{selection of eps K}, $\epsilon_{min}=0.130$ and $\epsilon_{max}=0.231$. The BIs are constructed with $\epsilon=0.130, 0.155, 0.180, 0.205, 0.231$ respectively.  Since in this case $d=p=2$,  the regularizer is $c=\frac{1}{100n}\big(\sum_{i=1}^n \lambda_{n,2}(z_i)\big)$,  where $\lambda_{n,2}(z_i)$ is the second largest eigenvalue of $C_{n,i}$ in the corresponding $\epsilon$-radius ball scheme, as specified in \eqref{selection of the regularizer 1}. Bottom Panels: The plots of the detected boundary points through the BIs in the corresponding top panels. }} \label{fig: BI 2}
\end{figure}

\section{Theoretical Analysis}\label{Theoretical Analysis section}
{In this section, we delve into the theoretical analysis of the local covariance matrix in the KNN scheme,  as well as the BI in both the $\epsilon$-radius ball scheme and the KNN scheme.}  To start, we provide an introduction of the essential geometric preliminaries.
\subsection{Preliminary definitions}
Suppose $g$ is the Riemannian metric on $M$. Denote $d_g(\cdot,\cdot)$ to be the distance function on $M$ associated with $g$. We define the following concepts around the boundary of $M$. 
\begin{definition}\label{preliminary def on boundary} \
\begin{enumerate}[label=(\arabic*)]
{\item For any $x \in M$,  the distance to the boundary function is defined as
\begin{equation}\label{Definition tildeepsilon gamma}
\tilde{\epsilon}(x) =d_g(x,\partial M)=\min_{y \in \partial M} d_g(y,x). 
\end{equation}
\item
For $\epsilon>0$, we define the $\epsilon$-neighborhood of $\partial M$ as 
\begin{equation*}
M_{\epsilon}=\{x \in M | d_g(x, \partial M) < \epsilon\}.
\end{equation*}
Denote $x_\partial:=\arg\min_{y \in  \partial M} d_g(y,x)$. When $x \in M_{\epsilon}$ and $\epsilon$ is sufficiently small, due to the smoothness of the boundary, such $x_\partial$ is unique and we have  $0\leq \tilde{\epsilon}(x) <\epsilon$.  }
\item
For any $x \in \partial M$, let $\gamma_x(t)$ be the unit speed geodesic such that $\gamma_x(0)=x$ and $\gamma\,'_x(0)$ is the unit inward normal vector {of $\partial M$} at $x$. 
\end{enumerate}
\end{definition}

The distance to the boundary function $d_g(x,\partial M)$ satisfies the following properties. 
\begin{proposition}\label{properties of the distance to boundary function}
Under Assumption \ref{main assumption 1}, $d_g(x,\partial M)$ is a continuous function on $M$ and differentiable almost everywhere on $M$. When $\epsilon$ is small enough, $d_g(x,\partial M)$ is smooth on $M_\epsilon$.
\end{proposition}

Recall that both the BI and the local covariance matrix are invariant under translation. Moreover, the BI is invariant under orthogonal transformation. Hence, we introduce the following assumption to simplify the proofs and the statements of of the main results.

\begin{assumption}\label{assumption tangent space} \
For each fixed $x_k$, we translate and rotate $\iota(M)$ in $\mathbb{R}^p$ as follows.
\begin{enumerate}[label=(\arabic*)]
\item
We translate $\iota(M)$ in $\mathbb{R}^p$ so that $\iota(x_k)=0$.
\item
Denote $\{e_i\}_{i=1}^p$ to be the canonical basis of $\mathbb{R}^p$, where $e_i$ is a unit vector with $1$ in the $i$-th entry. Denote $\iota_*T_{x}M$ to be the embedded tangent space of $\iota(M)$ at $\iota(x)$ in $\mathbb{R}^p$ and $(\iota_*T_{x}M)^\bot$ be the normal space at $\iota(x)$. Fix any $\iota(x_k)=0 \in \iota(M)$, we assume that {$\mathbb{R}^p$} has been properly rotated so that $\iota_*T_{x_k}M$ is spanned by $e_1,\ldots,e_d$.
\item
When $x_k \in M_{\epsilon}$ and $\epsilon$ is sufficiently small, let $x_{\partial,k}$ be the unique point on $\partial M$ which realizes the distance from $x_{k}$ to $\partial M$ {defined in (2)} of Definition \ref{preliminary def on boundary}. {Let $\gamma_{x_{\partial,k}}(t)$ represent} the unit speed geodesic with $\gamma_{x_{\partial,k}}(0)=x_{\partial,k}$ and $\gamma_{x_{\partial,k}}(\tilde{\epsilon}(x_k))=x_k$.  We further rotate {$\mathbb{R}^p$} so that 
$$e_d=\iota_* \frac{d}{dt}\gamma_{x_{\partial,k}}(\tilde{\epsilon}(x_k)).$$
In particular, when $x \in \partial M$, $e_d$ is the inward normal direction of $\partial \iota(M)$ at $\iota(x)$.
\end{enumerate}
\end{assumption}

At last, we introduce the following functions that will be used in the main results.  

\begin{definition} \label{sigmas summary} 
Let $|S^m|$ denote the volume of the $m$-dimensional unit sphere. We define the following functions on $[0,\infty)$, where the constant $\frac{|S^{d-2}|}{d-1}$ is defined to be $1$ when $d=1$.
\begin{align*}
\sigma_{0}(t, \epsilon)&:= \left\{
\begin{array}{ll}
\frac{|S^{d-1}|}{2d}+\frac{|S^{d-2}|}{d-1}\int_{0}^{\frac{t}{\epsilon}} (1-x^2)^{\frac{d-1}{2}}dx& \mbox{ for }0 \leq t \leq \epsilon\\
\frac{|S^{d-1}|}{d} & \mbox{ for }t > \epsilon
\end{array}\right.\\
\sigma_{1,d}(t,\epsilon)&:=\left\{\begin{array}{ll}
-\frac{|S^{d-2}|}{d^2-1}(1-(\frac{t}{\epsilon})^2)^{\frac{d+1}{2}}&\mbox{ for }0 \leq t \leq \epsilon\\
0&\mbox{ otherwise}
\end{array}\right.\nonumber\\
\sigma_{2}(t,\epsilon)&:=
\left\{
\begin{array}{ll}
\frac{|S^{d-1}|}{2d(d+2)}+\frac{|S^{d-2}|}{d^2-1}\int_{0}^{\frac{t}{\epsilon}} (1-x^2)^{\frac{d+1}{2}} dx& \mbox{ for }0 \leq t \leq \epsilon\\
\frac{|S^{d-1}|}{d(d+2)} &\mbox{ otherwise}
\end{array}\right.\nonumber\\
\sigma_{2,d}(t,\epsilon)&:=\left\{
\begin{array}{ll}
\frac{|S^{d-1}|}{2d(d+2)}+\frac{|S^{d-2}|}{d-1}\int_{0}^{\frac{t}{\epsilon}} (1-x^2)^{\frac{d-1}{2}}x^2 dx&\mbox{ for } 0 \leq t \leq \epsilon\\
\frac{|S^{d-1}|}{d(d+2)} & \mbox{ otherwise}
\end{array}\right.\nonumber
%\sigma_{3}(t,\epsilon)&:=
%\left\{
%\begin{array}{ll}
%-\frac{|S^{d-2}|}{(d^2-1)(d+3)}(1-(\frac{t}{\epsilon})^2)^{\frac{d+3}{2}}&\mbox{ for }0 \leq t \leq \epsilon\\
%0&\mbox{ otherwise} 
%\end{array}\right.\nonumber\\
%\sigma_{3,d}(t,\epsilon)&:=
%\left\{
%\begin{array}{ll}
%-\frac{|S^{d-2}|}{(d^2-1)(d+3)}(2+(d+1)(\frac{t}{\epsilon})^2)(1-(\frac{t}{\epsilon})^2)^{\frac{d+1}{2}}&\mbox{ for }0 \leq t \leq \epsilon\\
%0& \mbox{ otherwise}
%\end{array}\right.\nonumber
\end{align*}
\end{definition}

Note that {the function $\sigma_{1,d}(t,\epsilon)$ is bounded} from {above by $0$ and below by a constant} depending on $d$. The functions $\sigma_{0}(t, \epsilon)$, $\sigma_{2}(t,\epsilon)$, and $\sigma_{2,d}(t,\epsilon)$ are bounded from below and above by constants depending on $d$. These functions are smooth everywhere except at $t=\epsilon$. At $t=\epsilon$, they are continuous and their level of smoothness depends on $d$.

\subsection{Analysis of the boundary indicator in the $\epsilon$-radius ball scheme}
We provide the following bias and variance analyses of the BI under the $\epsilon$-radius ball scheme. The proof of the theorem is in Section \ref{proof prop BI} of the Supplementary Material.

\begin{theorem}\label{prop BI 1}
Under Assumptions \ref{main assumption 1}, \ref{main assumption 2},  and the $\epsilon$-radius ball scheme, suppose the regularizer $c=n\epsilon^{d+3}$ and suppose $\epsilon=\epsilon(n)$ so that $\frac{\sqrt{\log(n)}}{n^{1/2}\epsilon^{d/2+1}}\to 0$ and $\epsilon\to 0$ as $n\to \infty$. When $\epsilon$ is small enough, with probability greater than $1-2n^{-2}$,  for all $k=1,\ldots,n$,
\begin{equation*}
B_k=B(x_k)+O(\epsilon)+ O\Big(\frac{\sqrt{\log (n)}}{n^{1/2}\epsilon^{d/2}}\Big)\,,
\end{equation*}
where the constants in $O(\epsilon)$ and $O\Big(\frac{\sqrt{\log(n)}}{n^{1/2}\epsilon^{d/2}}\Big)$ depend on $P_m$,  the $C^1$ norm of $P$ and the second fundamental form of $\iota(M)$.

The function $B(x):M\to \mathbb{R}$ has the following properties:
\begin{enumerate}[label=(\arabic*)]
\item $B(x)= \frac{(\sigma_{1,d}(\tilde{\epsilon}(x),\epsilon))^2}{\sigma_0(\tilde{\epsilon}(x),\epsilon) \sigma_{2,d}(\tilde{\epsilon}(x),\epsilon)}$. Hence, $B(x)$ is always continuous on $M$. When $\epsilon$ is small enough, it is smooth except at the set $\{x \in M | d_g(x, \partial M) = \epsilon\}$.
\item $B(x)=\frac{4d^2(d+2)|S^{d-2}|^2}{(d^2-1)^2|S^{d-1}|^2}$ when $x\in \partial M$.
\item $B(x)=0$ when $x\in M\setminus M_\epsilon$.
\item For $x_1, x_2 \in \partial M$, $B(\gamma_{x_1}(t))=B(\gamma_{x_2}(t))$ for $0 \leq t \leq \epsilon$.
\item Fix any $x \in \partial M$, $B(\gamma_{x}(t))$ is a monotone decreasing function of $t$ for $0 \leq t \leq \epsilon$. Moreover 
$\frac{d^2(d+2)|S^{d-2}|^2}{(d^2-1)^2|S^{d-1}|^2} (1-(\frac{t}{\epsilon})^2)^{d+1} \leq  B(\gamma_{x}(t)) \leq \frac{4d^2(d+2)|S^{d-2}|^2}{(d^2-1)^2|S^{d-1}|^2} (1-(\frac{t}{\epsilon})^2)^{d+1}$ for $0 \leq t \leq \epsilon$.
\item For any $0<\tau<\frac{4d^2(d+2)|S^{d-2}|^2}{(d^2-1)^2|S^{d-1}|^2}$, $B^{-1}(\tau, \infty)=M_r$ with $M_r \subset M_{\epsilon}$.
\end{enumerate}
\end{theorem}

We discuss the implications of the above results regarding the BI in the $\epsilon$-radius ball scheme. By  (2) and (4), $B(x)$ remains constant and {attains its maximum} on $\partial M$. Additionally, (4) and (5) indicate that $B(x)$ decreases monotonically at the same speed along any geodesic perpendicular to $\partial M$. Therefore, according to (3),  the function $B(x)$ behaves like a bump function, concentrating on $\partial M$ and vanishing in $M \setminus  M_{\epsilon}$. 

%Second, we estimate a lower bound of $B(x)$ on $\partial M$ for all $d$. Recall that $B(x)=\frac{4d^2(d+2)|S^{d-2}|^2}{(d^2-1)^2|S^{d-1}|^2}$ for $x$ on $\partial M$.  For $d>2$, by Gautschi's inequality, $\frac{1}{\pi}(\frac{d}{2}-1) \leq  \frac{|S^{d-2}|^2}{|S^{d-1}|^2}$. Hence, $\frac{4d^2(d+2)(\frac{d}{2}-1)}{\pi(d^2-1)^2} \leq \frac{4d^2(d+2)|S^{d-2}|^2}{(d^2-1)^2|S^{d-1}|^2}$. Note that for $d>2$, $\frac{4d^2(d+2)(\frac{d}{2}-1)}{\pi(d^2-1)^2}$ is a monotone increasing function of $d$. When $d = 10$, $\frac{4d^2(d+2)(\frac{d}{2}-1)}{\pi(d^2-1)^2}>0.62$. By evaluating  $\frac{4d^2(d+2)|S^{d-2}|^2}{(d^2-1)^2|S^{d-1}|^2}$ directly for $d=1 ,\cdots, 9$, we conclude that when $x \in \partial M$, $B(x)>0.62$ for all $d$.  

Suppose $\epsilon$ and $n$ satisfy the conditions in Theorem \ref{prop BI 1}. For $n$ large enough, with high probability, we have {$|B_k-B(x_k)|=O(\epsilon)+ O\Big(\frac{\sqrt{\log (n)}}{n^{1/2}\epsilon^{d/2}}\Big)$, which is small enough}. Thus, when $x_k$ is near the boundary, $B_k$ approximates an order $1$ constant. Specifically, from \eqref{proof prop1 main}, we have $B_k=\frac{ \boldsymbol{1}_{N_k} ^\top G_{n,k}^\top \mathcal{I}_c(C_{n,k}) G_{n,k}\boldsymbol{1}_{N_k}}{N_k}$. The denominator $N_k$ acts like the $0-1$ kernel density estimator, eliminating the impact of the non-uniform distribution of the samples and ensuring that $B_k$ remains close to a constant near the boundary. Refer to Lemma \ref{Prop 1 Lemma 1.1} in the Supplementary Material for a detailed discussion. Additionally, refer to Proposition \ref{strong uniform consistency of  KDE} in the Supplementary Material for a strong uniform consistency result of kernel density estimation through $0-1$ kernel on a manifold with boundary. When $x_k \in M \setminus  M_{\epsilon}$, $B_k$ is of order {$O(\epsilon)+O\Big(\frac{\sqrt{\log (n)}}{n^{1/2}\epsilon^{d/2}}\Big)$. Therefore, by statement (6) in Theorem \ref{prop BI 1}, we can select a threshold on $B_k$ to identify points in $\iota(M_r)$ for some $M_r \subset M_\epsilon$. Moreover, as $n$ increases, choosing a smaller $\epsilon$ reduces the error between $B_k$ and $B(x_k)$ and decreases the size of $M_\epsilon$. Hence, larger $n$ enables more accurate detection of boundary points.}

\subsection{Analyses of the boundary indicator and the local covariance matrix in the KNN scheme}\label{Analysis KNN scheme 11}
We start with the following definition.
\begin{definition}\label{def of R(x)}
Let $B_{a}^{\mathbb{R}^p}(z)$ be the $p$-dimensional closed ball of radius $a$ centered at $z$ in $\mathbb{R}^p$. {Let $\mathcal{X}=\{z_i\}_{i=1}^n \subset \mathbb{R}^p$.  Under Assumption \ref{main assumption 1},} for any $x \in M$, define $N_a(x)=|B_{a}^{\mathbb{R}^p}(\iota(x)) \cap \mathcal{X}|$.  We define the following radius at $x$ associated with $K$: 
$$R(x)=\inf_a \{a>0, N_a(x) \geq K+1\}.$$
\end{definition}
Then, $R(x)$ has the following properties. The proof of the proposition is in Section \ref{proof prop BI 2} of the Supplementary Material. 
\begin{proposition}\label{continuity of R}  {Let $\mathcal{X}=\{z_i\}_{i=1}^n \subset \mathbb{R}^p$. Under Assumption \ref{main assumption 1},  we have}
\begin{enumerate}[label=(\arabic*)]
\item 
$R(x)$ is a continuous function on $M$. 
\item
For any $x \in \partial M$, suppose $\gamma_x(t)$ is {length minimizing} on $[0, t_2]$. Then $R(x)$ is Lipschitz along $\gamma_x(t)$ for $0 \leq t \leq t_2$. Specifically,
if $t_1<t_2$, then $|R(\gamma_x(t_1)) -R(\gamma_x(t_2))|  \leq t_2-t_1$. Moreover, $\frac{t_1}{R(\gamma_x(t_1))}<\frac{t_2}{R(\gamma_x(t_2))}$  whenever $t_1<R(\gamma_x(t_1))$.
\end{enumerate}
\end{proposition}
Since $R(x)$ is a continuous function on $M$ and $M$ is compact, $R(x)$ attains a maximum with 
\begin{align*}
R^*=\max_{x \in M} R(x).
\end{align*}
Recall the function $\sigma_0$ in Definition \ref{sigmas summary}. For a fixed $t \geq 0$, let
$$V(t, r)=\sigma_0(t, r) r^d.$$
Let $(x_1, x_2, \cdots, x_d)$ denote the coordinates in $\mathbb{R}^d$. The function $V(t, r)$ represents the volume of the region $\mathcal{R}_{t,r}$ between the ball of radius $r$ centered at the origin in $\mathbb{R}^d$ and the hyperplane $x_d=t$. Specifically, when $r \leq t$, $V(t, r)=\frac{|S^{d-1}|}{d} r^d$ is the volume of the ball of radius $r$. Note that $V(t,r):[0, \infty) \rightarrow [0, \infty)$ is a continuous, monotone increasing function of $r$ for a fixed $t$, and $V(t,r)$ is differentiable except at $r=t$. Hence, $s=V(t,r)$ has an inverse $r=U(t,s)$, where $U(t,s):[0, \infty) \rightarrow [0, \infty)$ is also monotone increasing. Specifically, $U(t,s)=(\frac{d s}{|S^{d-1}|})^{\frac{1}{d}}$ for $s<\frac{|S^{d-1}|}{d} t^d$.  Moreover, by the inverse function theorem, $U(t,s)$ is differentiable everywhere except at $s=0$ for $d>1$ and at $s=\frac{|S^{d-1}|}{d} t^d$.  Suppose $\tilde{\epsilon}(x)$ is the distance from $x \in M$ to $\partial M$ as defined in \eqref{Definition tildeepsilon gamma}. Let  
\begin{align*}
\tilde{R}(x)=U(\tilde{\epsilon}(x), \frac{K+1}{P(x)n}).
\end{align*}
We show that $\tilde{R}(x)$ is an estimator of $R(x)$. The proof of the proposition is in Section \ref{proof prop BI 2} of the Supplementary Material.

\begin{proposition}\label{tildeR and R}
{Under Assumptions \ref{main assumption 1} and \ref{main assumption 2},} suppose we have $\frac{K}{n} \rightarrow 0$ and $\frac{\log(n)}{K} (\frac{n}{K})^{2/d}\to 0$ as $n \to \infty$.  Then, for all $x \in M$, with probability greater than $1-2n^{-2}$,  we have $R(x)=\tilde{R}(x)(1+O((\frac{K}{n})^{\frac{1}{d}}))$, where the constant in $O((\frac{K}{n})^{\frac{1}{d}})$ depends on $d$, $C^1$ norm of $P$, and $P_m$.  Moreover, 
$$\frac{1}{2} (\frac{d}{|S^{d-1}|})^{\frac{1}{d}}(\frac{K}{P_M n})^{\frac{1}{d}} \leq R^* \leq 3 (\frac{2d}{|S^{d-1}|})^{\frac{1}{d}}(\frac{K}{P_m n})^{\frac{1}{d}}$$.
\end{proposition} 

Observe that for any $z_k \in \mathcal{X}$, $C_{n,k}$ constructed through the KNN scheme is equal to the $C_{n,k}$ constructed through the $R(x_k)$-radius ball scheme. Hence, by applying Proposition \ref{tildeR and R}, we provide the bias and variance analyses of the local covariance matrix in the KNN scheme. The proof of theorem is in Section \ref{proof THM BI 2}  of the Supplementary Material.

\begin{theorem}\label{Covariance KNN setup}
Under Assumptions \ref{main assumption 1}, \ref{main assumption 2}, and \ref{assumption tangent space}, let $\frac{1}{n}C_{n,k}$ be the local covariance matrix at $z_k$ constructed in the KNN scheme, where $C_{n,k}$ is defined in \eqref{local covariance matrix}.  Suppose $K=K(n)$ so that  $\frac{K}{n} \rightarrow 0$ and $\frac{\log(n)}{K} (\frac{n}{K})^{2/d}\to 0$ as $n \to \infty$. Then, with probability greater than $1-4n^{-2}$,  for all $k$, 
\begin{align*}
\frac{1}{n} C_{n,k}= {
\begin{bmatrix}
\tilde{M}^{(0)}(x_k) & 0 \\
0& 0 
\end{bmatrix}} 
+\begin{bmatrix}
\tilde{M}^{(11)}(x_k, \frac{K}{n}) & \tilde{M}^{(12)}(x_k, \frac{K}{n}) \\
\tilde{M}^{(21)}(x_k, \frac{K}{n}) & 0 
\end{bmatrix}+O( (\frac{K}{n})^{\frac{d+4}{d}}) +O\Big(\frac{\sqrt{K \log(n)}}{n} (\frac{K}{n})^{\frac{2}{d}}\Big). 
\end{align*}
The properties of the block matrices are summarized as follows.
\begin{enumerate}[label=(\arabic*)]
\item
For $x \in M$, $\tilde{M}^{(0)}(x) \in \mathbb{R}^{d \times d}$ is a diagonal matrix. 
\begin{enumerate}
\item The $i$-th diagonal entry of $\tilde{M}^{(0)}(x)$ is $\mu_1(x)+O((\frac{K}{n})^{\frac{d+3}{d}})$, for $i=1, \cdots, d-1$. The $d$th diagonal entry of $\tilde{M}^{(0)}(x)$ is $
\mu_2(x)+O((\frac{K}{n})^{\frac{d+3}{d}})$.
\item $\mu_1(x)$ and $\mu_2(x)$ are continuous functions on $M$. For all $x \in M$,
$$\frac{1}{2(d+2)}(\frac{d}{|S^{d-1}| P_M}) ^{\frac{2}{d}} (\frac{K+1}{n})^{\frac{d+2}{d}}\leq \mu_1(x), \mu_2(x) \leq \frac{2}{d+2}(\frac{2d}{|S^{d-1}| P_m}) ^{\frac{2}{d}} (\frac{K+1}{n})^{\frac{d+2}{d}}.$$
\item
When $x \in \partial M$, 
 $$\mu_1(x) =\mu_2(x)=\frac{1}{(d+2)}(\frac{2d}{|S^{d-1}| P(x)}) ^{\frac{2}{d}} (\frac{K+1}{n})^{\frac{d+2}{d}}.$$
\end{enumerate}
\item 
$\tilde{M}^{(11)}(x_k, \frac{K}{n})$ is symmetric and $\tilde{M}^{(12)}(x_k, \frac{K}{n})=\tilde{M}^{(21)}(x_k, \frac{K}{n})^\top$. The entries in those matrices are of order $O((\frac{K}{n})^{\frac{d+3}{d}})$, where the constant in $O((\frac{K}{n})^{\frac{d+3}{d}})$ depends on $d$, $P_m$, the $C^1$ norm of $P$, the second fundamental form of $\iota(M)$ in $\mathbb{R}^p$ at $\iota(x_k)$, and the second fundamental form of $\partial M$ in $M$ at $x_{\partial,k}$.
\item
For $x_k \in M \setminus  M_{R^*}$,
\begin{align*}
\frac{1}{n} C_{n,k}= {
\begin{bmatrix}
\tilde{M}^{(0)}(x_k) & 0 \\
0& 0 
\end{bmatrix}} 
+O( (\frac{K}{n})^{\frac{d+4}{d}}) +O\Big(\frac{\sqrt{K \log(n)}}{n} (\frac{K}{n})^{\frac{2}{d}}\Big). 
\end{align*}
The $i$-th diagonal entry of $\tilde{M}^{(0)}(x_k) $ is 
$\frac{1}{(d+2)}(\frac{d}{|S^{d-1}| P(x_k)}) ^{\frac{2}{d}} (\frac{K+1}{n})^{\frac{d+2}{d}}+O((\frac{K}{n})^{\frac{d+3}{d}}),$ 
for $1 \leq i \leq d$.
\item $O( (\frac{K}{n})^{\frac{d+4}{d}})$ and $ O\Big(\frac{\sqrt{K \log(n)}}{n} (\frac{K}{n})^{\frac{2}{d}}\Big)$ represent $p \times p $ matrices whose entries are of orders $O( (\frac{K}{n})^{\frac{d+4}{d}})$ and  $ O\Big(\frac{\sqrt{K \log(n)}}{n} (\frac{K}{n})^{\frac{2}{d}}\Big)$ respectively.
\end{enumerate}
\end{theorem}
Under Assumptions \ref{main assumption 1} and \ref{main assumption 2}, since the eigenvalues $\{\lambda_{n,i}(z_k)\}_{i=1}^p$ of $C_{n, k}$ are invariant under translation of $\mathcal{X}$ and orthogonal transformation on $\mathbb{R}^p$, and based on the above theorem and a perturbation argument (Appendix A in \cite{Wu_Wu:2017}), the eigenvalues $\{\lambda_{n,i}(z_k)\}_{i=1}^p$ of $C_{n, k}$ constructed in the KNN scheme can be characterized as follows for all $k$.
\begin{align*}
& \frac{\lambda_{n,i}(z_k)}{n}=\mu_1(x_k) +O( (\frac{K}{n})^{\frac{d+3}{d}}) +O\Big(\frac{\sqrt{K \log(n)}}{n} (\frac{K}{n})^{\frac{2}{d}}\Big) & &\text{for $i=1, \cdots, d-1$};\\
& \frac{\lambda_{n,i}(z_k)}{n}= \mu_2(x_k)+O( (\frac{K}{n})^{\frac{d+3}{d}}) +O\Big(\frac{\sqrt{K \log(n)}}{n} (\frac{K}{n})^{\frac{2}{d}}\Big) & &\text{for $i=d$}; \\
& \frac{\lambda_{n,i}(z_k)}{n}= O( (\frac{K}{n})^{\frac{d+4}{d}}) +O\Big(\frac{\sqrt{K \log(n)}}{n} (\frac{K}{n})^{\frac{2}{d}}\Big) & &\text{for $i=d+1, \cdots p$}.
\end{align*}
In particular, when $x_k \in M \setminus  M_{R^*}$, $\mu_1(x_k)=\mu_2(x_k)=\frac{1}{(d+2)}(\frac{d}{|S^{d-1}| P(x_k)}) ^{\frac{2}{d}} (\frac{K+1}{n})^{\frac{d+2}{d}}$.

Moreover, under Assumption \ref{assumption tangent space}, suppose $U_{n,k} \in O(p)$ is the corresponding orthonormal eigenvector matrix of $C_{n,k}$.  For any $k$,
\begin{align*}
U_{n,k}=\begin{bmatrix}
X_{k,1}& 0 \\
0& X_{k,2}
\end{bmatrix}+O( (\frac{K}{n})^{\frac{d+3}{d}}) +O\Big(\frac{\sqrt{K \log(n)}}{n} (\frac{K}{n})^{\frac{2}{d}}\Big),
\end{align*}
where $X_{k,1} \in O(d)$ and $ X_{k,2} \in O(p-d)$.
If $x_k \in M \setminus M_{R^*}$, then
\begin{align*}
U_{n,k}=\begin{bmatrix}
X_{k,1}& 0 \\
0& X_{k,2}
\end{bmatrix}+O( (\frac{K}{n})^{\frac{d+4}{d}}) +O\Big(\frac{\sqrt{K \log(n)}}{n} (\frac{K}{n})^{\frac{2}{d}}\Big).
\end{align*}

To end this subsection, we provide the following bias and variance analyses of the BI in the KNN scheme. The proof of the theorem is in Section \ref{proof THM BI 2} of the Supplementary Material.
\begin{theorem}\label{prop BI 2}
Under Assumptions \ref{main assumption 1}, \ref{main assumption 2},  and the KNN scheme, let $c=n (\frac{K}{n})^{\frac{d+3}{d}}$ and suppose $K=K(n)$ so that  $\frac{K}{n} \rightarrow 0$ and $\frac{\log(n)}{K} (\frac{n}{K})^{2/d}\to 0$ as $n \to \infty$. Then, with probability greater than $1-4n^{-2}$,  for all $k$, 
\begin{align}
B_k= \tilde{B} (x_k)+O((\frac{K}{n})^{\frac{1}{d}})+O\Big(\sqrt{\frac{\log(n)}{K}}\Big), \nonumber
\end{align}
The constants in $O((\frac{K}{n})^{\frac{1}{d}})$ and $O\Big(\sqrt{\frac{\log(n)}{K}}\Big)$ depend on $P_m$,  the $C^1$ norm of $P$ and the second fundamental form of $\iota(M)$. The function $\tilde{B}(x):M\to \mathbb{R}$ has the following properties:
\begin{enumerate}[label=(\arabic*)]
\item $\tilde{B}(x)$ is continuous on $M$.
\item  Define $\frac{|S^{d-2}|}{d-1}=1$ when $d=1$. $ \tilde{B}(x)=\frac{4d^2(d+2)|S^{d-2}|^2}{(d^2-1)^2|S^{d-1}|^2}$ when $x\in \partial M$.
\item  If $n$ is large enough, then $R^*$ is small enough and $\gamma_x(t)$ is mimizing on $[0, 2R^*]$ for all $x \in \partial M$. There exists $0 <t_x^*<R^*$ with $ \tilde{B}(\gamma_x(t))=0$ for $t \geq t_x^*$ and $ \tilde{B}(\gamma_x(t))$ decreasing for $0<t<t_x^*$. 
\end{enumerate}
\end{theorem}

We discuss the implications of the above results concerning the BI in the KNN scheme. By (2) and (3), $\tilde{B}(x)$ remains constant and attains maximum on $\partial M$. Furthermore, for any point $x$ on $\partial M$, $\tilde{B}(\gamma_x(t))$ decreases along the geodesic $\gamma_x(t)$ from $x$ to $\tilde{B}(\gamma_x(t_x^*))$ and $\tilde{B}(\gamma_x(t))=0$ when $t \geq t_x^*$.  Since $t_x^* < R^*$, the region where $\tilde{B}(x)$ is non zero is contained in $M _{R^*}$. Hence, $\tilde{B}(x)$ behaves like a bump function, concentrating on $\partial M$ and vanishing in $M \setminus  M_{R^*}$. However, unlike $B(x)$ in the $\epsilon$-radius ball scheme, $\tilde{B}(x)$ in the KNN scheme may not decrease at the same speed along geodesics perpendicular to $\partial M$. Refer to Figure \ref{fig: KNN B function} for an illustration.  If we choose $0<\tau<\frac{4d^2(d+2)|S^{d-2}|^2}{(d^2-1)^2|S^{d-1}|^2}$, $\tilde{B}^{-1}(\tau, \infty)=N_{\tau}$, where $N_{\tau} \subset M_{R^*}$ is a neighborhood of $\partial M$ in $M$.

Suppose $K$ and $n$ satisfy the conditions in Theorem \ref{prop BI 2}. For sufficiently large $n$, with high probability, we have {$|B_k-\tilde{B}(x_k)|=O((\frac{K}{n})^{\frac{1}{d}})+O\Big(\sqrt{\frac{\log(n)}{K}}\Big)$, which is small enough.} Therefore, when $x_k$ is near the boundary, $B_k$ approximates a constant value, and within $M \setminus  M_{R^*}$, $B_k$ is of order {$O((\frac{K}{n})^{\frac{1}{d}})+O\Big(\sqrt{\frac{\log(n)}{K}}\Big)$}.
Thus, we can set a threshold on $B_k$ to identify all points in a neighborhood of the boundary contained in $M_{R^*}$. According to Proposition \ref{tildeR and R}, as $n$ increases, $R^*$ deceases, leading to more precise identification of boundary points.
\begin{figure}[ht]
\centering
\includegraphics[scale=0.23]{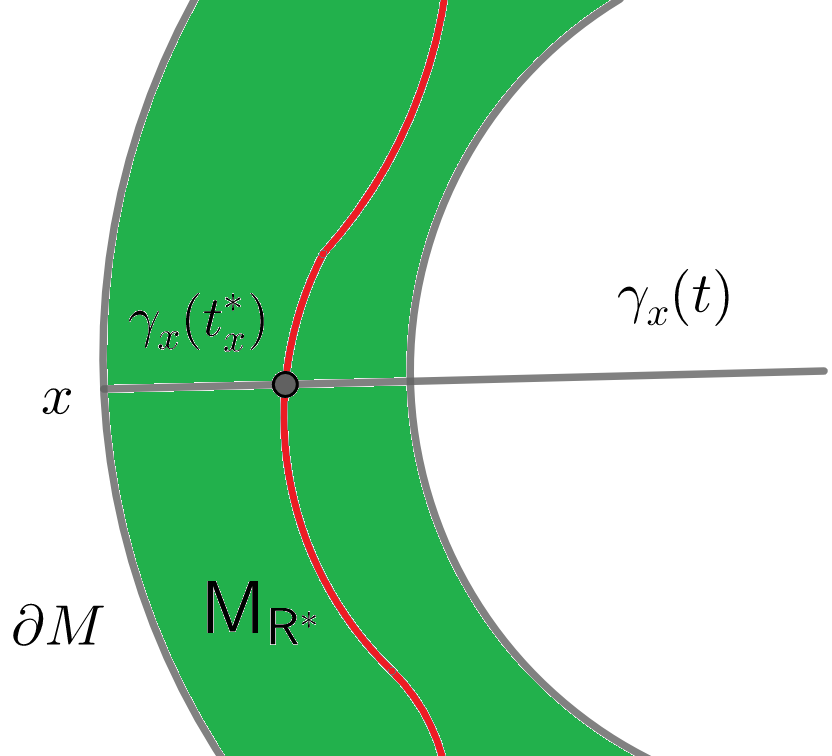}
\caption{An illustration to the function $ \tilde{B}(x)$ in the KNN scheme in a neighborhood near the boundary of $M$.  The green region is the intersection of $M_{R^*}$ and the neighborhood. The red curve is the union of $\gamma_x(t_x^*)$ corresponding to all $x \in \partial M$. For any $x \in \partial M$, $ \tilde{B}(\gamma_x(t))$ decreases along the geodesic $\gamma_x(t)$ from $x$ to $ \tilde{B}(\gamma_x(t_x^*))$ and $ \tilde{B}(\gamma_x(t))=0$ when $t \geq t_x^*$.   Since $t_x^* \leq R^*$, $\tilde{B}=0$ on  $M \setminus  M_{R^*}$.}
\label{fig: KNN B function}
\end{figure}

\section{Numerical results}\label{numerical simulation}
In this section, we compare the performances of BD-LLE in {four} examples with different boundary detection algorithms including $\alpha$-shape \cite{Edelsbrunner:1983,Edelsbrunner:1994}, BAND \cite{xue2009boundary}, BORDER \cite{Xia:2006}, BRIM \cite{Qiu:2007},  LEVER \cite{cao2018multidimensional}, SPINVER \cite{qiu2016clustering}, and the CPS algorithm \cite{calder2022boundary}(abbreviated by authors' initials for brevity). Detailed descriptions and discussions of all the algorithm are summarized in Section \ref{Summary algorithms} of the Supplementary Material, where each algorithm is presented along with the notations and setups used in this work. Note that all algorithms, except $\alpha$-shape, require either the $\epsilon$-radius ball scheme or the KNN scheme for nearest neighborhood search.  For $\alpha$-shape, we apply the boundary function in MATLAB, which includes a shrink factor $s \in[0,1]$ corresponding to $\alpha$. {For the BD-LLE algorithm, we use the $\epsilon$-radius ball search scheme. The scale parameter $\epsilon$ is chosen within the range between $\epsilon_{min}$ and $\epsilon_{max}$ as outlined in Section \ref{selection of eps K}, while the regularizer $c$ is selected according to \eqref{selection of the regularizer 0} or \eqref{selection of the regularizer 1} in Section \ref{Selection of the regularizer}.}

We introduce the following method to evaluate the performance of a boundary detection algorithm. Suppose we fix the scale parameter, $\epsilon$ or $K$, for an algorithm. Let $\partial \mathcal{X}$ denote the boundary points detected from $\mathcal{X}$. Let $M_r$ represent the $r$-neighborhood of $\partial M$ as defined in Definition \ref{preliminary def on boundary}. We define the $F1$ score of $\partial \mathcal{X}$ associated with $r$ as follows:
\begin{equation}\label{F1 def 1}
\text{F}1(\partial \mathcal{X}, r)=\frac{2}
{
\frac{1}{\frac{|\partial \mathcal{X}\cap \iota(M_{r})|}{|\partial \mathcal{X}|}}+
\frac{1}{\frac{|\partial \mathcal{X}\cap  \iota(M_{r})|}{|\mathcal{X} \cap  \iota(M_{r})|}}
}=\frac{2|\partial \mathcal{X} \cap  \iota(M_{r})|}{|\partial \mathcal{X}|+|\mathcal{X} \cap  \iota(M_{r})|}
\end{equation}
Since $\partial M$ is a measure $0$ subset of $M$, based on Assumption \ref{main assumption 2}, the probability for a sample in $\mathcal{X}$ to lie on the boundary of $\iota(M)$ is $0$. Therefore, our objective is to determine whether the detected boundary points $\partial \mathcal{X}$ coincide with the points within some regular neighborhood of the boundary, i.e. the $r$-neighborhood. We propose the metric $F1_{max}$,  defined as the maximum $F1$ score over a sequence $\{r_i=0.05i\}_{i=1}^k$,  i.e. 
\begin{align}\label{F1 def 2}
F1_{max}= \max_{1 \leq i \leq k}F1(\partial \mathcal{X}, r_i).
\end{align}
Note that unlike other algorithms, the CPS algorithm detects the boundary points through directly estimating the distance to the boundary function. In other words, for $z_k \in \mathcal{X}$ close to the boundary, $d_g(\iota^{-1}(z_k),\partial M)$ is estimated. By introducing an additional parameter $r$ alongside the scale parameter, the algorithm outputs $\partial \mathcal{X} (r)$ which estimates $\mathcal{X} \cap \iota(M_r)$. Therefore, for a given sequence $\{r_i=0.05i\}_{i=1}^k$, we apply the CPS algorithm to output the corresponding $\partial \mathcal{X}(r_i)$, and we define $F1_{max}= \max_{1 \leq i \leq k}F1(\partial \mathcal{X}(r_i), r_i)$. 

Next, we describe the construction of the point cloud for each example.

\subsection{Unit ball}
{We uniformly randomly sample $\{r_{i}\}_{i=1}^{8000}$, $\{\theta_{i}\}_{i=1}^{8000}$, and $\{\phi_{i}\}_{i=1}^{8000}$ from $[0,1]$,  $[0,2\pi]$, and $[0,\pi]$ respectively. Let $\mathcal{X}=\{z_i\}_{i=1}^{8000} \subset \mathbb{R}^3$, where
$$z_i=(r^{1/2}_{i}\sin(\phi_{i})\cos(\theta_{i}),r^{1/2}_{i}\sin(\phi_{i})\sin(\theta_{i}),r^{1/2}_{i}\cos(\phi_{i})).$$ 
Thus, we generate $8000$ non-uniform samples $\mathcal{X}=\{z_i\}_{i=1}^{8000}$ on the unit ball in $\mathbb{R}^3$. We apply BD-LLE to $\mathcal{X}$ and compare the result with those from other algorithms. The scale parameters and $F1_{\max}$ for all the algorithms are summarized in Table \ref{table 2} and Table \ref{table 1}. }
\begin{table}
\caption{Summary of the nearest neighborhood search schemes and the scale parameters in different algorithms .}\label{table 2}
\begin{center}
\begin{tabular}{ |c|c|c|c|c|c|} 
 \hline
Algorithms& Nearest Neighborhood & Unit ball& V-cut torus& T-cut torus& Klein bottle \\
 \hline
BD-LLE& $\epsilon$-radius ball & $\epsilon=0.2$ &  $\epsilon=1$ & $\epsilon=1.15$ & $\epsilon=0.25$ \\
 \hline
$\alpha$-shape& Shrink factor & 1 & 0 & 0& NA  \\
 \hline
BAND& KNN&  K=50& K=50 & K=50&  K=50 \\
 \hline
BORDER& KNN&  K=50& K=50 & K=50& K=50 \\
 \hline
BRIM & $\epsilon$-radius ball & $\epsilon=0.2$ &  $\epsilon=1$ & $\epsilon=1.15$ & $\epsilon=0.25$ \\
 \hline
CPS& $\epsilon$-radius ball & $\epsilon=0.2$ &  $\epsilon=1$ & $\epsilon=1.15$ & $\epsilon=0.25$ \\
 \hline
LEVER& KNN&  K=50& K=50 & K=50 & K=50 \\
 \hline
SPINVER&  KNN&  K=50& K=50 & K=50 & K=50 \\
 \hline
\end{tabular}
\end{center}
\end{table}

\begin{table}
\caption{Summary of $F1_{max}$ for all the algorithms. The largest $F1_{max}$ in each example is highlighted.}\label{table 1}
\begin{center}
\begin{tabular}{ |c|c|c|c|c| } 
 \hline
Algorithms& Unit ball& V-cut torus& T-cut torus& Klein bottle \\
 \hline
BD-LLE&  0.8705 &  {\color{red}0.9344} & {\color{red}0.7840}& {\color{red}0.8425} \\
 \hline
$\alpha$-shape& 0.8370& 0.1511& 0.2096& NA \\
 \hline
BAND& 0.0800& 0.3679& 0.3491& 0.2959 \\
 \hline
BORDER& 0.6507& 0.4895& 0.3833& 0.3176 \\
 \hline
BRIM & 0.5289 & 0.1238& 0.1017 & 0 \\
 \hline
CPS& {\color{red}0.8876}& 0.9022&  0.5810 & 0.7862 \\
 \hline
LEVER& 0.6472& 0.6679 & 0.5609& 0.5185 \\
 \hline
SPINVER & 0.3194& 0.5607 & 0.3313& 0.2913 \\
 \hline
\end{tabular}
\end{center}
\end{table}

\subsection{Vertical-cut (V-cut) torus}
We uniformly randomly sample $\{\theta_{i}\}_{i=1}^{5056}$ and $\{\phi_{j}\}_{i=1}^{5056}$ from $[-\pi,\pi)$ and $[-\pi,\pi) \setminus (-0.5,0.5)$ respectively. Let $\mathcal{X}=\{z_i\}_{i=1}^{5056}\subset \mathbb{R}^3$, where
\begin{equation*}
z_i=(3+1.2\cos(\theta_{i})\cos(\phi_{i}), 3+1.2\cos(\theta_{i})\sin(\phi_{i}), 1.2\sin(\theta_{i}) ).
\end{equation*}
Thus, we generate $5056$ non-uniform samples $\mathcal{X}=\{z_i\}_{i=1}^{5056}$ on the V-cut torus. We apply BD-LLE to $\mathcal{X}$ and compare the result with those from other algorithms. The scale parameters and $F1_{\max}$ for all the algorithms are summarized in Table \ref{table 2} and Table \ref{table 1}.  We plot $\mathcal{X}$ and the detected boundary points $\partial \mathcal{X}$ for each algorithm in Figure \ref{figure: gap torus}.
\begin{figure}
\centering
\SetFigLayout{2}{4}
  \subfigure[BD-LLE]{\includegraphics{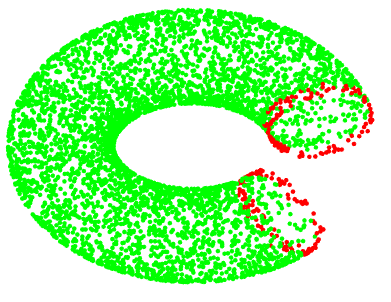}}
   \hfill
   \subfigure[$\alpha$-shape]{\includegraphics{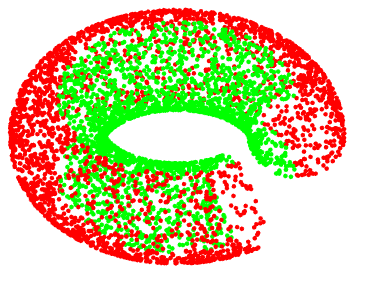}}
  \hfill
  \subfigure[BAND]{\includegraphics{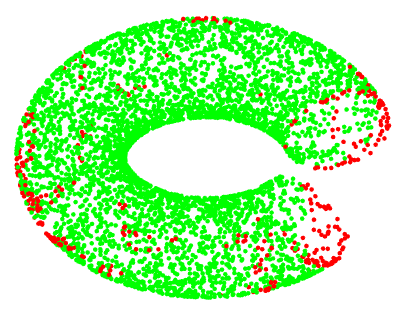}}
  \hfill
  \subfigure[BORDER]{\includegraphics{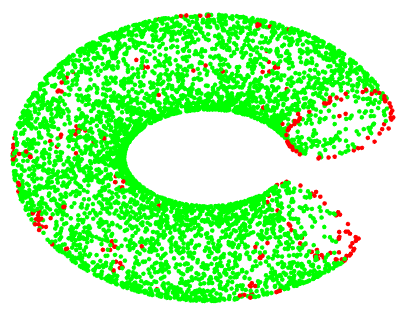}}
  \hfill
  \subfigure[BRIM]{\includegraphics{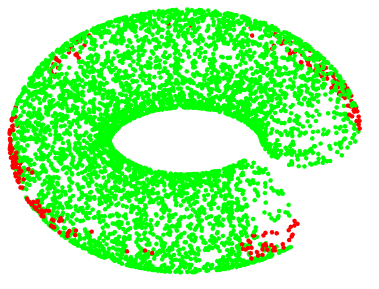}}
  \hfill
  \subfigure[CPS]{\includegraphics{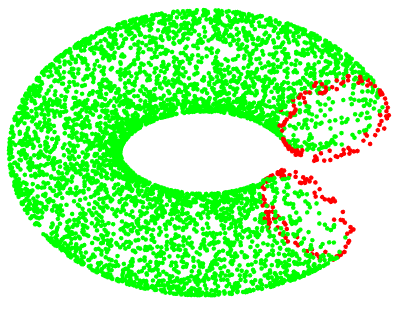}}
  \hfill
  \subfigure[LEVER]{\includegraphics{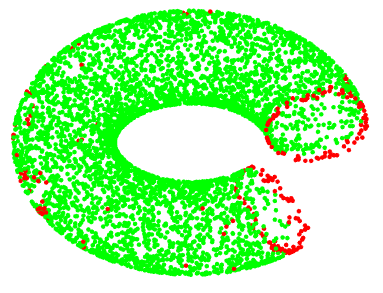}}
  \hfill
  \subfigure[SPINVER]{\includegraphics{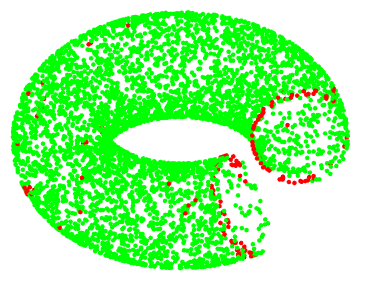}}
  \hfill
\caption{The plot of $\mathcal{X}$ (green and red) and $\partial \mathcal{X}$ (red) for different algorithms in the vertical-cut torus example.}\label{figure: gap torus}
\end{figure}
\subsection{Tilted-cut (T-cut) torus}
We uniformly randomly sample $\{\theta_{j}\}_{j=1}^{8000}$ and $\{\phi_{j}\}_{i=1}^{8000}$ from $[-\pi,\pi)$ respectively. Let $(u_j,v_j,w_j)=( 3+1.2\cos\theta_{j}\cos\phi_{j},3+1.2\cos\theta_{j}\sin\phi_{j},1.2\sin\theta_{j})$ be a point on a torus in $\mathbb{R}^3$. We rotate $\{(x_j,y_j,z_j)\}_{j=1}^{8000}$ around the y-axis through the following map,
$$(u_j,v_j,w_j) \rightarrow (u'_j, v'_j, w'_j)=(\cos(\frac{3\pi}{4}) u_j-\sin(\frac{3\pi}{4}) w_j, v_j, \sin(\frac{3\pi}{4}) u_j+\cos(\frac{3\pi}{4}) w_j).$$
Selecting all the points $\{(u'_j, v'_j, w'_j)\}$ with $w'_j<2.8$ generates $7596$ non-uniform samples $\mathcal{X}=\{z_i\}_{i=1}^{7596}$ on the T-cut torus. We apply BD-LLE to $\mathcal{X}$ and compare the result with those from other algorithms. The scale parameters and $F1_{\max}$ for all the algorithms are summarized in Table \ref{table 2} and Table \ref{table 1}.   We plot $\mathcal{X}$ and the detected boundary points $\partial \mathcal{X}$ for each algorithm in Figure \ref{figure: tilt-cut torus}.
\begin{figure}
\centering
\SetFigLayout{2}{4}
  \subfigure[BD-LLE]{\includegraphics[height=3cm]{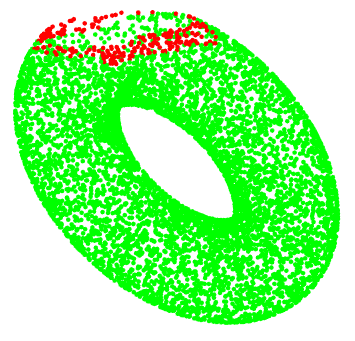}}
  \hfill
  \subfigure[$\alpha$-shape]{\includegraphics[height=3cm]{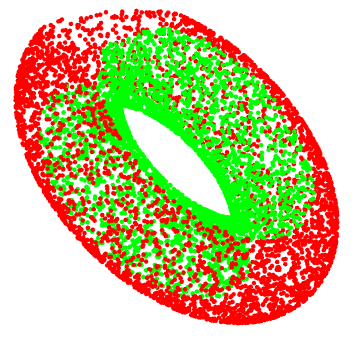}}
  \hfill
  \subfigure[BAND]{\includegraphics[height=3cm]{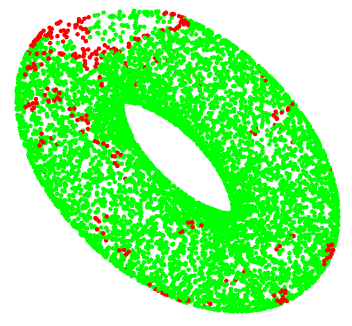}}
  \hfill
  \subfigure[BORDER]{\includegraphics[height=3cm]{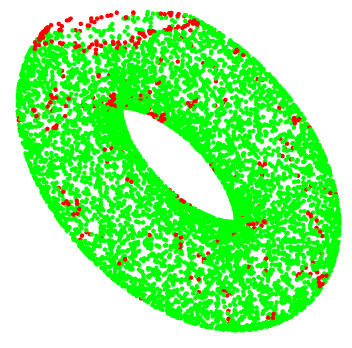}}
  \hfill
   \subfigure[BRIM]{\includegraphics[height=3cm]{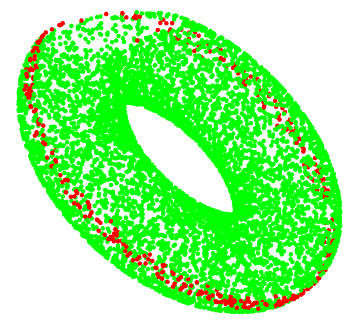}}
  \hfill
\hspace*{0.3cm}\subfigure[CPS]{\includegraphics[height=3cm]{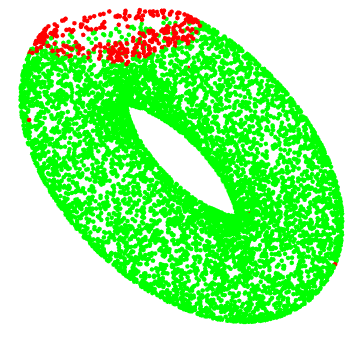}}
  \hfill
\hspace*{0.5cm}\subfigure[LEVER]{\includegraphics[height=3cm]{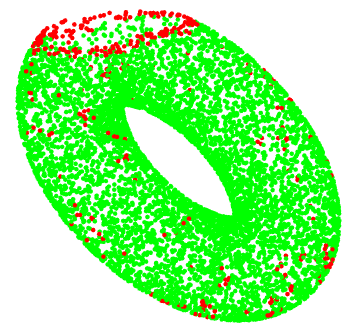}}
  \hfill
\hspace*{0.5cm}\subfigure[SPINVER]{\includegraphics[height=3cm]{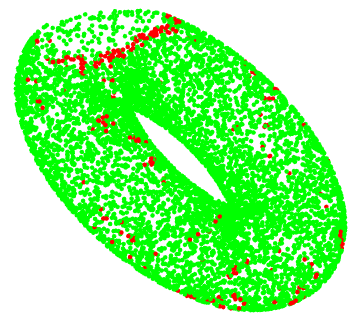}}
  \hfill
    \hfill
\caption{The plot of $\mathcal{X}$ (green and red) and $\partial \mathcal{X}$ (red) for different algorithms in the tilted-cut torus example.}
\label{figure: tilt-cut torus}
\end{figure}

\subsection{Punctured Klein bottle}
{Consider the domain $D=\{(\theta, \phi) |  0 \leq \theta <2\pi, 0\leq \phi <2\pi, (\theta-\pi)^2 + (\phi-\pi)^2 \geq 1\}$. For $(\theta, \phi) \in D$, the parametrization of a punctured Klein bottle in $\mathbb{R}^4$ is given by:
\begin{equation*}
w(\theta, \phi)=\Big((1 +\frac{1}{2}\cos \theta) \cos \phi, (1 +\frac{1}{2}\cos \theta) \sin \phi, \frac{1}{2} \sin \theta \cos \frac{\phi}{2}, \frac{1}{2}\sin \theta \sin \frac{\phi}{2} \Big)
\end{equation*}
By adding $496$ zeros in the coordinates after $w(\theta, \phi)$, we obtain a parametrization of the punctured Klein bottle in $\mathbb{R}^{500}$: $z(\theta, \phi)=(w(\theta, \phi), 0, \cdots, 0)$.  We randomly sample $\{(\theta_{i}, \phi_i)\}_{i=1}^{9689}$ from the domain $D$, so that the corresponding $\mathcal{X}=\{z_i(\theta_{i}, \phi_i)\}_{i=1}^{9689} \subset \mathbb{R}^{500}$ is uniformly distributed on the punctured Klein bottle. A visualization of $\mathcal{X}$ and the region removed from the Klein bottle is shown in Figure \ref{figure: visualize Klein bottle} through the projections $z_i \rightarrow \Big((1 +\frac{1}{2}\cos \theta_i) \cos \phi_i, (1 +\frac{1}{2}\cos \theta_i) \sin \phi_i, \frac{1}{2} \sin \theta_i \cos \frac{\phi_i}{2}\Big) $ and $z_i \rightarrow \Big((1 +\frac{1}{2}\cos \theta_i) \sin \phi_i, \frac{1}{2} \sin \theta_i \cos \frac{\phi_i}{2}, \frac{1}{2}\sin \theta_i \sin \frac{\phi_i}{2} \Big)$. We apply BD-LLE to $\mathcal{X}$ and compare the result with those from other algorithms. The $\alpha$ shape algorithm, implemented through Delaunay triangulation over $\mathcal{X}$ in $\mathbb{R}^{500}$,  is computationally extremely expensive. Hence, it is not included in the comparison. The scale parameters and $F1_{\max}$ for all the algorithms are summarized in Table \ref{table 2} and Table \ref{table 1}. Note that, under the parametrization of the punctured Klein bottle, the boundary corresponds to the unit circle centered at $(\pi, \pi)$ in the domain $D$. For each detected boundary point $z_i(\theta_{i}, \phi_i) \in \partial \mathcal{X}$, we plot the corresponding $(\theta_{i}, \phi_i)$ along with the samples  $\{(\theta_{i}, \phi_i)\}_{i=1}^{9689}$ in the domain $D$ in Figure \ref{figure: punctured Klein bottle}.}

\begin{figure}
 \centering
 \includegraphics[height=4cm, width=1\linewidth]{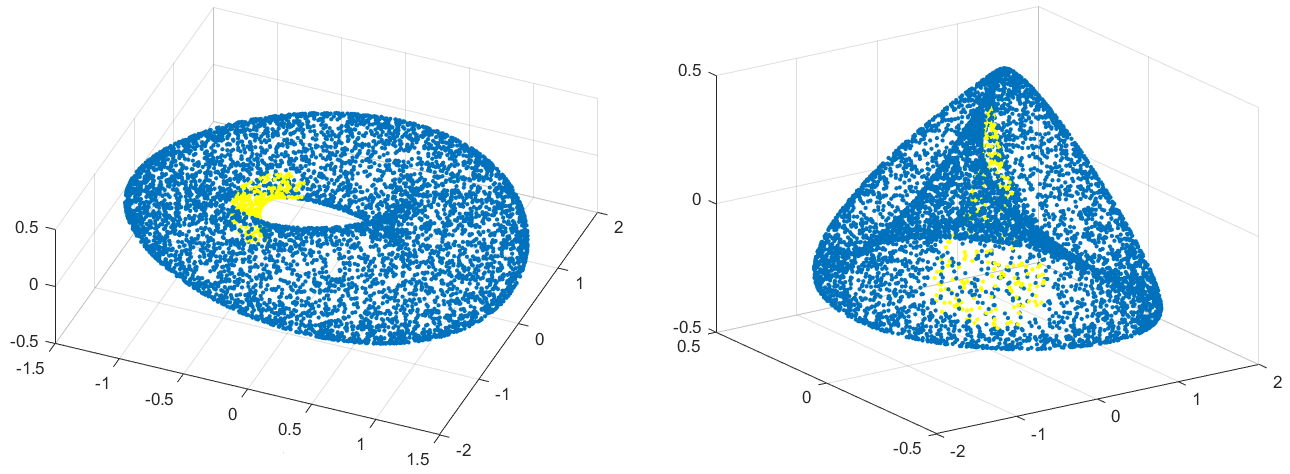}
\caption{{Left and right panels: The blue points represent the projection of points in $\mathcal{X}=\{z_i\}_{i=1}^{9689}$ to $\mathbb{R}^3$ through $z_i \rightarrow ((1 +\frac{1}{2}\cos \theta_i) \cos \phi_i, (1 +\frac{1}{2}\cos \theta_i) \sin \phi_i, \frac{1}{2} \sin \theta_i \cos \frac{\phi_i}{2})$ and $z_i \rightarrow ((1 +\frac{1}{2}\cos \theta_i) \sin \phi_i, \frac{1}{2} \sin \theta_i \cos \frac{\phi_i}{2}, \frac{1}{2}\sin \theta_i \sin \frac{\phi_i}{2})$ respectively.  The yellow points indicate the region removed from the Klein bottle under the same projections. }}
\label{figure: visualize Klein bottle}
\end{figure}
\begin{figure}
\centering
\SetFigLayout{2}{4}
  \subfigure[original data]{\includegraphics{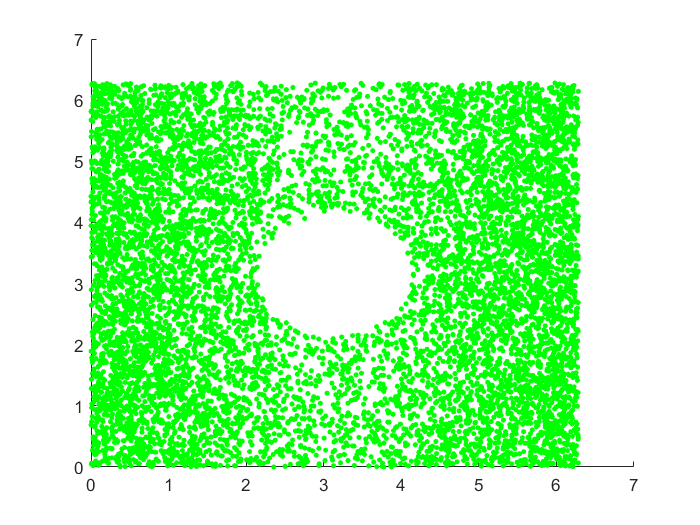}}
   \hfill
   \subfigure[BD-LLE]{\includegraphics{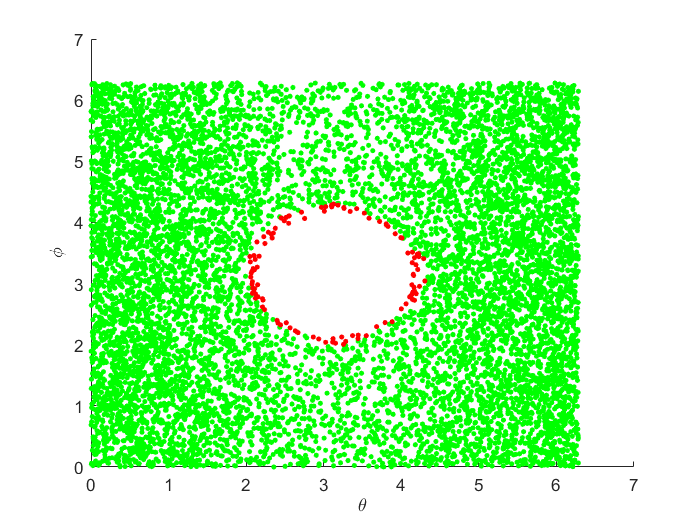}}
  \hfill
  \subfigure[BAND]{\includegraphics{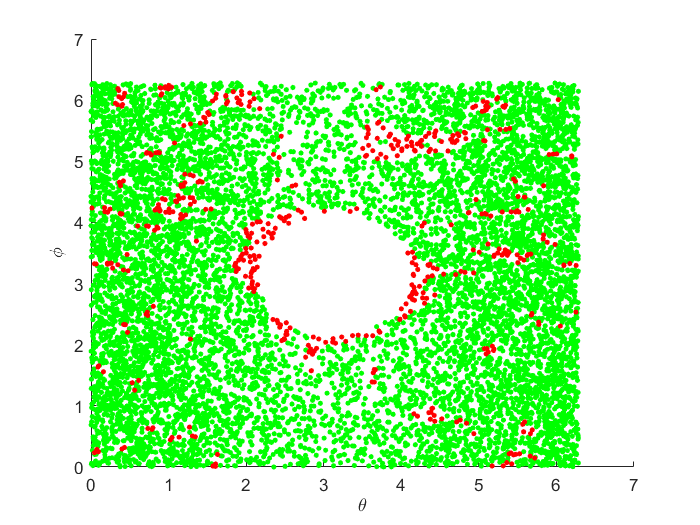}}
  \hfill
  \subfigure[BORDER]{\includegraphics{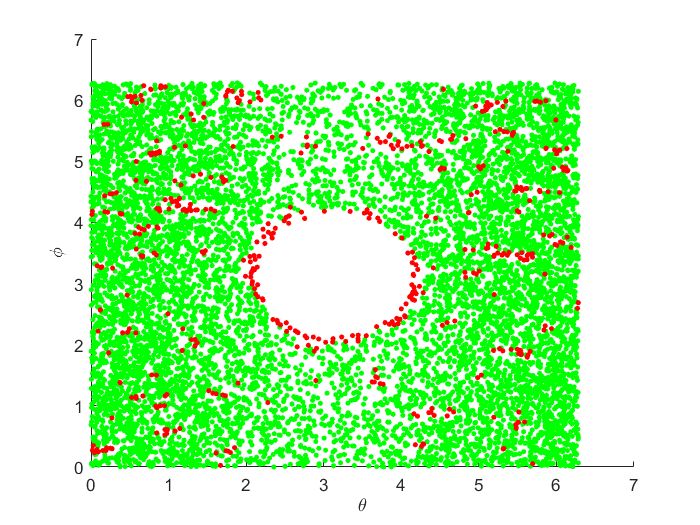}}
  \hfill
  \subfigure[BRIM]{\includegraphics{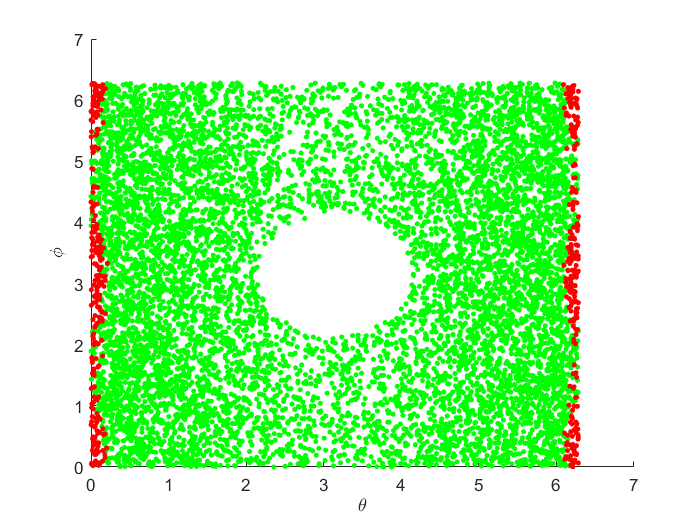}}
  \hfill
\subfigure[CPS]{\includegraphics{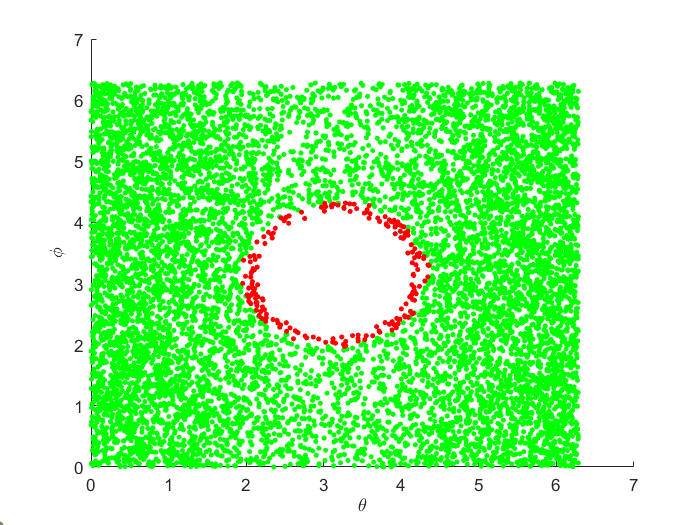}}
  \hfill
 \subfigure[LEVER]{\includegraphics{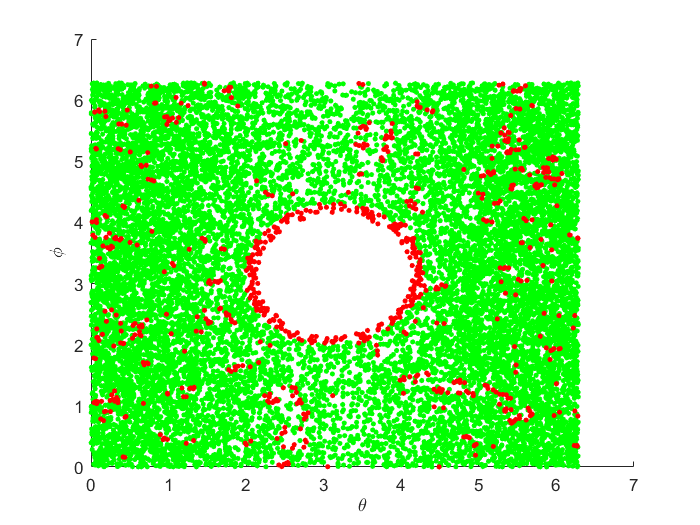}}
  \hfill
 \subfigure[SPINVER]{\includegraphics{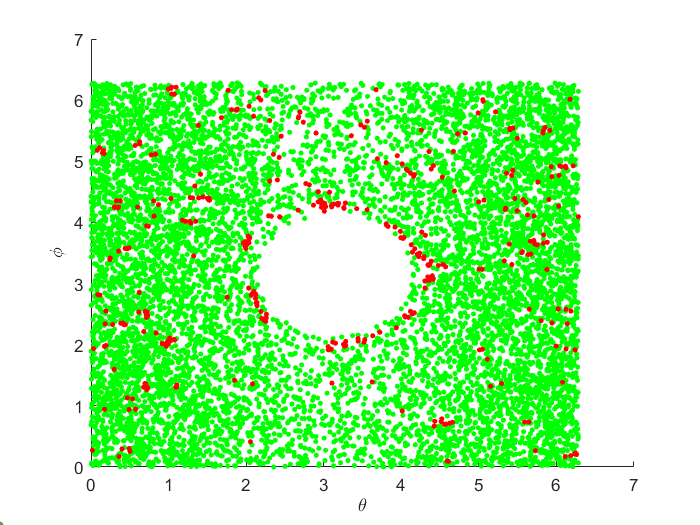}}
  \hfill
  \hfill
\caption{{The plot of the samples  $\{(\theta_{i}, \phi_i)\}_{i=1}^{9689}$ in the domain $D \subset \mathbb{R}^2$  (green and red) in the punctured Klein bottle example.  For each detected boundary point $z_i(\theta_{i}, \phi_i) \in \partial \mathcal{X}$, the corresponding $(\theta_{i}, \phi_i)$  is plotted as a red point. }}\label{figure: punctured Klein bottle}
\end{figure}
In the above results, $\alpha$-shape algorithm can successfully detects the boundary points when the dimension of $M$ equals the dimension of the ambient space, regardless of the distributions of the data.  However, it fails to handle the scenario when the manifold $M$ has a lower dimension. Additionally, the algorithm becomes computationally impractical when the dimension of the ambient space is large. The BAND, BORDER, BRIM,  LEVER, and SPINVER algorithms struggle to detect boundary points due to both the non-uniform distribution of the data and the extrinsic curvature of $\iota(M)$. In contrast, BD-LLE successfully identifies the boundary points in all examples and exhibits the best performance in the V-cut torus, T-cut torus, and Klein bottle examples, regardless of the extrinsic geometry of the manifold and data distributions.

\section{Boundary detection on noisy data}\label{Section: Boundary detection on noisy data}
{Previously, we consider the point cloud  $\mathcal{X}=\{z_i=\iota(x_i)\}_{i=1}^n$, where $\{x_i\}_{i=1}^n$ are sampled from a manifold with boundary $M$, and $M$ is isometrically embedded in $\mathbb{R}^p$ through $\iota$. In this section, for $\{x_i\}_{i=1}^n$ on $M$, we consider a noisy point cloud $\mathcal{X}=\{z_i\}_{i=1}^n \subset \mathbb{R}^p$ where $z_i=\iota(x_i)+\eta_i$ and $\eta_i \stackrel{i.i.d}{\sim} \mathcal{N}(0, \sigma^2 I_{p\times p})$ are sampled independently from $x_i$ for $i=1, \cdots, n$. Our goal is to detect the boundary points $\partial\mathcal{X} \subset \mathcal{X}$, such that $\partial \mathcal{X}$ consists of $\{z_i\}$ corresponding to all $\iota(x_i)$ in a small neighborhood of $\partial \iota(M)$ in $\iota(M)$.

Directly applying boundary detection algorithms may not accurately identify the boundary points of $\iota(M)$ from the noisy point cloud due to several factors. The components of the noise $\eta_i$ perpendicular to $\iota(M)$ at $\iota(x_i)$ cause the noisy points to be distributed in a tubular neighborhood of $\iota(M)$ in $\mathbb{R}^p$, which itself is a $p$-dimensional manifold with boundary.  Thus, interior points may be incorrectly identified as boundary points. Moreover, since $\eta_i$ has components tangent to $\iota(M)$ at $\iota(x_i)$, some boundary points may be displaced into the interior of $\iota(M)$ under the noise, while some interior points may be shifted close to the boundary, further complicating the boundary detection process.

We propose improving boundary detection performance through Diffusion Maps (DM) \cite{coifman2006diffusion}, a dimension reduction technique that constructs a kernel normalized graph Laplacian $L_{DM} \in \mathbb{R}^{n \times n}$ from a point cloud $\{z_i\}_{i=1}^n \subset \mathbb{R}^p$ and a kernel function with bandwidth $\epsilon_{DM}$.  Let $(\lambda^{DM}_i, V_i)_{i=0}^{n-1}$ be the orthonormal eigenpairs of $L_{DM}$ ordered by increasing eigenvalues.  Each $z_i$ is mapped to a low-dimensional space $\mathbb{R}^\ell$ through $z_i \rightarrow \tilde{z}_i=(V_1(i), V_2(i), \cdots, V_\ell(i))$. Refer to Section \ref{review of DM section} of the Supplementary Material for a review of the DM algorithm and {its theoretical foundation of DM for dimension reduction when the point cloud is distributed on a closed manifold.} Recent studies \cite{el2016graph, shen2020scalability, dunson2020diffusion, ding2022impact} show that DM is robust to noise. Moreover, when applied to clean points on $\iota(M)$, we expect the map $(V_1, V_2, \cdots, V_\ell)$  approximates a discretization of a diffeomorphism from $\iota(M)$ to an embedded manifold with boundary $\tilde{\iota}(\tilde{M})$ in $\mathbb{R}^\ell$ over the clean points. Hence, applying DM to a noisy point cloud $\mathcal{X}$ around $\iota(M)$ produces a much less noisy point cloud  $\mathcal{X}_{DM} =\{\tilde{z}_i\}_{i=1}^n \subset \mathbb{R}^{\ell}$ around $\tilde{\iota}(\tilde{M})$,  establishing a correspondence between points in a small neighborhood of $\partial \iota(M)$ in $\iota(M)$ and those in a small neighborhood of $\partial \tilde{\iota}(\tilde{M})$ in $\tilde{\iota}(\tilde{M})$. A boundary detection algorithm can identify the boundary points $\partial\mathcal{X}_{DM}$ from $\mathcal{X}_{DM}$.  The points  $\partial \mathcal{X}$, consisting of $\{z_i\} \subset \mathcal{X}$ associated with all $\tilde{z}_i$ in $\partial\mathcal{X}_{DM}$, should correspond to $\{\iota(x_i)\}$ in a small neighborhood of $\partial \iota(M)$ in $\iota(M)$, thereby representing the detected boundary points  in $\mathcal{X}$. We illustrate the performance of the proposed method through the following example.  

Consider a surface with boundary $\iota(M)$ in $\mathbb{R}^{500}$ parametrized by
{$$f(u,v)=(u,v,  0.2\sin(2\pi(u^2+v^2)), \cdots, a_1u^2+b_1v^2, a_{22}u^2+b_{22}u^2, 0\cdots,0) \in \mathbb{R}^{500}, \quad u^2+v^2 \leq 1,$$
where $a_j \stackrel{i.i.d}{\sim} \mathcal{N}(0, 0.1^2)$, and $b_j \stackrel{i.i.d}{\sim} \mathcal{N}(0, 0.05^2)$ for $j=1, \cdots, 22$. Thus, $\iota(M)$ is an embedded (curved) disk in $\mathbb{R}^{500}$. We randomly sample $\{(u_{i},v_{i})\}_{i=1}^{7897}$  uniformly on the unit disk} to obtain non-uniform samples $\mathcal{X}_{nn}=\{f(u_i,v_i)\}_{i=1}^{7897}$ on $\iota(M)$. Suppose $\eta_i \stackrel{i.i.d}{\sim} \mathcal{N}(0, \sigma^2 I_{500 \times 500})$ with $\sigma=0.05$ for $i=1, \cdots, 7897$. The noisy point cloud is given by $\mathcal{X}=\{z_i=(f(u_i,v_i)+\eta_i)\}_{i=1}^{7897}$. Refer to Figure \ref{fig: noisy surface} where we plot the projections of $\mathcal{X}_{nn}$ and $\mathcal{X}$ onto their first three coordinates. 

For the detected boundary points $\partial\mathcal{X}$ by an algorithm, we identify the corresponding points $\partial\mathcal{X}_{nn}$ in $\mathcal{X}_{nn}$. The $F1_{max}$ metric of $\partial\mathcal{X}$ is computed by applying $\mathcal{X}=\mathcal{X}_{nn}$ and $\partial\mathcal{X}=\partial\mathcal{X}_{nn}$  in \eqref{F1 def 1} and \eqref{F1 def 2}. This metric evaluates whether the corresponding clean points of the detected boundary points coincide with the points within some $r$-neighborhood of $\partial \iota(M)$. Note that the projection $f(u,v) \in \iota(M) \rightarrow (u,v)$ is a diffeomorphism which maps an $r$-neighborhood of $\partial \iota(M)$ to a $r'$-neighborhood of the unit disk. Therefore, if $z_i$ is detected as a boundary point, we plot the corresponding $(u_i, v_i)$ on the unit disk to better visualize the performance of the boundary point detection. We compare the results of BD-LLE with different boundary detection algorithms. For BD-LLE , we use the $\epsilon$-radius ball search scheme. The scale parameter $\epsilon$ is chosen within the range between $\epsilon_{min}$ and $\epsilon_{max}$ as outlined in Section \ref{selection of eps K}, while the regularizer $c$ is selected according to \eqref{selection of the regularizer 0} in Section \ref{Selection of the regularizer}.
\begin{figure}
    \centering
    \includegraphics[height=3cm, width=1\linewidth]{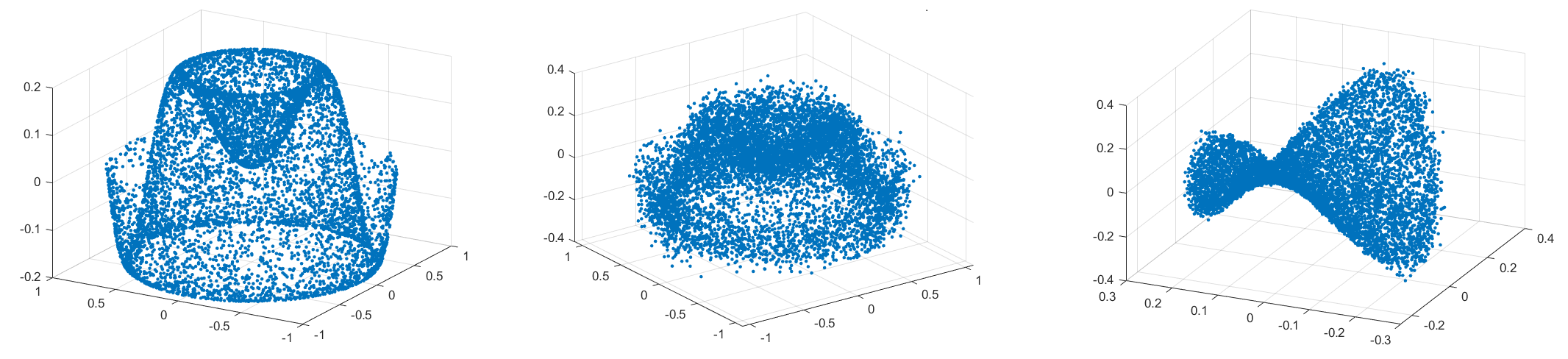}
    \caption{Left and middle panels: Plot of the projection of the clean point cloud $\mathcal{X}_{nn}$ and the noisy point cloud $\mathcal{X}$ onto the first three coordinates respectively.  Right panel: Plot of $\mathcal{X}_{DM}$ , which is constructed by applying DM to $\mathcal{X}$. $\mathcal{X}_{DM}$ is distributed on a saddle surface diffeomorphic to $\iota(M)$ in $\mathbb{R}^3$.}
    \label{fig: noisy surface}
\end{figure}

We directly apply the boundary detection algorithms to $\mathcal{X}$ to identify boundary points $\partial \mathcal{X}_1$. The scale parameters and $F1_{\max}$ for $\partial \mathcal{X}_1$ are summarized  for each algorithm in Table \ref{table noisy parameter}. To illustrate the performances, we plot $(u_i, v_i)$ corresponding to $z_i \in \partial \mathcal{X}_1$, along with $\{(u_{i},v_{i})\}_{i=1}^{7897}$ in Figure \ref{figure: noisy surface}. Due to the challenges discussed previously, all boundary detection algorithms fail to accurately detect the boundary points. Next, we apply the DM to $\mathcal{X}$ {with $\epsilon_{DM}=0.2$}. This creates a map $z_i \in \mathcal{X} \rightarrow \tilde{z}_i=(V_1(i), V_2(i), V_3(i)) \in \mathbb{R}^3$. We then apply the boundary detection algorithms to the lower-dimensional set $\mathcal{X}_{DM}=\{\tilde{z}_i\}_{i=1}^{7897}$ to identify boundary points $\partial \mathcal{X}_{DM}$.  Refer to Figure \ref{fig: noisy surface} for a plot of $\mathcal{X}_{DM}$. The points  $\partial \mathcal{X}_2$, consisting of $\{z_i\}$ associated with all $\tilde{z}_i$ in $\partial\mathcal{X}_{DM}$, represent the detected boundary points  in $\mathcal{X}$. The scale parameters for each algorithm applied to $\partial \mathcal{X}_{DM}$, along with $F1_{\max}$ for $\partial \mathcal{X}_2$, are summarized in Table \ref{table noisy parameter}. We plot $(u_i, v_i)$ corresponding to $z_i \in \partial \mathcal{X}_2$ and  $\{(u_{i},v_{i})\}_{i=1}^{7897}$ for an illustration of the performances in Figure \ref{figure: noisy surface DM}. After applying the DM, the performance of all boundary detection algorithms, except SPINVER, is significantly improved, with BD-LLE exhibiting the best performance.

\begin{table}
\caption{Summary of the scale parameters in different algorithms applied to $\mathcal{X}$ and $\mathcal{X}_{DM}$, as well as $F1_{\max}$ of the detected boundary points $\partial \mathcal{X}_1$ and $\partial \mathcal{X}_2$}\label{table noisy parameter}
\begin{center}
\begin{tabular}{ |c|c|c|c|c|c|} 
 \hline
Algorithms&  $Parameter\hspace{0.6mm}  for\hspace{0.6mm}   \mathcal{X}$ & $F1_{\max}$ of $\partial \mathcal{X}_1$ & $Parameter \hspace{0.6mm}  for\hspace{0.6mm}   \mathcal{X}_{DM}$ &  $F1_{\max}$ of $\partial \mathcal{X}_2$ \\
 \hline
BD-LLE&  $\epsilon=1.6$ & 0.2940 & $\epsilon=0.1$ &   \color{red} 0.7481  \\
 \hline
BAND&  K=90 & 0.1926 & K=90 & 0.6356 \\
 \hline
BORDER&  K=90 & 0.0754 & K=90 & 0.4454\\
 \hline
BRIM &  $\epsilon=1.6$ & 0.0103 &  $\epsilon=0.1$ & 0.4094 \\
 \hline
CPS&  $\epsilon=1.6$ & 0.3964 &  $\epsilon=0.1$ & 0.6924 \\
 \hline
LEVER&   K=90 & 0.1454& K=90 & 0.4484 \\
 \hline
SPINVER&   K=90 & 0.0468 & K=90 & 0.1053\\
 \hline
\end{tabular}
\end{center}
\end{table}

\begin{figure}
\centering
\SetFigLayout{2}{4}
\subfigure[BD-LLE]{\includegraphics{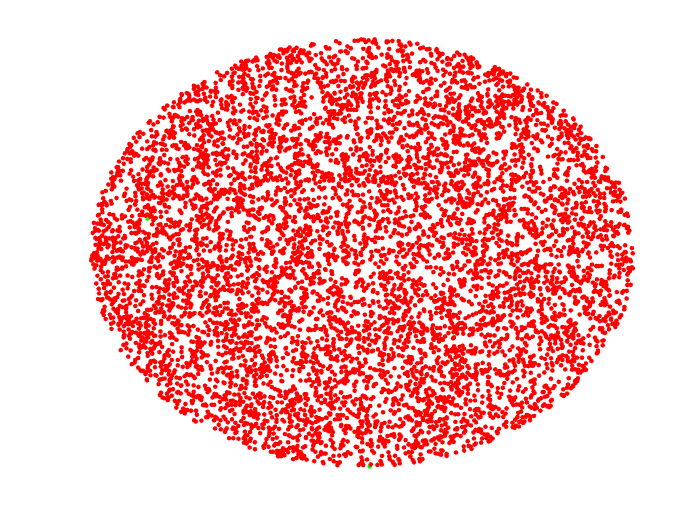}}
\hfill
\subfigure[BAND]{\includegraphics{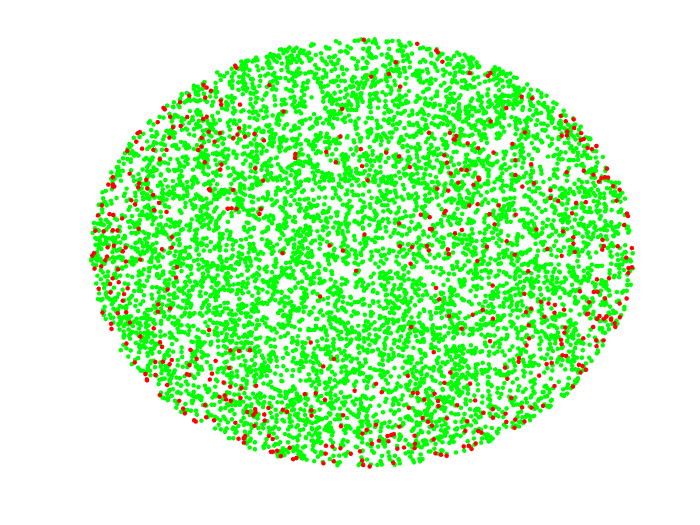}}
  \hfill
  \subfigure[BORDER]{\includegraphics{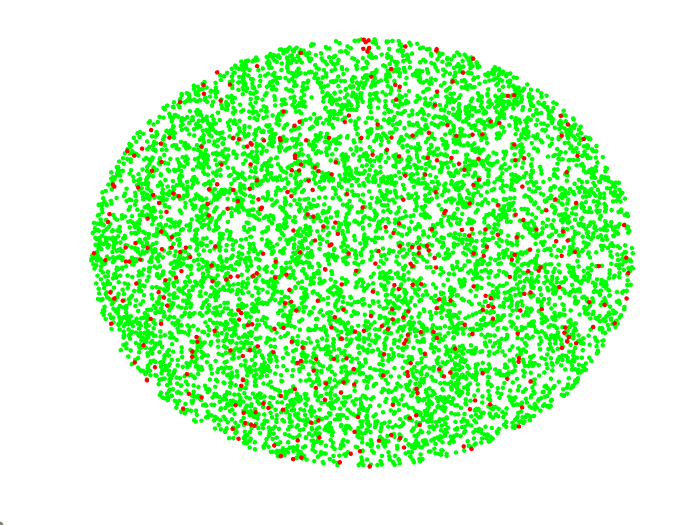}}
  \hfill
  \subfigure[BRIM]{\includegraphics{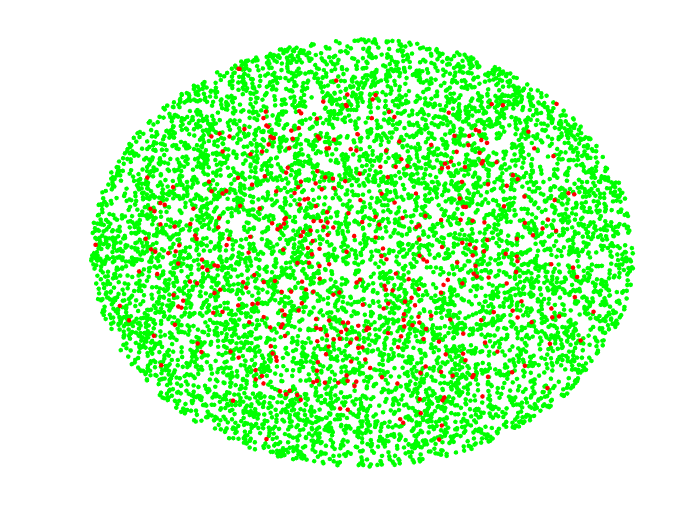}}
  \hfill
  \subfigure[CPS]{\includegraphics{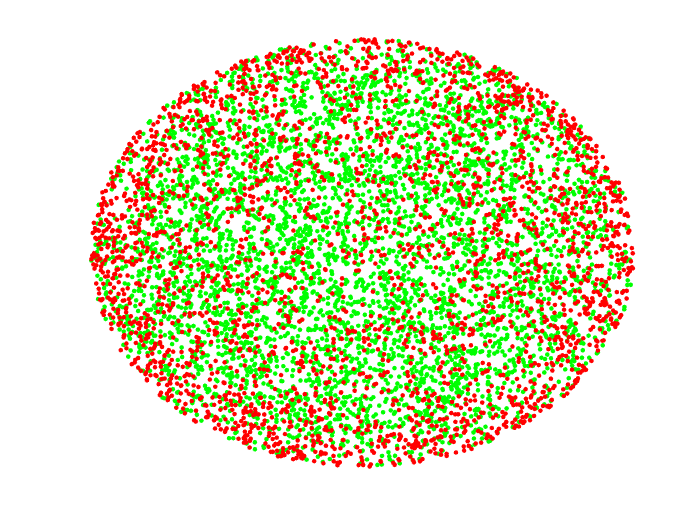}}
  \hfill
\hspace*{-3.5cm} \subfigure[LEVER]{\includegraphics{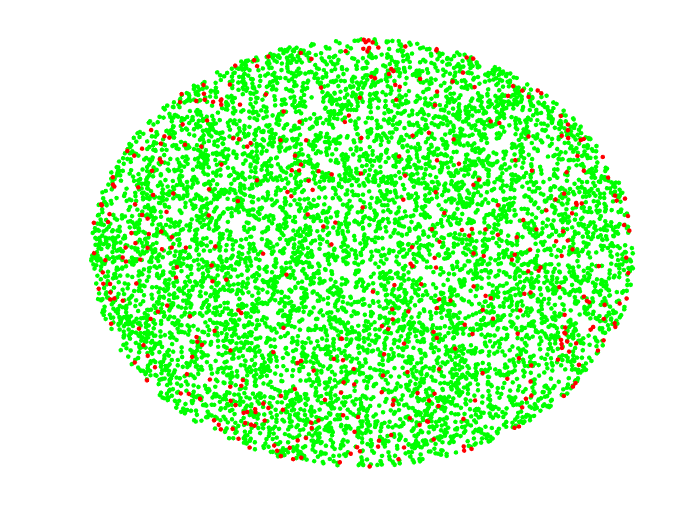}}
  \hfill
\hspace*{-3.5 cm} \subfigure[SPINVER]{\includegraphics{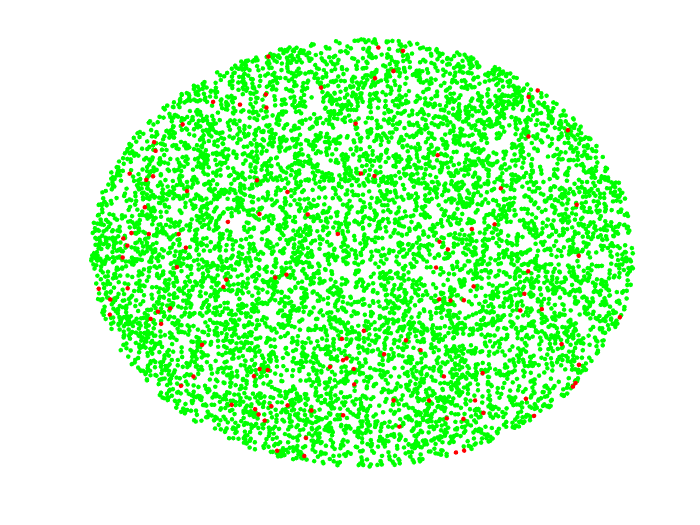}}
  \hfill
     \hfill
\caption{ Plot of $(u_i, v_i)$ (red) corresponding to $z_i $ from the detected boundary points $\partial \mathcal{X}_1$ for different algorithms along with $\{(u_{i},v_{i})\}_{i=1}^{7897}$ (red and green) in the domain of the parametrization of $\iota(M)$. }\label{figure: noisy surface}
\end{figure}

\begin{figure}
\centering
\SetFigLayout{2}{4}
  \subfigure[BD-LLE]{\includegraphics{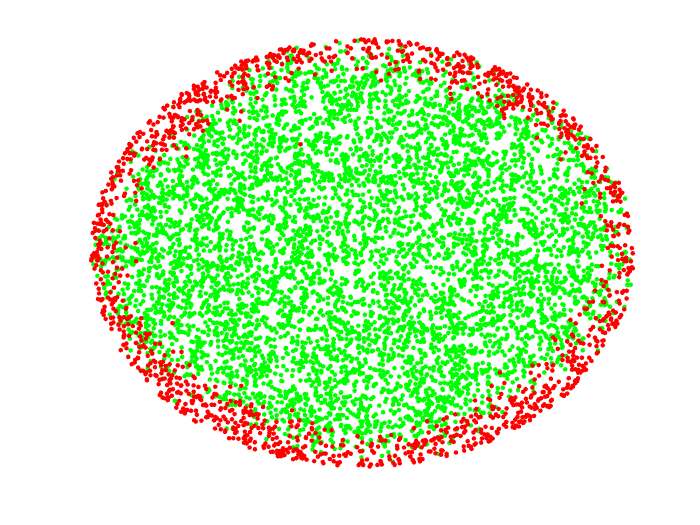}}
   \hfill
   \subfigure[BAND]{\includegraphics{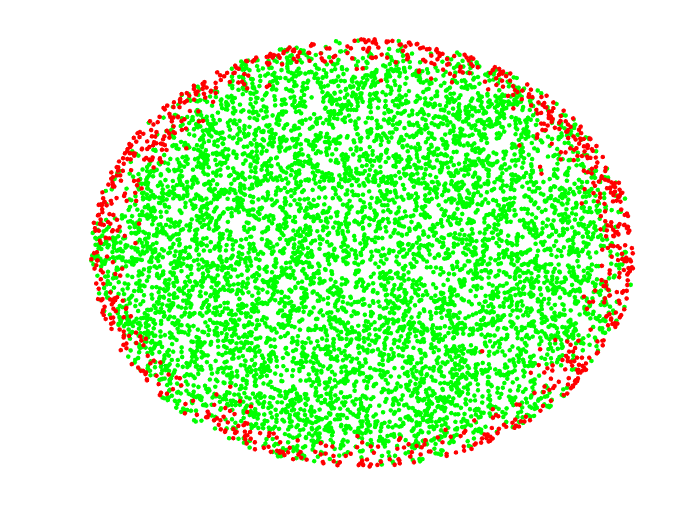}}
  \hfill
  \subfigure[BORDER]{\includegraphics{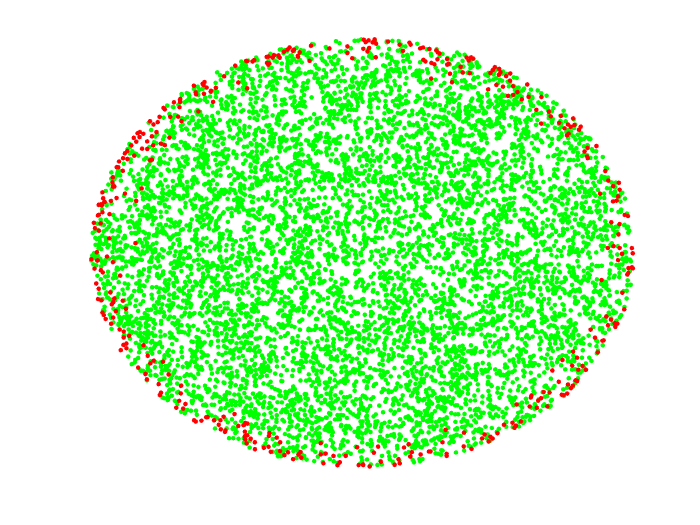}}
  \hfill
  \subfigure[BRIM]{\includegraphics{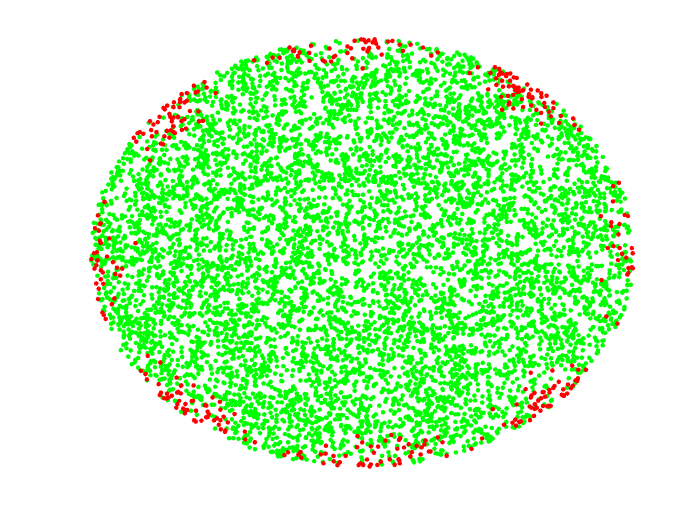}}
  \hfill
  \subfigure[CPS]{\includegraphics{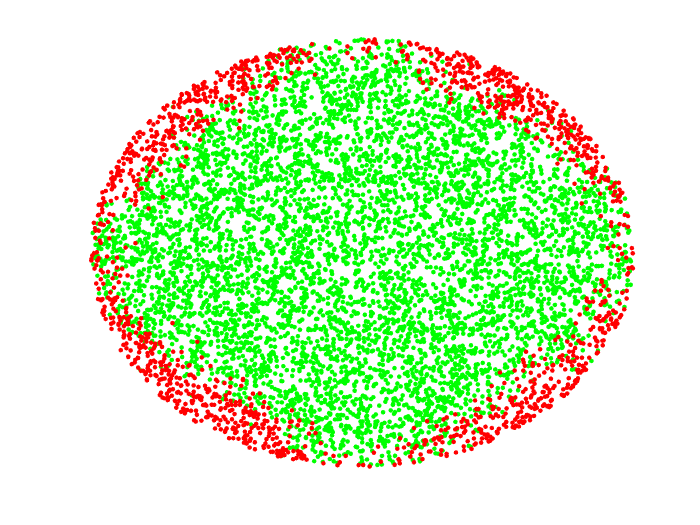}}
  \hfill
 \hspace*{-3.5cm} \subfigure[LEVER]{\includegraphics{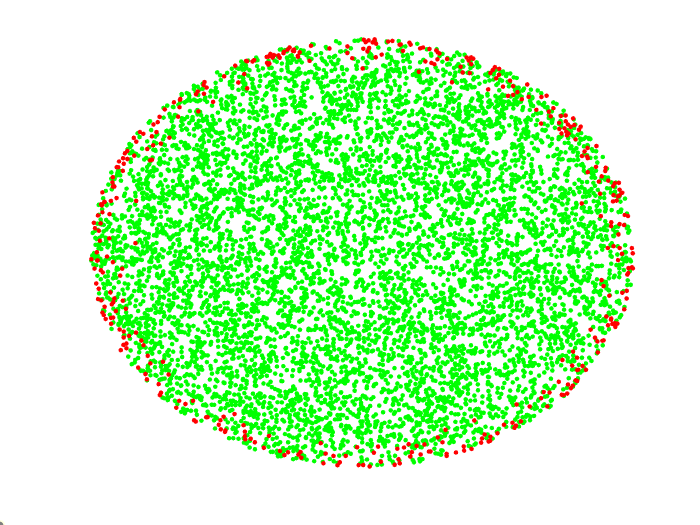}}
  \hfill
\hspace*{-3.5cm}  \subfigure[SPINVER]{\includegraphics{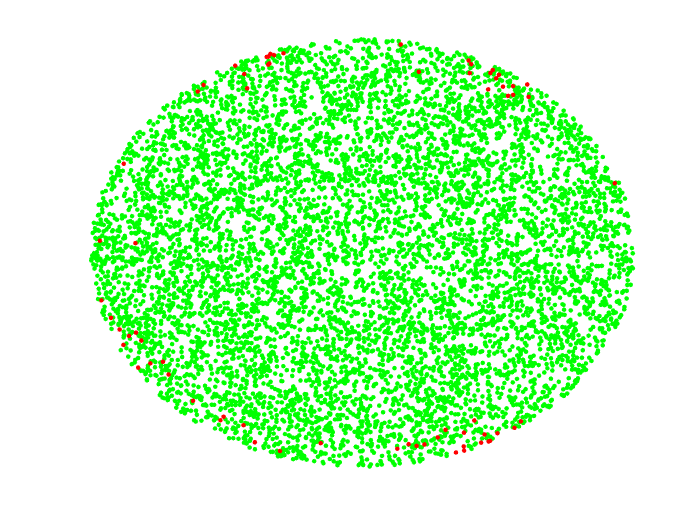}}
  \hfill
   \hfill
\caption{Plot of $(u_i, v_i)$ (red) corresponding to $z_i $ from the detected boundary points $\partial \mathcal{X}_2$ for different algorithms along with $\{(u_{i},v_{i})\}_{i=1}^{7897}$ (red and green) in the domain of the parametrization of $\iota(M)$.}\label{figure: noisy surface DM}
\end{figure}
}

\section{Discussion}
In this work, we delve into the challenge of identifying boundary points from samples on an embedded compact manifold with boundary.  We introduce the BD-LLE algorithm, which utilizes barycentric coordinates within the framework of LLE. This algorithm can be implemented using either an $\epsilon$-radius ball scheme or a KNN scheme. Barycentric coordinates are closely related to the local covariance matrix. We conduct bias and variance analyses of the BD-LLE algorithm under both nearest neighbor search schemes by exploring the spectral properties of the local covariance matrix. These analyses aid in parameter selection. We highlight several potential directions for future research.

The LLE can be considered as a kernel-based dimension reduction method with \eqref{solution y_n} representing an asymmetric kernel function adaptive to the data distribution and the geometry of the underlying manifold. The previous studies \cite{Wu_Wu:2017, wu2018locally} analyze the LLE within the $\epsilon$-radius ball scheme. A future direction involves applying our developed tools to analyze LLE within the KNN scheme. We expect establishing even more challenging results of the spectral convergence of the LLE in the KKN scheme on manifolds with or without boundary.

{In Section \ref{Section: Boundary detection on noisy data}, we apply DM to a noisy point cloud sampled from a compact, embedded manifold with boundary in a high-dimensional space, for the purposes of dimension reduction and denoising. As discussed in Section \ref{review of DM section} of the Supplementary Material, the theoretical foundation for DM is well established when the point cloud lies on a closed manifold; in such cases, DM is known to approximate an embedding of the manifold over the point cloud. Although experimental evidence suggests that DM can also approximate an embedding when the underlying manifold of the point cloud has a boundary, the rigorous theoretical analysis of how DM preserves the manifold structure in this setting remains an open question.}

Another potential direction of research concerns the boundary points augmentation. Given that the boundary is a lower-dimensional subset of the manifold, the limited number of boundary points may not suffice for accurately capturing the geometry of the boundary. Notably, the boundary comprises disjoint unions of closed manifolds without boundary. Hence, one strategy involves applying spectral clustering method to organize  detected boundary points into groups corresponding to different connected components. Closed manifold reconstruction methods \cite{dunson2021inferring} can be further employed on each group to interpolate more points on the boundary. 

\section*{Acknowledgment}
The authors thank Dr. Hau-Tieng Wu for valuable discussions and insightful suggestions. 

\appendix

\section{Analyses of  the local covariance matrix in the $\epsilon$-radius ball scheme}\label{local covariance matrix eps scheme}
We provide the bias and variance analyses of the local covariance matrix in the $\epsilon$-radius ball scheme. The proof of the theorem is a combination of Lemma 31 and Lemma 37 in \cite{wu2018locally}.

\begin{theorem}\label{local covariance epsilon}
Under Assumptions \ref{main assumption 1}, \ref{main assumption 2}, and \ref{assumption tangent space}, let $\frac{1}{n}C_{n,k}$ be the local covariance matrix at $z_k$ constructed in the $\epsilon$-radius ball scheme, where $C_{n,k}$ is defined in \eqref{local covariance matrix}.  Suppose $\epsilon=\epsilon(n)$ such that $\frac{\sqrt{\log(n)}}{n^{1/2}\epsilon^{d/2+1}}\to 0$ and $\epsilon\to 0$ as $n\to \infty$. We have with probability greater than $1-n^{-2}$ that for all $k=1,\ldots,n$,
\begin{align*}
\frac{1}{n\epsilon^{d+2}} C_{n,k}=  P(x_k){
\begin{bmatrix}
M^{(0)}(x_k,\epsilon) & 0 \\
0& 0 
\end{bmatrix}}+{
\begin{bmatrix}
M^{(11)}(x_k,\epsilon) & M^{(12)}(x_k,\epsilon) \\
M^{(21)}(x_k,\epsilon) & 0 
\end{bmatrix}} \epsilon
+O(\epsilon^2) +O\Big(\frac{\sqrt{\log(n)}}{n^{1/2}\epsilon^{d/2}}\Big).
\end{align*}
The block matrices in the above expression satisfy the following properties.
\begin{enumerate}[label=(\arabic*)]
\item $M^{(0)}(x,\epsilon) \in \mathbb{R}^{d \times d}$ is a diagonal matrix. The $i$-th diagonal entry is $\sigma_{2}(\tilde{\epsilon}(x_k) ,\epsilon)$ for $1 \leq i \leq d-1$ and the $d$th diagonal entry is $\sigma_{2,d}(\tilde{\epsilon}(x_k) ,\epsilon)$.
\item $M^{(11)}(x_k,\epsilon)$ is symmetric and  $M^{(12)}(x_k,\epsilon)=M^{(21)}(x_k,\epsilon)^\top$. The entries in $M^{(11)}(x_k,\epsilon)$, $M^{(12)}(x_k,\epsilon)$, and $M^{(21)}(x_k,\epsilon)$ are $0$ when $ x_k \in M \setminus  M_{\epsilon}$.
\item For all $x_k$, the magnitude of the entries in $M^{(11)}(x_k,\epsilon)$, $M^{(12)}(x_k,\epsilon)$, and $M^{(21)}(x_k,\epsilon)$ can be bounded from above by a constant depending on $d$, the $C^1$ norm of $P$, the second fundamental form of $\iota(M)$ in $\mathbb{R}^p$ at $\iota(x_k)$, and the second fundamental form of $\partial M$ in $M$ at $x_{\partial,k}$.
\item $O(\epsilon^2)$ and $O\Big(\frac{\sqrt{\log(n)}}{n^{1/2}\epsilon^{d/2}}\Big)$ represent $p \times p $ matrices whose entries are of orders $O(\epsilon^2)$ and $O\Big(\frac{\sqrt{\log(n)}}{n^{1/2}\epsilon^{d/2}}\Big)$ respectively.
\end{enumerate}
\end{theorem}
Under the assumptions in the above theorem,  suppose $\epsilon=\epsilon(n)$ such that $\frac{\sqrt{\log(n)}}{n^{1/2}\epsilon^{d/2+1}}\to 0$ and $\epsilon\to 0$ as $n\to \infty$. The above theorem implies that with probability greater than $1-n^{-2}$, for all $ x_k \in M \setminus  M_{\epsilon}$,
\begin{align*}
\frac{1}{n\epsilon^{d+2}} C_{n,k}= \frac{ P(x_k) |S^{d-1}|}{d(d+2)}{
\begin{bmatrix}
I_{d \times d}& 0 \\
0& 0 
\end{bmatrix}}+O(\epsilon^2) +O\Big(\frac{\sqrt{\log(n)}}{n^{1/2}\epsilon^{d/2}}\Big).
\end{align*}
This result matches the analysis of the local covariance matrix constructed from samples on a closed manifold.

Since the eigenvalues $\{\lambda_{n,i}(z_k)\}_{i=1}^p$ of $C_{n, k}$ are invariant under translation of $\mathcal{X}$ and orthogonal transformation on $\mathbb{R}^p$. By applying a perturbation argument (Appendix A in \cite{Wu_Wu:2017}), for all $k$
\begin{align*}
& \frac{\lambda_{n,i}(z_k)}{n}= P(x_k)\sigma_{2}(\tilde{\epsilon}(x_k) ,\epsilon) \epsilon^{d+2} +O(\epsilon^{d+3}) +O\Big(\sqrt{\frac{\log(n)}{n}}\epsilon^{d/2+2}\Big) \quad & & \text{for $i=1 ,\cdots, d-1$};\\
& \frac{\lambda_{n,i}(z_k)}{n}= P(x_k)\sigma_{2,d}(\tilde{\epsilon}(x_k) ,\epsilon) \epsilon^{d+2} +O(\epsilon^{d+3}) +O\Big(\sqrt{\frac{\log(n)}{n}}\epsilon^{d/2+2}\Big) & &\text{for $i=d$}; \\
& \frac{\lambda_{n,i}(z_k)}{n}= O(\epsilon^{d+4})  +O\Big(\sqrt{\frac{\log(n)}{n}}\epsilon^{d/2+2}\Big) & &\text{for $i=d+1, \cdots p$}.
\end{align*}
Suppose $U_{n,k} \in O(p)$ is the corresponding orthonormal eigenvector matrix of $C_{n,k}$. Suppose $X_{k,1} \in O(d)$ and $ X_{k,2} \in O(p-d)$. If $x_k \in M_\epsilon$, then
\begin{align*}
U_{n,k}=\begin{bmatrix}
X_{k,1}& 0 \\
0& X_{k,2}
\end{bmatrix}+O(\epsilon) +O\Big(\frac{\sqrt{\log(n)}}{n^{1/2}\epsilon^{d/2}}\Big).
\end{align*}
If $x_k \in M \setminus M_\epsilon$, then
\begin{align*}
U_{n,k}=\begin{bmatrix}
X_{k,1}& 0 \\
0& X_{k,2}
\end{bmatrix}+O(\epsilon^2) +O\Big(\frac{\sqrt{\log(n)}}{n^{1/2}\epsilon^{d/2}}\Big).
\end{align*}

Since we propose Assumption \ref{assumption tangent space}, $\begin{bmatrix}
X_{k,1}\\
0
\end{bmatrix}$ forms an orthonormal basis of $\iota_*T_{x_k}M$, and $\begin{bmatrix}
0\\
X_{k,2}
\end{bmatrix}$ forms an orthonormal basis of $(\iota_*T_{x_k}M)^\bot$. Therefore, when $x_k\in M_\epsilon$, an orthonormal basis of $\iota_*T_{x_k}M$ can be approximated by $U_{n,k}\begin{bmatrix}
I_{d \times d}\\
0_{(p-d) \times d}
\end{bmatrix}$, up to a matrix whose entries are of order $O(\epsilon) +O\Big(\frac{\sqrt{\log(n)}}{n^{1/2}\epsilon^{d/2}}\Big)$.

\section{Proof of Theorem \ref{prop BI 1}}\label{proof prop BI}

\subsection{Preliminary definitions}
Under Assumptions \ref{main assumption 1} and \ref{main assumption 2}, let $X$ be the random variable associated with the probability density function $P$ on $M$. Then, for any function $f$ on $M$ and any function {$F: \iota(M) \rightarrow \mathbb{R}^q$}, we have
{\begin{align}
&\mathbb{E}[f(X)]:= \int_{M}f(x)d\mathbb{P}_X=\int_{M}f(x)P(x)d\mathfrak{m}(x), \nonumber \\
&\mathbb{E}[F(\iota(X))f(X)]:= \int_{M}F(\iota(x))f(x)d\mathbb{P}_X=\int_{M}F(\iota(x))f(x)P(x)d\mathfrak{m}(x) \in \mathbb{R}^q. \nonumber
\end{align}}
Based on the above definitions, the expectation of the local covariance matrix at $\iota(x)$ is defined as
\begin{equation}
C_x:=\mathbb{E}[(\iota(X)-\iota(x))(\iota(X)-\iota(x))^{\top}\chi_{\big( B^{\mathbb{R}^p}_{\epsilon}(\iota(x)) \cap \iota(M) \big) } (\iota(X))]\in\mathbb{R}^{p\times p}. \nonumber
\end{equation}
Suppose $\texttt{rank}(C_x)=r\leq p$. Clearly $r$ depends on $x$, but we ignore $x$ for the simplicity. Denote the eigen-decomposition of $C_x$ as $C_x=U_x \Lambda_x U_x^\top$, where $U_x\in O(p)$ is composed of eigenvectors and $\Lambda_x$ is a diagonal matrix with the associated eigenvalues $\lambda_1 \geq \lambda_2 \geq \cdots \geq  \lambda_r>\lambda_{r+1}= \cdots = \lambda_p=0$. 

Through the eigenpairs of $C_x$, we can construct an augmented vector $\mathbf{T}(x)$ at $x \in M$.  
\begin{definition}\label{DefAugmented}
The augmented vector at $x\in M$ is 
\begin{align}
\mathbf{T}(x)^\top & = \mathbb{E}[(\iota(X)-\iota(x))\chi_{\big( B^{\mathbb{R}^p}_{\epsilon}(\iota(x)) \cap \iota(M) \big) } (\iota(X))] ^\top U_xI_{p,r}(\Lambda_x+\epsilon^{d+3}I_{p\times p})^{-1} U_x^\top \in \mathbb{R}^p\,, \nonumber
\end{align}
which is a $\mathbb{R}^p$-valued vector field on $M$.
\end{definition}

\subsection{Lemmas for the variance analysis}
{For notation simplicity, we define a vector
{\begin{align}
\mathbf{T}_{n,x_k}:= \mathcal{I}_c(C_{n,k})  G_{n,k}\boldsymbol{1}_{N_k} \label{Definition:Tn}\,.
\end{align}}
From \eqref{proof prop1 main} in Proposition \ref{invariant B_k}, we have}
\begin{align*}
B_k=\frac{\mathbf{T}^\top_{n,x_k}  G_{n,k}\boldsymbol{1}_{N_k}}{{N_k}}=\frac{\frac{1}{n\epsilon^d} \mathbf{T}^\top_{n,x_k}  G_{n,k}\boldsymbol{1}_{N_k}}{\frac{1}{n\epsilon^d}{N_k}}.\,
\end{align*}

We will study the terms $\frac{1}{n\epsilon^d}{N_k}$ and $\frac{1}{n\epsilon^d} \mathbf{T}^\top_{n,x_k}  G_{n,k}\boldsymbol{1}_{N_k}$  seperately in the next two lemmas.  We first introduce the following definitions.
\begin{definition}\label{definition of collection of balls}
Denote $B^{\mathbb{R}^p}$ to be a closed ball in $\mathbb{R}^p$ without specifying the center.  We define the following collections of balls intersecting $\iota(M)$,
\begin{align}
\mathcal{B}_r(\iota(M))=\left\{B^{\mathbb{R}^p} \cap \iota(M) \Big|\,B^{\mathbb{R}^p} \cap \iota(M) \not= \emptyset , \mbox{radius of } B^{\mathbb{R}^p} \leq r\right\}\,.
\end{align}
{For data points $\mathcal{X}$},  if $A \in \mathcal{B}_r(\iota(M))$, then $N(A)=|A \cap \mathcal{X}|$. 
\end{definition}

By Definition \ref{def of R(x)}, for any $x \in M$, $N_a(x)=|B_{a}^{\mathbb{R}^p}(\iota(x)) \cap \mathcal{X}|$. Now, we are ready to provide the variance analysis which relates $\frac{1}{n\epsilon^d}{N_\epsilon(x)}$ to $\frac{1}{\epsilon^d} \mathbb{E}[\chi_{\big( B^{\mathbb{R}^p}_{\epsilon}(\iota(x)) \cap \iota(M) \big) } (\iota(X))]$ for any $x \in M$.

\begin{lemma}\label{Prop 1 Lemma 1}

\

\begin{enumerate}[label=(\arabic*)]
\item
Suppose  $\sup_{A \in \mathcal{B}_{2r}(\iota(M))} \mathbb{E}[\chi_{A}(\iota(X))] \leq  b \leq \frac{1}{4}$. For $n$ large enough, with probability 
greater than $1-n^{-2}$, 
\begin{align*}
\sup_{A \in  \mathcal{B}_r(\iota(M))} |\frac{N(A)}{n}- \mathbb{E}[\chi_{A}(\iota(X))] |=O(\sqrt{\frac{b \log (n)}{n}}).
\end{align*}
\item
Suppose $\epsilon=\epsilon(n)$ so that $\frac{\sqrt{\log(n)}}{n^{1/2}\epsilon^{d/2+1}}\to 0$ and $\epsilon\to 0$ as $n\to \infty$. We have with probability greater than $1-n^{-2}$ that for all $x \in M$, 
\begin{equation}
\left|\frac{N_\epsilon(x)}{n\epsilon^d}  - \frac{1}{\epsilon^d} \mathbb{E}[\chi_{\big( B^{\mathbb{R}^p}_{\epsilon}(\iota(x)) \cap \iota(M) \big) } (\iota(X))]\right| =  O\Big(\frac{\sqrt{\log (n)}}{n^{1/2}\epsilon^{d/2}}\Big)\,,\nonumber
\end{equation}
where the constant in $O\Big(\frac{\sqrt{\log (n)}}{n^{1/2}\epsilon^{d/2}}\Big)$ depends on the $C^0$ norm of $P$ and the second fundamental form of $\iota(M)$.
\end{enumerate} 
\end{lemma}

The proof of (1) in the above lemma is a direct consequence of Lemma A.3 in \cite{wu2022strong}.  As shown in Remark A.1 in \cite{wu2022strong},  the proof of Lemma A.3 in \cite{wu2022strong} does not rely on the underlying manifold structure. Hence, it still holds when $M$ is a manifold with boundary.  The proof of (2) in  Lemma \ref{Prop 1 Lemma 1} is a consequence of (1) and is the same as the proof of Corollary 2.2 in \cite{wu2022strong}. The manifold structure is involved in the estimation of $\sup_{A \in \mathcal{B}_{2\epsilon}(\iota(M))} \mathbb{E}[\chi_{A}(\iota(X))]$ which is of order $\epsilon^d$.  

Note that $N_k=N_\epsilon(x_k)-1$. When $\frac{\sqrt{\log(n)}}{n^{1/2}\epsilon^{d/2+1}}\to 0$ and $n$ is large enough,  $\frac{1}{n\epsilon^d} <\frac{\sqrt{\log (n)}}{n^{1/2}\epsilon^{d/2}}$. Hence, the following lemma is a consequence of (2) in Lemma \ref{Prop 1 Lemma 1}.

\begin{lemma}\label{Prop 1 Lemma 1.1}
Suppose $\epsilon=\epsilon(n)$ so that $\frac{\sqrt{\log(n)}}{n^{1/2}\epsilon^{d/2+1}}\to 0$ and $\epsilon\to 0$ as $n\to \infty$. We have with probability greater than $1-n^{-2}$ that for $k=1, \cdots, n$, 
\begin{equation}
\left|\frac{1}{n\epsilon^d}\sum_{j=1}^{N_k} 1 - \frac{1}{\epsilon^d} \mathbb{E}[\chi_{\big( B^{\mathbb{R}^p}_{\epsilon}(\iota(x_k)) \cap \iota(M) \big) } (\iota(X))]\right|  =  O\Big(\frac{\sqrt{\log (n)}}{n^{1/2}\epsilon^{d/2}}\Big)\,,\nonumber
\end{equation}
where the constant in $O\Big(\frac{\sqrt{\log (n)}}{n^{1/2}\epsilon^{d/2}}\Big)$ depends on the $C^0$ norm of $P$ and the second fundamental form of $\iota(M)$. 
\end{lemma}

In the next lemma, we show that  $\frac{1}{\epsilon^d} \mathbf{T}(x_k)^\top \mathbb{E}(\iota(X)-\iota(x_k))\chi_{ \big(B^{\mathbb{R}^p}_\epsilon(\iota(x_k))\cap\iota(M)\big) }(\iota(X))$ is the limit of $\frac{1}{n\epsilon^d} \mathbf{T}^\top_{n,x_k}  G_{n,k}\boldsymbol{1}_{N_k}$ as $n\to \infty$ and we control the size of fluctuation. The lemma can be found as  (G.50) in \cite{wu2018locally}.

\begin{lemma}\label{Prop 1 Lemma 2}
Suppose $\epsilon=\epsilon(n)$ so that $\frac{\sqrt{\log(n)}}{n^{1/2}\epsilon^{d/2+1}}\to 0$ and $\epsilon\to 0$ as $n\to \infty$. Suppose $c=n\epsilon^{d+3}$ in the construction of $\mathbf{T}^\top_{n,x_k}$ in \eqref{Definition:Tn}.  We have with probability greater than $1-n^{-2}$ that for all $k=1,\ldots,n$, 
\begin{align}
\frac{1}{n\epsilon^d} \mathbf{T}^\top_{n,x_k}  G_{n,k}\boldsymbol{1}_{N_k}=\frac{1}{\epsilon^d} \mathbf{T}(x_k)^\top \mathbb{E}(\iota(X)-\iota(x_k))\chi_{ \big(B^{\mathbb{R}^p}_\epsilon(\iota(x_k))\cap\iota(M)\big) }(\iota(X)) +O\Big(\frac{\sqrt{\log(n)}}{n^{1/2}\epsilon^{d/2}}\Big),
\end{align}
where the constant in $O\Big(\frac{\sqrt{\log(n)}}{n^{1/2}\epsilon^{d/2}}\Big)$ depends on $P_m$,  the $C^1$ norm of $P$ and the second fundamental form of $\iota(M)$ at $\iota(x_k)$. 
\end{lemma}

\subsection{Lemmas for the bias analysis}
Recall the notations introduced in Definition \ref{preliminary def on boundary}, 
$$x_{\partial}:=\arg\min_{y \in  \partial M} d_g(y,x),  \quad \quad \tilde{\epsilon}(x) =d_g(x_{\partial},x).$$ In this subsection, we study the terms $\frac{1}{\epsilon^d} \mathbb{E}[\chi_{B_{\epsilon}^{\mathbb{R}^p}(\iota(x_k))}(\iota(X))]$  and $\mathbb{E}[(\iota(X)-\iota(x_k))\chi_{B_{\epsilon}^{\mathbb{R}^p}(\iota(x_k))}(\iota(X))]$.  The following lemma is a combination of Corollary 28 and Lemma 30 in \cite{wu2018locally}.

\begin{lemma}\label{Prop 1 Lemma 3}
Under Assumptions \ref{main assumption 1}, \ref{main assumption 2}, and \ref{assumption tangent space},  when $\epsilon>0$ is sufficiently small, the following expansions hold.
\begin{enumerate} [label=(\arabic*)]
\item $\mathbb{E}[\chi_{\big( B^{\mathbb{R}^p}_{\epsilon}(\iota(x)) \cap \iota(M) \big) } (\iota(X))]= P(x)\sigma_0(\tilde{\epsilon}(x) ,\epsilon)  \epsilon^d+O(\epsilon^{d+1})$, where $\sigma_0$ is defined in Definition \ref{sigmas summary} and the constant in $O(\epsilon^{d+1})$ depends on the $C^1$ norm of $P$.
\item  $\mathbb{E}[(\iota(X)-\iota(x))\chi_{\big( B^{\mathbb{R}^p}_{\epsilon}(\iota(x)) \cap \iota(M) \big) } (\iota(X))]= P(x) \sigma_{1,d}(\tilde{\epsilon}(x) ,\epsilon) \epsilon^{d+1}e_d+O(\epsilon^{d+2})$. $\sigma_{1,d}$ is defined in Definition \ref{sigmas summary}.  $O(\epsilon^{d+2})$ represents a vector in $\mathbb{R}^p$ whose entries are of order $O(\epsilon^{d+2})$. The constants in $O(\epsilon^{d+2})$ depend on the $C^1$ norm of $P$ and the second fundamental form of $\iota(M)$. 
\end{enumerate}
\end{lemma}

If we combine (2) in Lemma \ref{Prop 1 Lemma 1} and (1) in Lemma \ref{Prop 1 Lemma 3}, we have the following strong uniform consistency of kernel density estimation through $0-1$ kernel on a manifold with boundary.

\begin{proposition}\label{strong uniform consistency of  KDE}
Suppose $\epsilon=\epsilon(n)$ so that $\frac{\sqrt{\log(n)}}{n^{1/2}\epsilon^{d/2+1}}\to 0$ and $\epsilon\to 0$ as $n\to \infty$. We have with probability greater than $1-n^{-2}$ that for all $x \in M$, 
\begin{equation}
\left|\frac{N_\epsilon(x)}{n\epsilon^d \sigma_0(\tilde{\epsilon}(x) ,\epsilon)}  - P(x) \right| =O(\epsilon)+ O\Big(\frac{\sqrt{\log (n)}}{n^{1/2}\epsilon^{d/2}}\Big)\,,\nonumber
\end{equation}
where $\sigma_0$ is defined in Definition \ref{sigmas summary}, the constant $O(\epsilon)$ depends on the $C^1$ norm of $P$ and the constant in $O\Big(\frac{\sqrt{\log (n)}}{n^{1/2}\epsilon^{d/2}}\Big)$ depends on the $C^0$ norm of $P$ and the second fundamental form of $\iota(M)$. 
\end{proposition} 

{In addition to the functions in Definition \ref{sigmas summary}, we define the following two functions on $[0,\infty)$. Let $|S^m|$ denote the volume of the $m$-dimensional unit sphere and$\frac{|S^{d-2}|}{d-1}$ is defined to be $1$ when $d=1$.
\begin{align*}
\sigma_{3}(t,\epsilon)&:=
\left\{
\begin{array}{ll}
-\frac{|S^{d-2}|}{(d^2-1)(d+3)}(1-(\frac{t}{\epsilon})^2)^{\frac{d+3}{2}}&\mbox{ for }0 \leq t \leq \epsilon\\
0&\mbox{ otherwise} 
\end{array}\right.\nonumber\\
\sigma_{3,d}(t,\epsilon)&:=
\left\{
\begin{array}{ll}
-\frac{|S^{d-2}|}{(d^2-1)(d+3)}(2+(d+1)(\frac{t}{\epsilon})^2)(1-(\frac{t}{\epsilon})^2)^{\frac{d+1}{2}}&\mbox{ for }0 \leq t \leq \epsilon\\
0& \mbox{ otherwise}
\end{array}\right.\nonumber
\end{align*}}
Then, the bias analysis of the augmented vector is summarized in the following lemma. The lemma can be found as Proposition 8 (or a combination of Corollary 28 and Lemma 32) in \cite{wu2018locally}.

\begin{lemma}\label{Prop 1 Lemma 4}
Suppose Assumptions \ref{main assumption 1}, \ref{main assumption 2}, and \ref{assumption tangent space} hold. If $x \in M_{\epsilon}$, then
\begin{align}
\mathbf{T}(x)=\,&\frac{\sigma_{1,d}(\tilde{\epsilon}(x) ,\epsilon) }{\sigma_{2,d}(\tilde{\epsilon}(x) ,\epsilon) }\frac{1}{\epsilon} e_d +\frac{P(x)}{2} \Bigg[ \Big(\sigma_{2}(\tilde{\epsilon}(x) ,\epsilon) -\frac{\sigma_{1,d}(\tilde{\epsilon}(x) ,\epsilon) }{\sigma_{2,d}(\tilde{\epsilon}(x) ,\epsilon) }\sigma_{3}(\tilde{\epsilon}(x) ,\epsilon) \Big) v_1(x)\nonumber\\
&\qquad\qquad+\Big(\sigma_{2,d}(\tilde{\epsilon}(x) ,\epsilon) -\frac{\sigma_{1,d}(\tilde{\epsilon}(x) ,\epsilon) }{\sigma_{2,d}(\tilde{\epsilon}(x) ,\epsilon) }\sigma_{3,d}(\tilde{\epsilon}(x) ,\epsilon) \Big)v_2(x)\Bigg]  \frac{1}{\epsilon}+O(1).\nonumber
\end{align}
We have $v_1(x), v_2(x) \in (\iota_* T_{x}M)^{\bot}$. $O(1)$ represents a vector in $\mathbb{R}^p$ whose entries are of order $O(1)$. The constants in $O(1)$ depend on $P_m$,  the $C^1$ norm of $P$ and the second fundamental form of $\iota(M)$ at $\iota(x)$. 

If $x \in M \setminus  M_{\epsilon}$, then
 \begin{align}
\mathbf{T}(x)=\,&\frac{P(x)}{2}  \bigg[\frac{|S^{d-1}|}{d(d+2)} v_3(x)\bigg] \frac{1}{\epsilon}+O(1),\nonumber
\end{align}
where $v_3(x) \in (\iota_* T_{x}M)^{\bot}$. $O(1)$ represents a vector in $\mathbb{R}^p$ whose entries are of order $O(1)$. The constants in $O(1)$ depend on $P_m$,  the $C^1$ norm of $P$, and the second fundamental form of $\iota(M)$ at $\iota(x)$. 
\end{lemma}

\subsection{Combining the bias and the variance analyses to prove Theorem \ref{prop BI 1} }
{For notational simplicity, we prove the theorem under Assumption \ref{assumption tangent space}, which allows us to utilize the previous lemmas. However, by Proposition \ref{invariant B_k}, the value of the BI at each point $z_k$, denoted $B_k$, is invariant under translation of $\mathcal{X}$ and orthogonal transformation on $\mathbb{R}^p$. Hence, the result of the theorem remains valid without Assumption \ref{assumption tangent space}.}
For any $x_k$, since $e_d$ belongs to $\iota_* T_{x_k}M$, we have $e_d^\top v_j(x_k)=0$ for $j=1, 2, 3$. Thus, by (2) in Lemma \ref{Prop 1 Lemma 3} and Lemma \ref{Prop 1 Lemma 4}, when $x_k \in M_\epsilon$,
\begin{align}\label{proof prop1 f1}
 \mathbf{T}(x_k)^\top \mathbb{E}(\iota(X)-\iota(x_k))\chi_{ \big(B^{\mathbb{R}^p}_\epsilon(\iota(x_k))\cap\iota(M)\big) }(\iota(X))=P(x_k)\frac{(\sigma_{1,d}(\tilde{\epsilon}(x_k),\epsilon))^2}{\sigma_{2,d}(\tilde{\epsilon}(x_k) ,\epsilon) }\epsilon^d+O(\epsilon^{d+1}), 
\end{align}
where the constant in $O(\epsilon^{d+1})$ depends on $P_m$,  the $C^1$ norm of $P$ and the second fundamental form of $\iota(M)$ at $\iota(x_k)$.  When $x_k \in M \setminus  M_{\epsilon}$,
\begin{align}\label{proof prop1 f2}
 \mathbf{T}(x_k)^\top \mathbb{E}(\iota(X)-\iota(x_k))\chi_{ \big(B^{\mathbb{R}^p}_\epsilon(\iota(x_k))\cap\iota(M)\big) }(\iota(X))=O(\epsilon^{d+1}), 
\end{align}
the constant in $O(\epsilon^{d+1})$ depends on $P_m$,  the $C^1$ norm of $P$ and the second fundamental form of $\iota(M)$ at $\iota(x_k)$. 

Suppose $\epsilon=\epsilon(n)$ so that $\frac{\sqrt{\log(n)}}{n^{1/2}\epsilon^{d/2+1}}\to 0$ and $\epsilon\to 0$ as $n\to \infty$. By \eqref{proof prop1 f1}, \eqref{proof prop1 f2}, and Lemma \ref{Prop 1 Lemma 2}, with probability greater than $1-n^{-2}$ that for any $x_k \in M_\epsilon$, 
\begin{align}
\frac{1}{n\epsilon^d} \mathbf{T}^\top_{n,x_k}  G_{n,k}\boldsymbol{1}_{N_k}=& \frac{1}{\epsilon^d} \mathbf{T}(x_k)^\top \mathbb{E}(\iota(X)-\iota(x_k))\chi_{ \big(B^{\mathbb{R}^p}_\epsilon(\iota(x_k))\cap\iota(M)\big) }(\iota(X)) +O\Big(\frac{\sqrt{\log(n)}}{n^{1/2}\epsilon^{d/2}}\Big) \nonumber \\
=& P(x_k) \frac{(\sigma_{1,d}(\tilde{\epsilon}(x_k),\epsilon))^2}{\sigma_{2,d}(\tilde{\epsilon}(x_k),\epsilon)}+O(\epsilon)+O\Big(\frac{\sqrt{\log(n)}}{n^{1/2}\epsilon^{d/2}}\Big),\nonumber
\end{align}
and any $x_k \in M \setminus  M_\epsilon$, 
\begin{align}
\frac{1}{n\epsilon^d} \mathbf{T}^\top_{n,x_k}  G_{n,k}\boldsymbol{1}_{N_k}=O(\epsilon)+O\Big(\frac{\sqrt{\log(n)}}{n^{1/2}\epsilon^{d/2}}\Big),\nonumber
\end{align}
where the constants in $O(\epsilon)$ and $O\Big(\frac{\sqrt{\log(n)}}{n^{1/2}\epsilon^{d/2}}\Big)$ depend on $P_m$,  the $C^1$ norm of $P$ and the second fundamental form of $\iota(M)$ at $\iota(x_k)$. Note that when $x_k \in M \setminus M_\epsilon$ , we have $\tilde{\epsilon}(x_k)\geq \epsilon$. By the definition of $\sigma_{1,d}(\tilde{\epsilon}(x_k),\epsilon)$, $\frac{(\sigma_{1,d}(\tilde{\epsilon}(x_k),\epsilon))^2}{\sigma_{2,d}(\tilde{\epsilon}(x_k),\epsilon)}=0$ when $\tilde{\epsilon}(x_k)  \geq \epsilon$. Thus, we can combine the above two cases and we conclude that 
with probability greater than $1-n^{-2}$  for any $x_k$, 
\begin{align}\label{proof prop1 f3}
\frac{1}{n\epsilon^d} \mathbf{T}^\top_{n,x_k}  G_{n,k}\boldsymbol{1}_{N_k}=  P(x_k) \frac{(\sigma_{1,d}(\tilde{\epsilon}(x_k),\epsilon))^2}{\sigma_{2,d}(\tilde{\epsilon}(x_k),\epsilon)}+O(\epsilon)+O\Big(\frac{\sqrt{\log(n)}}{n^{1/2}\epsilon^{d/2}}\Big).
\end{align}

By Lemma \ref{Prop 1 Lemma 1.1} and (1) in Lemma \ref{Prop 1 Lemma 3}, with probability greater than $1-n^{-2}$  for any $x_k$, 
\begin{align}\label{proof prop1 f4}
\frac{1}{n\epsilon^d}{N_k}=P(x_k)\sigma_0(\tilde{\epsilon}(x_k),\epsilon) +O(\epsilon)+O\Big(\frac{\sqrt{\log (n)}}{n^{1/2}\epsilon^{d/2}}\Big),
\end{align}
where the constants in $O(\epsilon)$ and $O\Big(\frac{\sqrt{\log(n)}}{n^{1/2}\epsilon^{d/2}}\Big)$ depend on the $C^1$ norm of $P$ and the second fundamental form of $\iota(M)$.

By \eqref{proof prop1 main}, $B_k=\frac{\frac{1}{n\epsilon^d} \mathbf{T}^\top_{n,x_k}  G_{n,k}(x_k)\boldsymbol{1}_{N_k}}{\frac{1}{n\epsilon^d}{N_k}}$. By \eqref{proof prop1 f3} and \eqref{proof prop1 f4} and taking a union bound, with probability greater than $1-2 n^{-2}$  for any $x_k$,
\begin{align}
B_k=\frac{ P(x_k) \frac{(\sigma_{1,d}(\tilde{\epsilon}(x_k),\epsilon))^2}{\sigma_{2,d}(\tilde{\epsilon}(x_k),\epsilon)}+O(\epsilon)+O\Big(\frac{\sqrt{\log(n)}}{n^{1/2}\epsilon^{d/2}}\Big)}{P(x_k)\sigma_0(\tilde{\epsilon}(x_k),\epsilon) +O(\epsilon)+O\Big(\frac{\sqrt{\log (n)}}{n^{1/2}\epsilon^{d/2}}\Big)}. \nonumber 
\end{align}
Note that $\sigma_0(\tilde{\epsilon}(x_k),\epsilon)$ is bounded from below and from above by constants. Hence,  
\begin{align}\label{final expression of Bk}
B_k= \frac{(\sigma_{1,d}(\tilde{\epsilon}(x_k),\epsilon))^2}{\sigma_0(\tilde{\epsilon}(x_k),\epsilon) \sigma_{2,d}(\tilde{\epsilon}(x_k),\epsilon)}+O(\epsilon)+O\Big(\frac{\sqrt{\log(n)}}{n^{1/2}\epsilon^{d/2}}\Big),
\end{align}
where the constants in $O(\epsilon)$ and $O\Big(\frac{\sqrt{\log(n)}}{n^{1/2}\epsilon^{d/2}}\Big)$ depend on $P_m$,  the $C^1$ norm of $P$ and the second fundamental form of $\iota(M)$.

The properties (1), (2), (3), (4) and (5) follow directly from Proposition \ref{properties of the distance to boundary function} and the definitions of $\tilde{\epsilon}(x)$,  $\sigma_0(\tilde{\epsilon}(x),\epsilon)$, $\sigma_{1,d}(\tilde{\epsilon}(x),\epsilon)$, and $\sigma_{2,d}(\tilde{\epsilon}(x),\epsilon)$. (6) follows from (1), (3), and (4).

\section{Proofs of Proposition \ref{continuity of R} and Proposition \ref{tildeR and R}}\label{proof prop BI 2}
\subsection{Proof of Proposition \ref{tildeR and R}}
We first prove the following lemma about $V(t,r)$ and $U(t,s)$.

\begin{lemma}\label{lemma U and V}
\
\begin{enumerate}
\item 
When $\delta$ is small enough depending on $d$, $V(t, r(1-\delta)) \leq V(t, r) (1- C \delta)$ and $V(t, r(1+\delta)) \geq V(t, r) (1+ C \delta)$, where $C>0$ is a constant depending on $d$.
\item For any $t$,  we have  $ (\frac{d}{|S^{d-1}|})^{\frac{1}{d}}s^{\frac{1}{d}}\leq U(t, s) \leq  (\frac{2d}{|S^{d-1}|})^{\frac{1}{d}}s^{\frac{1}{d}}$
\end{enumerate}
\end{lemma}

\begin{proof}
(1) We prove $V(t, r(1-\delta)) \geq V(t, r) (1- C \delta)$, while $V(t, r(1+\delta)) \leq V(t, r) (1+ C \delta)$ can be proved similarly. Without loss of generality, suppose $t<r(1-\delta)<r$. The cases when $r(1-\delta)<t<r$ or $r(1-\delta)<r<t$ are straightforward. By the definition of $V(t, r)$,
\begin{align}
& V(t, r)- V(t, r(1-\delta)) \nonumber \\
=& \frac{|S^{d-1}|}{2d}r^d(1-(1-\delta)^d)+\frac{|S^{d-2}|}{d-1}\int_{0}^{t} \big[ (r^2-x^2)^{\frac{d-1}{2}}-(r^2(1-\delta)^2-x^2)^{\frac{d-1}{2}} \big]dx \nonumber \\
=&  \frac{|S^{d-1}|}{2d}r^d(1-(1-\delta)^d)+\frac{|S^{d-2}|}{d-1}\int_{0}^{t}  (r^2-x^2)^{\frac{d-1}{2}}\big[1-(\frac{r^2(1-\delta)^2-x^2}{r^2-x^2})^{\frac{d-1}{2}} \big]dx \nonumber \\
=&  \frac{|S^{d-1}|}{2d}r^d(1-(1-\delta)^d)+\frac{|S^{d-2}|}{d-1}\int_{0}^{t}  (r^2-x^2)^{\frac{d-1}{2}}\big[1-(1-\frac{2\delta r^2-r^2\delta^2}{r^2-x^2})^{\frac{d-1}{2}} \big]dx \nonumber \\
= & \frac{|S^{d-1}|}{2d}r^d(1-(1-\delta)^d)+\frac{|S^{d-2}|}{d-1}\int_{0}^{t}  (r^2-x^2)^{\frac{d-1}{2}}\big[1-(1-\frac{2\delta-\delta^2}{1-x^2/r^2})^{\frac{d-1}{2}} \big]dx  \nonumber \\
\geq & \frac{|S^{d-1}|}{2d}r^d(1-(1-\delta)^d)+\frac{|S^{d-2}|}{d-1}\int_{0}^{t}  (r^2-x^2)^{\frac{d-1}{2}}\big[1-(1-2\delta+\delta^2)^{\frac{d-1}{2}} \big]dx  \nonumber \\
=  & \frac{|S^{d-1}|}{2d}r^d(1-(1-\delta)^d)+\frac{|S^{d-2}|}{d-1}\int_{0}^{t}  (r^2-x^2)^{\frac{d-1}{2}}\big[1-(1-\delta)^{d-1} \big]dx 
 \nonumber \\
\geq & \frac{d}{2}\delta \frac{|S^{d-1}|}{2d}r^d+\frac{d-1}{2} \delta \frac{|S^{d-2}|}{d-1}\int_{0}^{t}  (r^2-x^2)^{\frac{d-1}{2}}dx  \geq \frac{d-1}{2} \delta V(t, r). \nonumber
\end{align}
Note that in the second last step we use the fact that $1-(1-\delta)^d \geq  \frac{d}{2} \delta$ when $\delta$ is small enough depending on $d$.

(2) Fix any $s$, $r=U(t, s)$ is the radius of the region $\mathcal{R}_{t,r}$ with volume $s$ bounded between the ball of radius $r$ centered at the origin in $\mathbb{R}^d$ and the hyperplane $x_d=t$.   The radius of the region achieves maximum when $t=0$, i.e. $s= \frac{|S^{d-1}|}{2d} r^d$. Hence $U(t,s) \leq (\frac{2d}{|S^{d-1}|})^{\frac{1}{d}}s^{\frac{1}{d}}$.   The radius of the region achieves minimum when $t \geq (\frac{d}{|S^{d-1}|})^{\frac{1}{d}}s^{\frac{1}{d}}$, i.e. when $\mathcal{R}_{t,r}$ is the ball of volume of $s$  . Hence $U(t,s) \geq (\frac{d}{|S^{d-1}|})^{\frac{1}{d}}s^{\frac{1}{d}}$
\end{proof}

\noindent\underline{\textbf{Prove Proposition \ref{tildeR and R} by applying Lemma \ref{lemma U and V}}}

We estimate the probability of the events $\{R(x) \leq \tilde{R}(x)(1-\delta)\}$ and $\{R(x) \geq \tilde{R}(x)(1+\delta)\}$. Based on the definition of $R(x)$,  the event of $\{R(x) \leq \tilde{R}(x)(1-\delta)\}$ is same as the event that $\{N_{\tilde{R}(x)(1-\delta)}(x)\geq K+1\}$. Thus, we have
\begin{equation}
\Pr\big\{R(x) \leq \tilde{R}(x)(1-\delta)\big\} = \Pr\big\{\frac{1}{n}N_{\tilde{R}(x)(1-\delta)}(x) \geq \frac{K+1}{n}\big\}. 
 \end{equation}
Similarly, we have
 \begin{equation}
\Pr\big\{R(x) \geq \tilde{R}(x)(1+\delta)\big\} =\Pr\big\{\frac{1}{n}N_{\tilde{R}(x)(1+\delta)}(x) \leq \frac{K+1}{n}\big\}. 
 \end{equation}
 
We start to evaluate $\Pr\big\{\frac{1}{n}N_{\tilde{R}(x)(1-\delta)}(x) \geq \frac{K+1}{n}\big\}$. By (2) in Lemma \ref{lemma U and V},
\begin{align}\label{proof R(x) tilde R(x) eq 0}
\tilde{R}(x)=U(\tilde{\epsilon}(x), \frac{K+1}{P(x)n}) \leq (\frac{2d}{|S^{d-1}|})^{\frac{1}{d}}(\frac{K+1}{P_m n})^{\frac{1}{d}}:=C_1(\frac{K+1}{n})^{\frac{1}{d}}.
\end{align}
Define $\tilde{R}^*:=C_1(\frac{K+1}{n})^{\frac{1}{d}}$. Recall the definitions of $B^{\mathbb{R}^p}_r$ and $\mathcal{B}_r(\iota(M))$ in Definition \ref{definition of collection of balls}. Observe that for any $r>0$, no matter where the center of $B^{\mathbb{R}^p}_r$ is, if $B^{\mathbb{R}^p}_r \cap \iota(M) \not= \emptyset$, then $B^{\mathbb{R}^p}_r \cap \iota(M) \subset B^{\mathbb{R}^p}_{2r}(\iota(x')) \cap \iota(M)$ for $\iota(x') \in B^{\mathbb{R}^p}_r \cap \iota(M)$. Hence, by (1) in Lemma \ref{Prop 1 Lemma 3},
\begin{align}
\sup_{A \in \mathcal{B}_{4\tilde{R}^*}(\iota(M))} \mathbb{E}[\chi_{A}(\iota(X))] \leq \sup_{x' \in M} \mathbb{E}[\chi_{\big( B^{\mathbb{R}^p}_{8\tilde{R}^*}(\iota(x')) \cap \iota(M) \big) } (\iota(X))] \leq C_2 (\frac{K+1}{n}),
\end{align}
where $C_2$ depends on $C^0$ norm of $P$ and $P_m$. 

\

By (1) in Lemma \ref{Prop 1 Lemma 1},
Suppose  $\frac{K}{n} \rightarrow 0$ as $n \rightarrow \infty$, then $C_2 (\frac{K+1}{n}) \leq \frac{1}{4}$. For $n$ large enough, with probability 
greater than $1-n^{-2}$, for all $x$, 
\begin{align}
&|\frac{1}{n}N_{\tilde{R}(x)(1-\delta)}(x)- \mathbb{E}[\chi_{\big( B^{\mathbb{R}^p}_{\tilde{R}(x)(1-\delta)}(\iota(x)) \cap \iota(M) \big) } (\iota(X))] | \\
\leq & \sup_{A \in  \mathcal{B}_{2\tilde{R}^*}(\iota(M)) } |\frac{N(A)}{n}- \mathbb{E}[\chi_{A}(\iota(X))] | \leq C_3 \frac{\sqrt{(K+1) \log (n)}}{n}, \nonumber
\end{align}
where $C_3$ depends on  $C^0$ norm of $P$ and $P_m$.

Next, we derive the condition on $\delta$ such that $\frac{1}{n}N_{\tilde{R}(x)(1-\delta)}(x) \geq \frac{K+1}{n}$ implies 
\begin{align}\label{proof R(x) tilde R(x) eq 1}
\frac{1}{n}N_{\tilde{R}(x)(1-\delta)}(x)- \mathbb{E}[\chi_{\big( B^{\mathbb{R}^p}_{\tilde{R}(x)(1-\delta)}(\iota(x)) \cap \iota(M) \big) } (\iota(X))] \geq C_3 \frac{\sqrt{(K+1) \log (n)}}{n}.
\end{align}
If we subtract both sides of  $\frac{1}{n}N_{\tilde{R}(x)(1-\delta)}(x) \geq \frac{K+1}{n}$ by $\mathbb{E}[\chi_{\big( B^{\mathbb{R}^p}_{\tilde{R}(x)(1-\delta)}(\iota(x)) \cap \iota(M) \big) } (\iota(X))]$, we have 
\begin{align}
& \frac{1}{n}N_{\tilde{R}(x)(1-\delta)}(x)- \mathbb{E}[\chi_{\big( B^{\mathbb{R}^p}_{\tilde{R}(x)(1-\delta)}(\iota(x)) \cap \iota(M) \big) } (\iota(X))] \\
\geq & \frac{K+1}{n} - \mathbb{E}[\chi_{\big( B^{\mathbb{R}^p}_{\tilde{R}(x)(1-\delta)}(\iota(x)) \cap \iota(M) \big) } (\iota(X))] \nonumber \\
\geq & \frac{K+1}{n} - P(x) V(\tilde{\epsilon}(x) ,\tilde{R}(x)(1-\delta)) -C_4 \big(\tilde{R}(x)(1-\delta)\big)^{d+1} \nonumber \\
\geq & \frac{K+1}{n} - P(x) V(\tilde{\epsilon}(x) ,\tilde{R}(x))+C\delta  P(x) V(\tilde{\epsilon}(x) ,\tilde{R}(x))  -C_4 \big(\tilde{R}(x)(1-\delta)\big)^{d+1} \nonumber \\
= & C\delta \frac{K+1}{n} -C_4 \big(\tilde{R}(x)(1-\delta)\big)^{d+1}. \nonumber
\end{align}
where $C_4>0$ depends on $C^1$ norm of $P$. Lemma \ref{Prop 1 Lemma 3} is applied in the third last step, Lemma \ref{lemma U and V} is applied in the second last step, and the definitions of the function $V$ and the term $\tilde{R}(x)$ are applied in the last step. If 
\begin{align}\label{proof R(x) tilde R(x) eq 2}
C\delta \frac{K+1}{n} -C_4 \big(\tilde{R}(x)(1-\delta)\big)^{d+1} \geq C_3 \frac{\sqrt{(K+1) \log (n)}}{n}, 
\end{align}
then we have \eqref{proof R(x) tilde R(x) eq 1}. Hence, 
\begin{align*}
& \Pr\big\{\frac{1}{n}N_{\tilde{R}(x)(1-\delta)}(x) \geq \frac{K+1}{n}\big\} \\
\leq & \Pr\Bigg\{ \Big(\frac{1}{n}N_{\tilde{R}(x)(1-\delta)}(x)- \mathbb{E}[\chi_{\big( B^{\mathbb{R}^p}_{\tilde{R}(x)(1-\delta)}(\iota(x)) \cap \iota(M) \big) } (\iota(X))] \Big)\geq C_3 \frac{\sqrt{(K+1) \log (n)}}{n}\Bigg\} \\
\leq & \Pr\Bigg\{\Big| \frac{1}{n}N_{\tilde{R}(x)(1-\delta)}(x)- \mathbb{E}[\chi_{\big( B^{\mathbb{R}^p}_{\tilde{R}(x)(1-\delta)}(\iota(x)) \cap \iota(M) \big) } (\iota(X))] \Big| \geq C_3 \frac{\sqrt{(K+1) \log (n)}}{n}\Bigg\}  \leq n^{-2}.
\end{align*}
In order to have \eqref{proof R(x) tilde R(x) eq 2}, it suffices to require
\begin{align}\label{eqn delta bound 1}
\frac{C}{2}\delta \frac{K+1}{n} \geq C_4 ( \tilde{R}^*)^{d+1} \geq C_4 \big(\tilde{R}(x)(1-\delta)\big)^{d+1},
\end{align} 
and 
\begin{align}\label{eqn K N bound 1}
\frac{C}{2}\delta \frac{K+1}{n}  \geq C_3 \frac{\sqrt{(K+1) \log (n)}}{n}.
\end{align}

By \eqref{proof R(x) tilde R(x) eq 0}, \eqref{eqn delta bound 1} is equivalent to $\delta \geq \frac{2C_4}{C} C_1^{d+1} (\frac{K+1}{n})^{\frac{1}{d}}$. Moreover, \eqref{eqn K N bound 1} is equivalent to  $\delta \geq \frac{2C_3}{C} \sqrt{\frac{\log(n)}{K+1}}$. If $\frac{\log(n)}{K+1} (\frac{n}{K+1})^{2/d}\to 0$ as $n \to \infty$, then we have $\frac{2C_4}{C} C_1^{d+1} (\frac{K+1}{n})^{\frac{1}{d}} \geq  \frac{2C_3}{C} \sqrt{\frac{\log(n)}{K+1}}$. Hence, it suffices to require $\delta \geq \frac{2C_4}{C} C_1^{d+1} (\frac{K+1}{n})^{\frac{1}{d}}$ which is guaranteed by $\delta \geq \frac{4C_4}{C} C_1^{d+1} (\frac{K}{n})^{\frac{1}{d}}$. Therefore, we choose  $\delta = \frac{4C_4}{C} C_1^{d+1} (\frac{K}{n})^{\frac{1}{d}}$. Note that $\frac{\log(n)}{K+1} (\frac{n}{K+1})^{2/d}\to 0$ is equivalent to $\frac{\log(n)}{K} (\frac{n}{K})^{2/d}\to 0$. 

Hence, we show that if $\frac{K}{n} \rightarrow 0$ and $\frac{\log(n)}{K} (\frac{n}{K})^{2/d}\to 0$ as $n \to \infty$, then with probability less than $n^{-2}$, for all $x$, $R(x) \leq \tilde{R}(x)(1-\delta)$, where $\delta = \frac{4C_4}{C} C_1^{d+1} (\frac{K}{n})^{\frac{1}{d}}$. 

\

By \eqref{proof R(x) tilde R(x) eq 0},  if $\delta$ is small, $\tilde{R}(x)(1+\delta)\leq 2\tilde{R}^*$.  By (1) in Lemma \ref{Prop 1 Lemma 1}, suppose  $\frac{K}{n} \rightarrow 0$ as $n \rightarrow \infty$, then $C_2 (\frac{K+1}{n}) \leq \frac{1}{4}$. Hence, for $n$ large enough, with probability 
greater than $1-n^{-2}$, for all $x$, 
\begin{align}
&|\frac{1}{n}N_{\tilde{R}(x)(1+\delta)}(x)- \mathbb{E}[\chi_{\big( B^{\mathbb{R}^p}_{\tilde{R}(x)(1+\delta)}(\iota(x)) \cap \iota(M) \big) } (\iota(X))] | \\
\leq & \sup_{A \in  \mathcal{B}_{2\tilde{R}^*}(\iota(M)) } |\frac{N(A)}{n}- \mathbb{E}[\chi_{A}(\iota(X))] | \leq C_3 \frac{\sqrt{(K+1) \log (n)}}{n}, \nonumber
\end{align}
where $C_3$ depends on  $C^0$ norm of $P$ and $P_m$.

We derive the condition on $\delta$ such that $\frac{1}{n}N_{\tilde{R}(x)(1+\delta)}(x) \leq \frac{K+1}{n}$ implies 
\begin{align}\label{proof R(x) tilde R(x) eq 3}
\frac{1}{n}N_{\tilde{R}(x)(1+\delta)}(x)- \mathbb{E}[\chi_{\big( B^{\mathbb{R}^p}_{\tilde{R}(x)(1+\delta)}(\iota(x)) \cap \iota(M) \big) } (\iota(X))] \leq -C_3 \frac{\sqrt{(K+1) \log (n)}}{n}.
\end{align}
If we subtract both sides of  $\frac{1}{n}N_{\tilde{R}(x)(1+\delta)}(x) \leq \frac{K+1}{n}$ by $\mathbb{E}[\chi_{\big( B^{\mathbb{R}^p}_{\tilde{R}(x)(1+\delta)}(\iota(x)) \cap \iota(M) \big) } (\iota(X))]$, we have 
\begin{align}
& \frac{1}{n}N_{\tilde{R}(x)(1+\delta)}(x)- \mathbb{E}[\chi_{\big( B^{\mathbb{R}^p}_{\tilde{R}(x)(1+\delta)}(\iota(x)) \cap \iota(M) \big) } (\iota(X))] \\
\leq & \frac{K+1}{n} - \mathbb{E}[\chi_{\big( B^{\mathbb{R}^p}_{\tilde{R}(x)(1+\delta)}(\iota(x)) \cap \iota(M) \big) } (\iota(X))] \nonumber \\
\leq & \frac{K+1}{n} - P(x) V(\tilde{\epsilon}(x) ,\tilde{R}(x)(1+\delta)) +C_4 \big(\tilde{R}(x)(1+\delta)\big)^{d+1} \nonumber \\
\leq & \frac{K+1}{n} - P(x) V(\tilde{\epsilon}(x) ,\tilde{R}(x))-C\delta  P(x) V(\tilde{\epsilon}(x) ,\tilde{R}(x)) +C_4 \big(\tilde{R}(x)(1+\delta)\big)^{d+1} \nonumber \\
= & C_4 \big(\tilde{R}(x)(1+\delta)\big)^{d+1}-C\delta \frac{K+1}{n} . \nonumber
\end{align}
where $C_4>0$ depends on $C^1$ norm of $P$. Lemma \ref{Prop 1 Lemma 3} is applied in the third last step, Lemma \ref{lemma U and V} is applied in the second last step, and the definitions of the function $V$ and the term $\tilde{R}(x)$ are applied in the last step. If 
\begin{align}\label{proof R(x) tilde R(x) eq 4}
C_4 \big(\tilde{R}(x)(1+\delta)\big)^{d+1}-C\delta \frac{K+1}{n} \leq -C_3 \frac{\sqrt{(K+1) \log (n)}}{n}, 
\end{align}
then we have \eqref{proof R(x) tilde R(x) eq 3}. Hence, 
\begin{align*}
& \Pr\big\{\frac{1}{n}N_{\tilde{R}(x)(1+\delta)}(x) \leq \frac{K+1}{n}\big\} \\
\leq & \Pr\Bigg\{\Big( \frac{1}{n}N_{\tilde{R}(x)(1+\delta)}(x)- \mathbb{E}[\chi_{\big( B^{\mathbb{R}^p}_{\tilde{R}(x)(1+\delta)}(\iota(x)) \cap \iota(M) \big) } (\iota(X))]\Big) \leq -C_3 \frac{\sqrt{(K+1) \log (n)}}{n}\Bigg\} \\
\leq & \Pr\Bigg\{\Big| \frac{1}{n}N_{\tilde{R}(x)(1+\delta)}(x)- \mathbb{E}[\chi_{\big( B^{\mathbb{R}^p}_{\tilde{R}(x)(1+\delta)}(\iota(x)) \cap \iota(M) \big) } (\iota(X))] \Big| \geq C_3 \frac{\sqrt{(K+1) \log (n)}}{n}\Bigg\}  \leq n^{-2}.
\end{align*}
In order to have \eqref{proof R(x) tilde R(x) eq 4}, it suffices to require
\begin{align}\label{eqn delta bound 11}
\frac{C}{2}\delta \frac{K+1}{n} \geq C_4 ( 2\tilde{R}^*)^{d+1} \geq C_4 \big(\tilde{R}(x)(1+\delta)\big)^{d+1},
\end{align} 
and 
\begin{align}\label{eqn K N bound 11}
\frac{C}{2}\delta \frac{K+1}{n}  \geq C_3 \frac{\sqrt{(K+1) \log (n)}}{n}.
\end{align}

Note that by \eqref{proof R(x) tilde R(x) eq 0}, \eqref{eqn delta bound 11} is equivalent to $\delta \geq \frac{2^{d+2}C_4}{C} C_1^{d+1} (\frac{K+1}{n})^{\frac{1}{d}}$ and \eqref{eqn K N bound 11} is equivalent to  $\delta \geq \frac{2C_3}{C} \sqrt{\frac{\log(n)}{K+1}}$. If $\frac{\log(n)}{K+1} (\frac{n}{K+1})^{2/d}\to 0$ as $n \to \infty$, then we have $\frac{2^{d+2}C_4}{C} C_1^{d+1} (\frac{K+1}{n})^{\frac{1}{d}} \geq  \frac{2C_3}{C} \sqrt{\frac{\log(n)}{K+1}}$. Hence, it suffices to require $\delta \geq \frac{2^{d+2}C_4}{C} C_1^{d+1} (\frac{K+1}{n})^{\frac{1}{d}}$ which is guaranteed by $\delta \geq \frac{2^{d+3}C_4}{C} C_1^{d+1} (\frac{K}{n})^{\frac{1}{d}}$. Therefore, we choose  $\delta = \frac{2^{d+3}C_4}{C} C_1^{d+1} (\frac{K}{n})^{\frac{1}{d}}$. At last, note that $\frac{\log(n)}{K+1} (\frac{n}{K+1})^{2/d}\to 0$ is equivalent to $\frac{\log(n)}{K} (\frac{n}{K})^{2/d}\to 0$. 

Hence, we show that if $\frac{K}{n} \rightarrow 0$ and $\frac{\log(n)}{K} (\frac{n}{K})^{2/d}\to 0$ as $n \to \infty$, than with probability less than $n^{-2}$, for all $x$, $R(x) \geq \tilde{R}(x)(1+\delta)$, where $\delta = \frac{2^{d+3}C_4}{C} C_1^{d+1} (\frac{K}{n})^{\frac{1}{d}}$. 

In conclusion if $\frac{K}{n} \rightarrow 0$ and $\frac{\log(n)}{K} (\frac{n}{K})^{2/d}\to 0$ as $n \to \infty$, than with probability greater than $1-2n^{-2}$, for all $x$, $R(x) = \tilde{R}(x)(1+O( (\frac{K}{n})^{\frac{1}{d}}))$, where the constant in $O( (\frac{K}{n})^{\frac{1}{d}})$ depends on $d$, $C^1$ norm of $P$, and $P_m$

When $n$ is large enough, we have $\frac{1}{2}\tilde{R}(x) \leq R(x) \leq \frac{3}{2}\tilde{R}(x)$. Hence, by (2) in Lemma \ref{lemma U and V} and the definition of $\tilde{R}(x)$,  
\begin{align}\label{Appendix proof upper and lower bounds of Rx for all x}
\frac{1}{2} (\frac{d}{|S^{d-1}|})^{\frac{1}{d}}(\frac{K}{P_M n})^{\frac{1}{d}} \leq \frac{1}{2} (\frac{d(K+1)}{|S^{d-1}|P_M n})^{\frac{1}{d}}\leq R(x) \leq  \frac{3}{2} (\frac{2d(K+1)}{|S^{d-1}|P_m n})^{\frac{1}{d}}\leq 3 (\frac{2d}{|S^{d-1}|})^{\frac{1}{d}}(\frac{K}{P_m n})^{\frac{1}{d}}.
\end{align}
Hence, 
\begin{align*}
\frac{1}{2} (\frac{d}{|S^{d-1}|})^{\frac{1}{d}}(\frac{K}{P_M n})^{\frac{1}{d}} \leq  R^* \leq 3 (\frac{2d}{|S^{d-1}|})^{\frac{1}{d}}(\frac{K}{P_m n})^{\frac{1}{d}}.
\end{align*}

\subsection{Proof Proposition \ref{continuity of R}}
(1) Consider $x, x' \in M$. Assume $R(x') \geq R(x)$. Observe that by triangle inequality, $B^{\mathbb{R}^p}_{R(x)}(\iota(x)) \subset B^{\mathbb{R}^p}_{R(x)+\|\iota(x)-\iota(x')\|_{\mathbb{R}^p}}(\iota(x'))$. Hence, $B^{\mathbb{R}^p}_{R(x)+\|\iota(x)-\iota(x')\|_{\mathbb{R}^p}}(\iota(x'))$ contains at least $K+1$ points. We have $R(x') \leq R(x)+\|\iota(x)-\iota(x')\|_{\mathbb{R}^p}$ , i.e. $R(x') -R(x) \leq \|\iota(x)-\iota(x')\|_{\mathbb{R}^p}$. Similarly, when $R(x') \leq R(x)$, we have $R(x) -R(x') \leq \|\iota(x)-\iota(x')\|_{\mathbb{R}^p}$. Hence, $|R(x') -R(x)| \leq \|\iota(x)-\iota(x')\|_{\mathbb{R}^p}$. When $d_g(x, x') \to 0$, then $\|\iota(x)-\iota(x')\|_{\mathbb{R}^p} \to 0$ and $|R(x') -R(x)| \to 0$.

\

(2) From the proof of (1),  $|R(\gamma_x(t_1)) -R(\gamma_x(t_2))| \leq \|\iota(\gamma_x(t_1))-\iota(\gamma_x(t_2))\|_{\mathbb{R}^p}$. Since $\gamma_x(t)$ is unit speed and {length minimizing} on $[0, t_2]$, we have 
\begin{align}\label{relation R and t}
|R(\gamma_x(t_1)) -R(\gamma_x(t_2))| \leq \|\iota(\gamma_x(t_1))-\iota(\gamma_x(t_2))\|_{\mathbb{R}^p} \leq t_2-t_1. 
\end{align}
Observe that
\begin{align}
\frac{t_2}{R(\gamma_x(t_2))}-\frac{t_1}{R(\gamma_x(t_1))}=\frac{R(\gamma_x(t_1))-t_1 \frac{R(\gamma_x(t_2))-R(\gamma_x(t_1))}{t_2-t_1}}{R(\gamma_x(t_1))R(\gamma_x(t_2))/(t_2-t_1).}\nonumber
\end{align}
If $R(\gamma_x(t_2)) \leq R(\gamma_x(t_1))$, then $\frac{t_2}{R(\gamma_x(t_2))}>\frac{t_1}{R(\gamma_x(t_1))}$. If $R(\gamma_x(t_2))>R(\gamma_x(t_1))$, we have $\frac{R(\gamma_x(t_2))-R(\gamma_x(t_1))}{t_2-t_1}<1$ by \eqref{relation R and t}. Hence, if $t_1<R(\gamma_x(t_1))$, then $\frac{t_1}{R(\gamma_x(t_1))}<\frac{t_2}{R(\gamma_x(t_2))}$. The conclusion follows.

\section{Proofs of Theorem \ref{Covariance KNN setup} and Theorem \ref{prop BI 2}}\label{proof THM BI 2}

\subsection{Proof of Theorem \ref{Covariance KNN setup}}
First, we relate  $C_{n,k}$ constructed through the KNN scheme to $C_{n,k}$ constructed through the $\epsilon$ radius ball scheme through $R(x)$ defined in Definition \ref{def of R(x)}. Observe that for each $x_k$, $C_{n,k}$ constructed through the KNN scheme is equal to the  $C_{n,k}$ constructed through the $R(x_k)$-radius ball scheme. By Theorem \ref{local covariance epsilon},  if for any $k$, $R(x_k) \to 0$ and $\frac{\sqrt{\log(n)}}{n^{1/2} R(x_k)^{d/2+1}}\to 0$ and  as $n\to \infty$, then with probability greater than $1-2n^{-2}$,  for all $k$,
\begin{align}\label{Cnk in R(x) 1}
\frac{1}{n} C_{n,k}= & P(x_k){
\begin{bmatrix}
M^{(0)}(x_k,R(x_k)) & 0 \\
0& 0 
\end{bmatrix}} R(x_k)^{d+2} 
+\begin{bmatrix}
M^{(11)}(x_k, R(x_k)) & M^{(12)}(x_k,R(x_k)) \\
M^{(21)}(x_k,R(x_k)) & 0 
\end{bmatrix}
R(x_k)^{d+3}  \\
&+O(R(x_k)^{d+4})+O\Big(\frac{\sqrt{\log(n)}}{n^{1/2}}R(x_k)^{d/2+2}\Big). \nonumber
\end{align}
When $\tilde{\epsilon}(x_k) \geq R^*$,
\begin{align}
\frac{1}{n} C_{n,k}= & P(x_k){
\begin{bmatrix}
M^{(0)}(x_k,R(x_k)) & 0 \\
0& 0 
\end{bmatrix}} R(x_k)^{d+2} 
+O(R(x_k)^{d+4}) +O\Big(\frac{\sqrt{\log(n)}}{n^{1/2}}R(x_k)^{d/2+2}\Big). \label{Cnk in R(x) 2}
\end{align}
Second, we bound $R(x)$ by $\frac{K}{n}$. {By \eqref{Appendix proof upper and lower bounds of Rx for all x},} suppose we have $\frac{K}{n} \rightarrow 0$ and $\frac{\log(n)}{K} (\frac{n}{K})^{2/d}\to 0$ as $n \to \infty$, then for all $x$, with probability greater than $1-2n^{-2}$, 
\begin{align*}
\frac{1}{2} (\frac{d}{|S^{d-1}|})^{\frac{1}{d}}(\frac{K}{P_M n})^{\frac{1}{d}} \leq R(x) \leq  3 (\frac{2d}{|S^{d-1}|})^{\frac{1}{d}}(\frac{K}{P_m n})^{\frac{1}{d}}.
\end{align*}
Hence,  $\frac{K}{n} \rightarrow 0$ is equivalent to $R(x_k)\to 0$ and $\frac{\log(n)}{K} (\frac{n}{K})^{2/d}\to 0$ is equivalent to $\frac{\sqrt{\log(n)}}{n^{1/2}{R(x_k)^{d/2+1}}}\to 0$.  If we substitute the above bounds of $R(x)$ into  \eqref{Cnk in R(x) 1} and \eqref{Cnk in R(x) 2}, then with probability greater than $1-4n^{-2}$, we have
\begin{align*}
\frac{1}{n} C_{n,k}= & P(x_k){
\begin{bmatrix}
M^{(0)}(x_k,R(x_k)) & 0 \\
0& 0 
\end{bmatrix}} R(x_k)^{d+2} 
+\begin{bmatrix}
\tilde{M}^{(11)}(x_k, \frac{K}{n}) & \tilde{M}^{(12)}(x_k, \frac{K}{n}) \\
\tilde{M}^{(21)}(x_k, \frac{K}{n}) & 0 
\end{bmatrix} \\
&+O( (\frac{K}{n})^{\frac{d+4}{d}}) +O\Big(\frac{\sqrt{K \log(n)}}{n} (\frac{K}{n})^{\frac{2}{d}}\Big). 
\end{align*}
The magnitudes of the entries in $\tilde{M}^{(11)}(x_k, \frac{K}{n})$ , $\tilde{M}^{(12)}(x_k, \frac{K}{n})$, and $\tilde{M}^{(21)}(x_k, \frac{K}{n}) $ are bounded by $\tilde{C}  (\frac{K}{n})^{\frac{d+3}{d}}$, where $\tilde{C}$ is constant depending on $d$, $P_m$, the $C^1$ norm of $P$, the second fundamental form of $\iota(M)$ in $\mathbb{R}^p$ at $\iota(x_k)$, and the second fundamental form of $\partial M$ in $M$ at $x_{\partial,k}$.

When $\tilde{\epsilon}(x_k) \geq R^*$,
\begin{align*}
\frac{1}{n} C_{n,k}= & P(x_k){
\begin{bmatrix}
M^{(0)}(x_k,R(x_k)) & 0 \\
0& 0 
\end{bmatrix}} R(x_k)^{d+2} 
+O( (\frac{K}{n})^{\frac{d+4}{d}}) +O\Big(\frac{\sqrt{K \log(n)}}{n} (\frac{K}{n})^{\frac{2}{d}}\Big). 
\end{align*}

At last, we discuss the entries in $\tilde{M}^{(0)}(x)=P(x)M^{(0)}(x,R(x))
 R(x)^{d+2}$ for $x \in M$.  $\tilde{M}^{(0)}(x)$ is a diagonal matrix.  By Theorem \ref{local covariance epsilon} and Proposition \ref{tildeR and R}, the $i$th diagonal entry of $\tilde{M}^{(0)}(x)$ is 
\begin{align*}
P(x) \sigma_{2}(\tilde{\epsilon}(x) ,R(x)) R(x)^{d+2} = P(x) \sigma_{2}(\tilde{\epsilon}(x) ,R(x)) (\tilde{R}(x)(1+O( (\frac{K}{n})^{\frac{1}{d}})))^{d+2},
\end{align*}
for $i=1, \cdots, d-1$. And the $d$th diagonal entry is 
\begin{align*}
P(x) \sigma_{2,d}(\tilde{\epsilon}(x) ,R(x)) R(x)^{d+2} = P(x) \sigma_{2,d}(\tilde{\epsilon}(x) ,R(x)) (\tilde{R}(x)(1+O( (\frac{K}{n})^{\frac{1}{d}})))^{d+2}.
\end{align*}
By \eqref{proof R(x) tilde R(x) eq 0}, we can derive the following equalities,
\begin{align*}
& P(x) \sigma_{2}(\tilde{\epsilon}(x) ,R(x)) R(x)^{d+2} =P(x)\sigma_{2}(\tilde{\epsilon}(x) ,R(x)) \tilde{R}(x)^{d+2}+O((\frac{K}{n})^{\frac{d+3}{d}}). \\
& P(x) \sigma_{2,d}(\tilde{\epsilon}(x) ,R(x)) R(x)^{d+2} =P(x) \sigma_{2,d}(\tilde{\epsilon}(x) , R(x))  \tilde{R}(x)^{d+2}+O((\frac{K}{n})^{\frac{d+3}{d}}).
\end{align*}
Define $\mu_1(x)=P(x) \sigma_{2}(\tilde{\epsilon}(x) , R(x))  \tilde{R}(x)^{d+2}$ and $ \mu_2(x)=P(x) \sigma_{2,d}(\tilde{\epsilon}(x) , R(x))  \tilde{R}(x)^{d+2}$. 

Both $\mu_1(x)$ and $\mu_2(x)$ are continuous functions. We focus on $ \mu_1(x)$, while $\mu_2(x)$ can be discussed similarly.  Note that $\frac{|S^{d-1}|}{2d(d+2)} \leq \sigma_{2}(\tilde{\epsilon}(x) , R(x)) \leq \frac{|S^{d-1}|}{d(d+2)}$. By (2) in Lemma \ref{lemma U and V} and the definition of $\tilde{R}(x)$, $ (\frac{d}{|S^{d-1}|})^{\frac{d+2}{d}}(\frac{K+1}{P(x) n})^{\frac{d+2}{d}}\leq \tilde{R}(x)^{d+2} \leq  (\frac{2d}{|S^{d-1}|})^{\frac{d+2}{d}}(\frac{K+1}{P(x) n})^{\frac{d+2}{d}}$. Hence, 
$$\frac{1}{2(d+2)}(\frac{d}{|S^{d-1}| P_M}) ^{\frac{2}{d}} (\frac{K+1}{n})^{\frac{d+2}{d}}\leq \mu_1(x) \leq \frac{2}{d+2}(\frac{2d}{|S^{d-1}| P_m}) ^{\frac{2}{d}} (\frac{K+1}{n})^{\frac{d+2}{d}}.$$
When $x \in \partial M$, $\sigma_{2}(\tilde{\epsilon}(x) , R(x))=\frac{|S^{d-1}|}{2d(d+2)}$ and $\tilde{R}(x)^{d+2}=(\frac{2d}{|S^{d-1}|})^{\frac{d+2}{d}}(\frac{K+1}{P(x) n})^{\frac{d+2}{d}}$. Therefore,
$$\mu_1(x) = \frac{1}{(d+2)}(\frac{2d}{|S^{d-1}| P(x)}) ^{\frac{2}{d}} (\frac{K+1}{n})^{\frac{d+2}{d}}.$$
The same results hold for $\mu_2(x)$.

When $\tilde{\epsilon}(x_k)  \geq R^*$,
$$\sigma_{2}(\tilde{\epsilon}(x_k) , R(x_k))=\sigma_{2,d}(\tilde{\epsilon}(x_k) , R(x_k))= \frac{|S^{d-1}|}{d(d+2)}.$$
Moreover, by (2) in Lemma \ref{lemma U and V}, $\tilde{R}(x_k)=(\frac{d}{|S^{d-1}|})^{\frac{1}{d}}(\frac{K+1}{P(x_k) n})^{\frac{1}{d}}$. Therefore, 
$$\mu_1(x_k)= \mu_2(x_k)=\frac{1}{(d+2)}(\frac{d}{|S^{d-1}| P(x_k)}) ^{\frac{2}{d}} (\frac{K+1}{n})^{\frac{d+2}{d}}.$$

\subsection{Proof of Theorem \ref{prop BI 2}}
By Proposition \ref{tildeR and R}, suppose $\frac{K}{n} \rightarrow 0$ and $\frac{\log(n)}{K} (\frac{n}{K})^{2/d}\to 0$ as $n \to \infty$.   Then, for all $x$, with probability greater than $1-2 n^{-2}$,  we have $\frac{1}{2}\tilde{R}(x) \leq R(x) \leq \frac{3}{2}\tilde{R}(x)$. Hence, by (2) in Lemma \ref{lemma U and V} and the definition of $\tilde{R}(x)$,  
\begin{align}\label{R X_k double bounds}
\frac{1}{2} (\frac{d}{|S^{d-1}|})^{\frac{1}{d}}(\frac{K}{P_M n})^{\frac{1}{d}} \leq R(x) \leq  3 (\frac{2d}{|S^{d-1}|})^{\frac{1}{d}}(\frac{K}{P_m n})^{\frac{1}{d}}.
\end{align}

Observe that for each $x_k$, $B_k$ constructed through the KNN scheme is equal to the $B_k$ constructed through the $R(x_k)$-radius ball scheme. Based on the proof of Lemmas \ref{Prop 1 Lemma 2} and \ref{Prop 1 Lemma 4} (refer to \cite{wu2018locally}), the conclusions of the lemmas still hold whenever we choose $\tilde{C}_1 n \epsilon^{d+3}\leq c \leq \tilde{C}_2 n \epsilon^{d+3}$, where $\tilde{C}_1$ and $\tilde{C}_2$ are constants independent of $n$ and $\epsilon$.  Suppose we choose $\tilde{C}_1 n R(x_k)^{d+3} \leq c \leq \tilde{C}_2 n R(x_k)^{d+3}$ where $\tilde{C}_1$ and $\tilde{C}_2$  are constants independent of $n$ and $K$. By \eqref{final expression of Bk} {which follows from Lemmas \ref{Prop 1 Lemma 2} and \ref{Prop 1 Lemma 4}}, if  for any $k$, $R(x_k) \to 0$ and $\frac{\sqrt{\log(n)}}{n^{1/2} R(x_k)^{d/2+1}}\to 0$ and  as $n\to \infty$, then with probability greater than $1-2n^{-2}$,  for all $k$,
\begin{align}\label{B_k in R_k}
B_k= \frac{(\sigma_{1,d}(\tilde{\epsilon}(x_k),R(x_k)))^2}{\sigma_0(\tilde{\epsilon}(x_k),R(x_k)) \sigma_{2,d}(\tilde{\epsilon}(x_k),R(x_k))}+O(R(x_k))+O\Big(\frac{\sqrt{\log(n)}}{n^{1/2}R(x_k)^{d/2}}\Big). 
\end{align}
where $\tilde{\epsilon}(x_k)$ is the distance from $x_k \in M$ to $\partial M$ as defined in \eqref{Definition tildeepsilon gamma}.  The constants in $O(R(x_k))$ and $O\Big(\frac{\sqrt{\log(n)}}{n^{1/2}R(x_k)^{d/2}}\Big)$ depend on $P_m$,  the $C^1$ norm of $P$ and the second fundamental form of $\iota(M)$. Moreover, if  $\tilde{\epsilon}(x)$ is the distance from $x \in M$ to $\partial M$, then we define
$$ \tilde{B}(x)=\frac{(\sigma_{1,d}(\tilde{\epsilon}(x),R(x)))^2}{\sigma_0(\tilde{\epsilon}(x),R(x)) \sigma_{2,d}(\tilde{\epsilon}(x),R(x))}.$$
By \eqref{R X_k double bounds}, $\frac{K}{n} \rightarrow 0$ is equivalent to $R(x_k)\to 0$ and $\frac{\log(n)}{K} (\frac{n}{K})^{2/d}\to 0$ is equivalent to $\frac{\sqrt{\log(n)}}{n^{1/2}{R(x_k)^{d/2+1}}}\to 0$.  Moreover, if $c=n (\frac{K}{n})^{\frac{d+3}{d}}$ , then 
$$(\frac{2}{3})^{d+3} (\frac{|S^{d-1}|}{2d})^{\frac{d+3}{d}} P_m^{\frac{d+3}{d}} n R(x_k)^{d+3} \leq  c \leq  2^{d+3} (\frac{|S^{d-1}|}{d})^{\frac{d+3}{d}} P_M^{\frac{d+3}{d}}n R(x_k)^{d+3}.$$

By taking the union bound for the probability, with probability greater than $1-4n^{-2}$,  for all $k$, we have \eqref{R X_k double bounds} for all $x_k$ and \eqref{B_k in R_k}. If we substitute \eqref{R X_k double bounds} into \eqref{B_k in R_k},
$$ B_k= \frac{(\sigma_{1,d}(\tilde{\epsilon}(x_k),R(x_k)))^2}{\sigma_0(\tilde{\epsilon}(x_k),R(x_k)) \sigma_{2,d}(\tilde{\epsilon}(x_k),R(x_k))}+O((\frac{K}{n})^{\frac{1}{d}})+O\Big(\sqrt{\frac{\log(n)}{K}}\Big). $$
The constants in $O(\frac{K}{n})^{\frac{1}{d}}$ and $O\Big(\sqrt{\frac{\log(n)}{K}}\Big)$ depend on $P_m$,  the $C^1$ norm of $P$ and the second fundamental form of $\iota(M)$. 

Next, we discuss the properties of $\tilde{B}(x)$. By Proposition \ref{continuity of R} and the definitions of $\sigma_0$, $\sigma_{1,d}$, and $\sigma_{2,d}$,
$\tilde{B}(x)$ is a continuous function on $M$. When $x\in \partial M$, $\tilde{\epsilon}(x)=0$ and we have $ \tilde{B}(x)=\frac{4d^2(d+2)|S^{d-2}|^2}{(d^2-1)^2|S^{d-1}|^2}$ .

Suppose that we have $t_1 > R(\gamma_x(t_1))$ and $t_1<t_2< R^*$.  Since $\gamma_x(t)$ is distance mimizing on $[0, 2R^*]$, by \eqref{relation R and t}, $R(\gamma_x(t_2))<R(\gamma_x(t_1))+ t_2-t_1<t_2$. Since $R(x)$ is continuous, $R(\gamma_x(0))>0$, and $R(\gamma_x(R^*)) \leq R^*$, by the intermediate value theorem and the above discussion, there is a $0<t_x^* \leq R^*$ such that 
\begin{enumerate}
\item $R(\gamma_x(t_x^*))=t_x^*$,
\item $R(\gamma_x(t))>t$, for $t<t_x^*$,
\item $R(\gamma_x(t))<t$, for $t>t_x^*$.
\end{enumerate}

Fix $x \in \partial M$, $d_g(\gamma_x(t), \partial M)=t$ for $0\leq t \leq 2R^*$. Then,
$$ \tilde{B}(\gamma_x(t))=\frac{(\sigma_{1,d}(t,R(\gamma_x(t))))^2}{\sigma_0(t,R(\gamma_x(t))) \sigma_{2,d}(t,R(\gamma_x(t)))}.$$ 
Based on the definitions of $\sigma_0$, $\sigma_{1,d}$, and $\sigma_{2,d}$, $ \tilde{B}(\gamma_x(t))=0$ for $t \geq t_x^*$. 
Suppose $t_1<t_2<t_x^*$, then $t_1<R(\gamma_x(t_1))$. Hence, by Proposition \ref{continuity of R}, we have $\frac{t_1}{R(\gamma_x(t_1))}<\frac{t_2}{R(\gamma_x(t_2))}$.
Since $ \tilde{B}(\gamma_x(t))$ is a decreasing function of $\frac{t}{R(\gamma_x(t))}$ based on the definitions, the conclusion follows.

\section{Review of the boundary detection algorithms}\label{Summary algorithms}
Let $(M,g)$ be a d-dimensional compact, smooth Riemannian manifold with boundary isometrically embedded in $\mathbb{R}^p$ via $\iota:M \hookrightarrow \mathbb{R}^p$. We assume the boundary of $M$, denoted as $\partial M$, is smooth. Suppose  $\{x_1 \cdots, x_n\} \subset M$ are i.i.d. samples based on a p.d.f $P$ on M. Given $\mathcal{X}=\{z_i=\iota(x_i)\}_{i=1}^n$, the detected boundary points from $\mathcal{X}$ are denoted as $\partial \mathcal{X}$. In this section, we review the boundary detection algorithms that we apply in Section \ref{numerical simulation}. Furthermore, in the original formulations of some algorithms, the threshold parameters are not explicitly specified. We describe our chosen thresholds for the algorithm implementations in Section \ref{numerical simulation}.

\subsection{$\alpha$-shape algorithm}
The $\alpha$-shape algorithm\cite{Edelsbrunner:1983,Edelsbrunner:1994} is widely applied algorithm in boundary detection. It works effectively when $M$ has the same dimension $p$ as the ambient space $\mathbb{R}^p$. Intuitively, since each connected component of $\partial \iota(M)$ is a hypersurface in $\mathbb{R}^p$, we approximate $\partial \iota(M)$ using hyperspheres, where points on these hyperspheres can be classified as $\partial \mathcal{X}$. The algorithm is summarized as follows. First, the {\em generalized $\alpha$ ball} in $\mathbb{R}^p$ for $\alpha\in \mathbb{R}$ is defined in the following way. For $\alpha>0$, a generalized $\alpha$ ball is a closed $p$-ball of radius $1/\alpha$; for $\alpha<0$, it is the closure of complement of a $p$-ball of radius $-1/\alpha$; if $\alpha=0$, it is the closed half space. Using the generalized $\alpha$ ball, we can define the {\em $\alpha$-boundary}. If there is an $\alpha$ ball containing $\mathcal{X}$ and there are $p$ points of $\mathcal{X}$ on the boundary of the $\alpha$ ball, then these $p$ points are called $\alpha$-neighbours. The union of all $\alpha$-neighbors is called $\alpha$ boundary points, denoted as  $\partial \mathcal{X}$.  However, identifying $\alpha$-boundary points directly from the definition is generally challenging. In practice, the relationship between $\alpha$-boundary points and the Delaunay triangulation is utilized.  Recall that the Delaunay triangulation of $\mathcal{X}$ is a triangulation, denoted as $\texttt{DT}(\mathcal{X})$, such that no point in $\mathcal{X}$ is in the circumhypersphere of any $p$-simplex in $\texttt{DT}(\mathcal{X})$. For each k-simplex $T$ in $\texttt{DT}(\mathcal{X})$,  where $0 \leq k \leq p$, let $\sigma_T$ be the radius of circumhypersphere of $T$. The {\em $\alpha$-complex} $C_\alpha$ is defined as $\{ T \in \texttt{DT}(\mathcal{X}) \mid \sigma_T < 1/|\alpha|\}$. The vertices on the boundary of $C_\alpha$ constitute the $\alpha$-boundary. For a comprehensive review of the Delaunay triangulation and $\alpha$-complex, refer to \cite{toth2017handbook}.

\subsection{ BORDER algorithm}

In BORDER algorithm \cite{Xia:2006}, let $\mathcal{O}_k \subset \mathcal{X}$  be the nearest neighbors of $z_k \in \mathcal{X}$ in the KNN scheme. The {\em reverse $K$ nearest neighbors} of $z_k$ is defined as $\mathcal{R}_k:= \{z_i \in \mathcal{X}| z_k \in \mathcal{O}_{i} \}$. If $|\mathcal{R}_k|$ is smaller than a specified threshold, $z_k$ is classified as a boundary point. Otherwise, it is an interior point. Note that distinguishing $|\mathcal{R}_k|$ between boundary and interior points can be challenging when the points in $\mathcal{X}$ are not uniformly distributed on an embedded manifold within Euclidean space. Consequently, the algorithm's performance may be sensitive to the data distribution.

Suppose $\delta$ represents the value at the $5$th percentile of $\{|\mathcal{R}_i|\}_{i=1}^n$. In the simulations in Section \ref{numerical simulation}, we implement BORDER such that $z_k \in \partial \mathcal{X}$ if $|\mathcal{R}_k|<\delta$.

\subsection{BRIM algorithm}
In BRIM algorithm \cite{Qiu:2007}, let $\mathcal{O}_k \subset \mathcal{X}$  be the nearest neighbors of $z_k \in \mathcal{X}$ in the $\epsilon$-radius ball scheme, consisting of $N_k$ points.  For each $z_k$,  the attractor of $z_k$ is defined as $\texttt{Att}(z_k)=\arg\max_{z_i \in \mathcal{O}_k} N_i$. For each $z_i \in \mathcal{O}_k$, define $\theta(z_i)=\angle_{z_i, z_k, \texttt{Att}(z_k)}\in [0, \pi] $. Using the $\theta$ function, define $\texttt {PN}(z_k):=\{z_i \in \mathcal{O}_k| \theta(z_i) \leq \pi/2 \}$ and $\texttt {NN}(z_k):=\{z_i \in \mathcal{O}_k| \theta(z_i) > \pi/2 \}$. Finally, define $\texttt {BD}(z_k):=\frac{|\texttt {PN}(z_k)|}{|\texttt {NN}(z_k)|} \big||\texttt {PN}(z_k)|-|\texttt {NN}(z_k)|\big|$. A threshold $\delta$ is chosen such that if $\texttt{BD}(z_k) > \delta$, then $z_k \in \partial \mathcal{X}$; otherwise, it is an interior point. However, the distinction in $\texttt{BD}(z_k)$ between a boundary point and an interior point is significant only if the attractor is selected appropriately. Specifically, under the manifold assumption, for any $z_i \in \mathcal{O}_\epsilon(z_k)$, $N_i$ is of order $n\epsilon^d$ up to a constant depending on the density of the data. Therefore, the algorithm's accuracy depends on comparing quantities of the same order with respect to $\epsilon$ and could be sensitive to the data distribution. 

Suppose $\delta$ represents the value at the $95$th percentile of $\{\texttt{BD}(z_i)\}_{i=1}^n$. In the simulations in Section \ref{numerical simulation}, we implement BRIM such that $z_k \in \partial \mathcal{X}$ if $\texttt{BD}(z_k)>\delta$.

\subsection{BAND,  LEVER, and SPINVER algorithms}
Let $\mathcal{O}_k \subset \mathcal{X}$  denote the nearest neighbors of $z_k \in \mathcal{X}$ in the KNN scheme, consisting of $N_k$ points. 

In BAND \cite{xue2009boundary}, define $D(z_k)$ as the inverse of the average distance between $z_k$ and the points in $\mathcal{O}_k$, given by $D(z_k) = ( \frac{1}{N_k} \sum_{z_i \in \mathcal{O}_k}\|z_i - z_k\|_{\mathbb{R}^p} )^{-1}$. This makes $D(z_k)$ function as a density estimator. Let $VD(z_k)$ represent the variance of $D(z_i)$ over the points ${z_i}$ in ${z_k} \cup \mathcal{O}_k$. Suppose the data points are distributed according to a density function with a small derivative. Intuitively, near the boundary, $D(z_k)$ should be smaller than in the interior to reflect the lack of symmetry near the boundary. Conversely, the variance $VD(z_k)$ should be small in the interior, indicating a slow change in density. Consequently, the authors propose thresholds $\delta$ and $\delta'$ such that $z_k \in \partial \mathcal{X}$ if $D(z_k) < \delta$ and $VD(z_k) > \delta'$. 

In SPINVER \cite{qiu2016clustering}, let $s(z_k)=||\sum_{z_i \in \mathcal{O}_k}( z_i-z_k)||_{1}$, where $\|\cdot\|_1$ denotes the $L^1$ norm. Thus, $s(z_k)$ quantifies the asymmetry of the neighborhood $\mathcal{O}_k$ with respect to $z_k$.  Moreover, the authors propose using $f(z_k)=\exp(\frac{1}{N_k}\sum_{z_i \in \mathcal{O}_k}\|z_i-z_k\|_{\mathbb{R}^p}^{2})$ to measure the local data density in $\mathcal{O}_k$. Assuming uniform data distribution, when $z_k$ is near the boundary, the data points in $\mathcal{O}_k$ should be sparser and less symmetric. Therefore, thresholds $\delta$ and $\delta'$ are suggested such that $z_k \in \partial \mathcal{X}$ if $s(z_k) > \delta$ and $f(z_k)<\delta'$.

The idea of LEVER \cite{cao2018multidimensional} is similar to SPINVER. Let $H(z_k)=||z_k-\frac{1}{N_k}\sum_{z_i \in \mathcal{O}_k}z_i||_{1}$. In fact, $H(z_k)=\frac{1}{N_k}s(z_k)$, where $s(z_k)$ is defined in the SPINVER.  Hence,  $H(z_k)$ assesses the asymmetry of $\mathcal{O}_k$ with respect to $z_k$. Define $D(z_k)=\sum_{z_i \in \mathcal{O}_k}\exp(\|z_i-z_k\|_{\mathbb{R}^p})$ to quantify the data density in $\mathcal{O}_k$. Similarly, $z_k$ is identified as  a boundary point if $H(z_k)>\delta$ and $D(z_k)<\delta'$ for thresholds $\delta$ and $\delta'$.  Alternatively, the authors suggest selecting bounds $\delta<\delta'$ such that $z_k \in \partial \mathcal{X}$ if $\delta<H(z_k)D(z_k)<\delta'$.

Clearly, the BAND SPINVER, and  LEVER are sensitive to the data distribution, and selecting appropriate thresholds becomes especially challenging when the data points are non-uniformly distributed.

Let $\delta$ denote the value at the 20th percentile of $\{D(z_i)\}_{i=1}^n$, and $\delta'$ denote the value at the 80th percentile of $\{VD(z_i)\}_{i=1}^n$. In the simulations in Section \ref{numerical simulation}, BAND is implemented such that $z_k \in \partial \mathcal{X}$ if $D(z_k) < \delta$ and $VD(z_k) > \delta'$. Similarly, let $\delta$ represent the value at the 95th percentile of $\{s(z_i)\}_{i=1}^n$, and $\delta'$ represent the value at the 5th percentile of $\{f(z_i)\}_{i=1}^n$. For SPINVER in the simulations in Section \ref{numerical simulation}, $z_k \in \partial \mathcal{X}$ if $s(z_k) > \delta$ and $f(z_k) < \delta'$. Lastly, suppose $\delta$ indicates the value at the 95th percentile of $\{H(z_i)\}_{i=1}^n$, and $\delta'$ indicates the value at the 5th percentile of $\{D(z_i)\}_{i=1}^n$. In the simulations in Section \ref{numerical simulation}, LEVER is implemented such that $z_k \in \partial \mathcal{X}$ if $H(z_k) > \delta$ and $D(z_k) < \delta'$.

\subsection{CPS algorithm}
We introduce the CPS algorithm \cite{calder2022boundary}(abbreviated by authors' initials for brevity). The authors propose detecting the boundary points by directly estimating the distance to the boundary function:
$$d_g(\iota(x),\partial \iota (M))=\min_{y \in \partial M} d_g(\iota(x) ,\iota(y)).$$
A key observation is that if $B^{\mathbb{R}^p}_{\epsilon}(\iota(x)) \cap \partial \iota (M) \not= \emptyset$, then 
$$d_g(\iota(x),\partial \iota (M))=\max_{\iota(y) \in B^{\mathbb{R}^p}_{\epsilon}(\iota(x)) \cap \iota(M)}\big(d_g(\iota(x),\partial \iota (M))-d_g(\iota(y),\partial \iota (M))\big),$$
where the maximum is attained when $y \in \partial M$. Let $\gamma(t)$ be the unit speed geodesic defined as in (3) of Assumption \ref{assumption tangent space}, perpendicular to $\partial M$ and passing through $x$ at $t=t_0=d_g(\iota(x),\partial \iota (M))$. Define $v(\iota(x))=\iota_*\frac{d \gamma(t_0)}{dt}$ to be the unit tangent vector at $\iota(x)$. Through a second order Taylor expansion of $d_g(\iota(x),\partial \iota (M))-d_g(\iota(y),\partial \iota (M))$ with respect to $\iota(x)-\iota(y)$, the authors approximate $d_g(\iota(x),\partial \iota (M))$ as
\begin{align}\label{CPS estimator}
\max_{\iota(y) \in B^{\mathbb{R}^p}_{\epsilon}(\iota(x)) \cap \iota(M)}\big(\iota(x)-\iota(y)\big) \cdot \frac{1}{2} \big(v(\iota(x))+v(\iota(y))\big).
\end{align}

With the above motivation, the steps of the CPS algorithm can be summarized as follows. Let $\mathcal{O}_k \subset \mathcal{X}$  denote the nearest neighbors of $z_k \in \mathcal{X}$ in the $\epsilon$-radius ball scheme. Suppose the $\frac{\epsilon}{2}$-radius ball neighborhood of $z_k$ contains $\tilde{N}_k$ points.
For any $z_k \in \mathcal{X}$ close to $\partial \iota(M)$, $v(z_k)$ can be approximated by taking the mean of $z_i-z_k$ in  $\mathcal{O}_k$, adjusted by a $0-1$ kernel density estimation.  Specifically, define  
$$\hat{v}(z_k) =\frac{\tilde{v}(z_k)}{\|\tilde{v}(z_k)\|_{\mathbb{R}^p}},  \quad \quad \tilde{v}(z_k)=\frac{|S^{d-1}|}{d} (\frac{\epsilon}{2})^d \sum_{z_i \in \mathcal{O}_k} \frac{z_i-z_k}{\tilde{N}_i}.$$
Then, $\hat{v}(z_k) $ is an estimator of $v(z_k)$. Let $\frac{1}{n}C_{n,k}$ be the local covariance matrix associated with $\mathcal{O}_k$ defined in \eqref{local covariance matrix}. Let $T_k$ be the subspace generated by the eigenvectors corresponding to the first $d$ eigenvalues of $\frac{1}{n}C_{n,k}$. Thus, $T_k$  approximates the tangent space of $\iota(M)$ at $z_k$.  If $\mathcal{P}_k$ is the projection operator from $\mathbb{R}^p$ onto $T_k$, then $\big(z_i-z_k\big) \cdot \frac{1}{2} \big(v(z_i)+v(z_k)\big)$ can be approximated by 
$$\mathcal{P}_k(z_i-z_k\big) \cdot \Bigg(\frac{\hat{v}(z_i)+\hat{v}(z_k)}{2}\Bigg)=\mathcal{P}_k(z_i-z_k\big) \cdot \Bigg(\hat{v}(z_k)+\frac{\hat{v}(z_i)-\hat{v}(z_k)}{2}\Bigg).$$
However, when $z_i$ and $z_k$ are away from $\partial \iota(M)$. The estimations $\hat{v}(z_i)$ and $\hat{v}(z_k)$ may not be accurate and can even form an angle close to $\pi$.  Therefore, the authors suggest adding a cutoff function to $\frac{\hat{v}(z_i)-\hat{v}(z_k)}{2}$.  According to \eqref{CPS estimator}, the estimator of $d_g(z_k,\partial \iota (M))$ is defined as
$$\hat{d}_k=\max_{z_i \in \mathcal{O}_k}\mathcal{P}_k(z_i-z_k\big) \cdot \Bigg(\hat{v}(z_k)+\frac{\hat{v}(z_i)-\hat{v}(z_k)}{2} \mathcal{\chi}_{\mathbb{R}^+}\Big(\mathcal{P}_k(\hat{v}(z_i)) \cdot \mathcal{P}_k(\hat{v}(z_k)) \Big)\Bigg),$$
where $ \mathcal{\chi}_{\mathbb{R}^+}$ is the characteristic function supported on $\mathbb{R}^+$. By applying a small threshold $r$, $z_k \in \partial \mathcal{X}$ if $\hat{d}_k<r$.

\section{Review of the Diffusion map}\label{review of DM section}
{We provide a brief review of Diffusion Maps (DM). The algorithm described below corresponds to the DM with $\alpha=1$ normalization, as introduced in the original work \cite{coifman2006diffusion}.  The $\alpha=1$ normalization is designed to preserve the intrinsic structure of the data regardless of its distribution. Given point cloud $\{z_i\}_{i=1}^n \subset \mathbb{R}^p$, DM constructs a kernel normalized graph Laplacian $L_{DM} \in \mathbb{R}^{n \times n}$ using the kernel $k(z, z')=\exp(-\frac{\|z-z'\|^2_{\mathbb{R}^p}}{4\epsilon_{DM}^2})$ as shown in the following steps.
\begin{enumerate}
\item
Let $W_{ij}=\frac{k(z_i,z_j)}{q(z_i) q(z_j)} \in \mathbb{R}^{n \times n}$, $1 \leq i,j, \leq n$, where $q(z_i)=\sum_{j=1}^n k(z_i, z_j)$.
\item
Define an $n \times n$ diagonal matrix $D$ as $D_{ii}=\sum_{j=1}^{n} W_{ij}$, where $i=1,\ldots,n$. 
\item The kernel normalized graph Laplacian $L_{DM}$ is defined as $L_{DM}=\frac{I-D^{-1}W}{\epsilon^2_{DM}} \in  \mathbb{R}^{n \times n}$.
\end{enumerate}
Suppose $(\lambda^{DM}_j, V_j)_{j=0}^{n-1}$ are the orthonormal eigenpairs of $L_{DM}$ with $\lambda^{DM}_0 \leq \lambda^{DM}_1 \leq \cdots, \leq \lambda^{DM}_{n-1}$. Then, $\lambda^{DM}_0=0$ and $V_0$ is a constant vector. The map $z_i \rightarrow (V_1(i), \cdots, V_{\ell}(i)) \in \mathbb{R}^{\ell}$ provides the coordinates of $z_i$ in a low-dimensional space $\mathbb{R}^{\ell}$.}

{Next, we review results that justify how DM reduces the dimensionality of data while preserving the underlying manifold structure. Let $-\Delta$ denote the Laplace-Beltrami operator of a closed (compact without boundary) smooth Riemannian manifold $M$. Let $\{\lambda_j\}_{j=0}^\infty$ be the eigenvalues of $-\Delta$, ordered so that $0=\lambda_0 <\lambda_1 \leq \lambda_2 < \cdots$, and let $-\Delta \varphi_j=\lambda_j \varphi_j$ with $\varphi_j$ being the corresponding eigenfunction normalized in $L^2(M)$. It is shown in \cite{jones2008manifold, bates2014, portegies2016embeddings} that $\Phi=(\varphi_1, \cdots, \varphi_\ell):M\rightarrow \mathbb{R}^\ell$ constitutes an embedding of $M$ in $\mathbb{R}^\ell$ for sufficiently large $\ell$. Suppose $M$ is isometrically embedded in $\mathbb{R}^p$ via $\iota$ and let $\{x_i\}_{i=1}^n \subset M$ with $\mathcal{X}=\{z_i=\iota(x_i)\}_{i=1}^n$ being the point cloud.  We construct $L_{DM}$ from $\mathcal{X}$ and let $(\lambda^{DM}_j, V_j)_{j=0}^{n-1}$ be the increasing ordered orthonormal eigenpairs of $L_{DM}$ defined as above. Then, as shown in \cite{dunson2021spectral}, for sufficiently large $n$ , $\lambda^{DM}_j$ approximates $\lambda_j$ and $V_j$ (after a proper normalization) approximates the vector $(\varphi_j(x_1), \cdots, \varphi_j(x_n))$ in $\ell^\infty$ norm for $j=0, \cdots, \ell$. Similar results are proved in \cite{wormell2021, calder2022lipschitz} under different assumptions about the manifold and the kernel used in DM. Hence, the map $z_i \rightarrow (V_1(i), \cdots, V_{\ell}(i))$ approximates the embedding $\Phi$ of $M$ in $\mathbb{R}^\ell$ over $\{x_i\}_{i=1}^n$ and topologically preserves the underlying manifold structure of $\mathcal{X}$.  When  $\{x_i\}_{i=1}^n$ are sampled from $M$ which is a compact smooth manifold with boundary, numerical evidence suggests that DM still approximates an embedding of $M$ in the Euclidean space. A recent result \cite{peoples2025} shows that if $\varphi_j$ satisfies the Neumann boundary condition, the $j$-th eigenvector of a symmetrized graph Laplacian converges to $\varphi_j$ in $\ell^2$ sense. However, a complete theoretical framework establishing that DM approximates an embedding in the case of manifold with boundary, analogous to the closed manifold case, remains to be developed.}

\bibliographystyle{plain}  
\bibliography{bib}   

\begin{thebibliography}{10}

\bibitem{aamari2023minimax}
Eddie Aamari, Catherine Aaron, and Cl{\'e}ment Levrard.
\newblock Minimax boundary estimation and estimation with boundary.
\newblock {\em Bernoulli}, 29(4):3334--3368, 2023.

\bibitem{alvarez2020local}
Javier {\'A}lvarez-Vizoso, Michael Kirby, and Chris Peterson.
\newblock Local eigenvalue decomposition for embedded {R}iemannian manifolds.
\newblock {\em Linear Algebra and its Applications}, 604:21--51, 2020.

\bibitem{bates2014}
Jonathan Bates.
\newblock The embedding dimension of {L}aplacian eigenfunction maps.
\newblock {\em Applied and Computational Harmonic Analysis}, 37(3):516--530,
  2014.

\bibitem{berry2017density}
Tyrus Berry and Timothy Sauer.
\newblock Density estimation on manifolds with boundary.
\newblock {\em Computational Statistics \& Data Analysis}, 107:1--17, 2017.

\bibitem{Bo:2008}
L.~Bo, H.~Zhang, and W.~Chen.
\newblock Boundary constrained manifold unfolding.
\newblock In {\em Machine Learning and Applications, 2008. ICMLA'08. Seventh
  International Conference on}, pages 174--181. IEEE, 2008.

\bibitem{calder2022lipschitz}
Jeff Calder, Nicolas Garcia~Trillos, and Marta Lewicka.
\newblock Lipschitz regularity of graph {Laplacians} on random data clouds.
\newblock {\em SIAM Journal on Mathematical Analysis}, 54(1):1169--1222, 2022.

\bibitem{calder2022boundary}
Jeff Calder, Sangmin Park, and Dejan Slep{\v{c}}ev.
\newblock Boundary estimation from point clouds: {A}lgorithms, guarantees and
  applications.
\newblock {\em Journal of Scientific Computing}, 92(2):56, 2022.

\bibitem{calder2022improved}
Jeff Calder and Nicolas~Garcia Trillos.
\newblock Improved spectral convergence rates for graph {L}aplacians on
  $\varepsilon$-graphs and k-{NN} graphs.
\newblock {\em Applied and Computational Harmonic Analysis}, 60:123--175, 2022.

\bibitem{cao2018multidimensional}
Xiaofeng Cao, Baozhi Qiu, Xiangli Li, Zenglin Shi, Guandong Xu, and Jianliang
  Xu.
\newblock Multidimensional balance-based cluster boundary detection for
  high-dimensional data.
\newblock {\em IEEE transactions on neural networks and learning systems},
  30(6):1867--1880, 2018.

\bibitem{cheng2013local}
Ming-Yen Cheng and Hau-Tieng Wu.
\newblock Local linear regression on manifolds and its geometric
  interpretation.
\newblock {\em Journal of the American Statistical Association},
  108(504):1421--1434, 2013.

\bibitem{cheng2022convergence}
Xiuyuan Cheng and Hau-Tieng Wu.
\newblock Convergence of graph laplacian with k-{NN} self-tuned kernels.
\newblock {\em Information and Inference: A Journal of the IMA},
  11(3):889--957, 2022.

\bibitem{chintakunta2013distributed}
Harish Chintakunta and Hamid Krim.
\newblock Distributed boundary tracking using alpha and {Delaunay-Cech} shapes.
\newblock {\em arXiv preprint arXiv:1302.3982}, 2013.

\bibitem{cholaquidis2016set}
Alejandro Cholaquidis, Ricardo Fraiman, G{\'a}bor Lugosi, and Beatriz
  Pateiro-L{\'o}pez.
\newblock Set estimation from reflected {B}rownian motion.
\newblock {\em Journal of the Royal Statistical Society: Series B},
  78(5):1057--1078, 2016.

\bibitem{coifman2006diffusion}
Ronald~R Coifman and St{\'e}phane Lafon.
\newblock Diffusion maps.
\newblock {\em Applied and Computational Harmonic Analysis}, 21(1):5--30, 2006.

\bibitem{dibakar1999computational}
Sen Dibakar and TS~Mruthyunjaya.
\newblock A computational geometry approach for determination of boundary of
  workspaces of planar manipulators with arbitrary topology.
\newblock {\em Mechanism and Machine theory}, 34(1):149--169, 1999.

\bibitem{ding2019bdrygp}
Liang Ding, Simon Mak, and CF~Wu.
\newblock Bdrygp: a new {G}aussian process model for incorporating boundary
  information.
\newblock {\em arXiv preprint arXiv:1908.08868}, 2019.

\bibitem{ding2022impact}
Xiucai Ding and Hau-Tieng Wu.
\newblock Impact of signal-to-noise ratio and bandwidth on graph {L}aplacian
  spectrum from high-dimensional noisy point cloud.
\newblock {\em IEEE Transactions on Information Theory}, 69(3):1899--1931,
  2022.

\bibitem{dunson2021spectral}
David~B Dunson, Hau-Tieng Wu, and Nan Wu.
\newblock Spectral convergence of graph {Laplacian} and heat kernel
  reconstruction in ${L}^\infty$ from random samples.
\newblock {\em Applied and Computational Harmonic Analysis}, 55:282--336, 2021.

\bibitem{dunson2020diffusion}
David~B. Dunson, Hau-Tieng Wu, and Nan Wu.
\newblock Graph based {Gaussian} processes on restricted domains.
\newblock {\em Journal of the Royal Statistical Society: Series B},
  84(2):414--439, 2022.

\bibitem{dunson2021inferring}
David~B Dunson and Nan Wu.
\newblock Inferring manifolds from noisy data using {Gaussian} processes.
\newblock {\em arXiv preprint arXiv:2110.07478}, 2021.

\bibitem{Edelsbrunner:1983}
H.~Edelsbrunner, D.~Kirkpatrick, and R.~Seidel.
\newblock On the shape of a set of points in the plane.
\newblock {\em IEEE Transactions on information theory}, 29(4):551--559, 1983.

\bibitem{Edelsbrunner:1994}
H.~Edelsbrunner and E.~P. M{\"u}cke.
\newblock Three-dimensional alpha shapes.
\newblock {\em ACM Transactions on Graphics (TOG)}, 13(1):43--72, 1994.

\bibitem{el2016graph}
Noureddine El~Karoui and Hau-Tieng Wu.
\newblock Graph connection {Laplacian} methods can be made robust to noise.
\newblock {\em The Annals of Statistics}, 44(1):346--372, 2016.

\bibitem{jiang2023ghost}
Shixiao~Willing Jiang and John Harlim.
\newblock Ghost point diffusion maps for solving elliptic {PDE}s on manifolds
  with classical boundary conditions.
\newblock {\em Communications on Pure and Applied Mathematics}, 76(2):337--405,
  2023.

\bibitem{jones2008manifold}
Peter~W Jones, Mauro Maggioni, and Raanan Schul.
\newblock Manifold parametrizations by eigenfunctions of the {Laplacian} and
  heat kernels.
\newblock {\em Proceedings of the National Academy of Sciences},
  105(6):1803--1808, 2008.

\bibitem{kaslovsky2014non}
D.~N. Kaslovsky and F.~G. Meyer.
\newblock Non-asymptotic analysis of tangent space perturbation.
\newblock {\em Information and Inference: a Journal of the IMA}, 3(2):134--187,
  2014.

\bibitem{little2017multiscale}
A.~V. Little, M.~Maggioni, and L.~Rosasco.
\newblock Multiscale geometric methods for data sets {I}: Multiscale {SVD},
  noise and curvature.
\newblock {\em Applied and Computational Harmonic Analysis}, 43(3):504--567,
  2017.

\bibitem{peoples2021spectral}
J~Wilson Peoples and John Harlim.
\newblock Spectral convergence of symmetrized graph {L}aplacian on manifolds
  with boundary.
\newblock {\em arXiv preprint arXiv:2110.06988}, 2021.

\bibitem{peoples2025}
J.~Wilson Peoples and John Harlim.
\newblock Spectral convergence of symmetrized graph laplacian on manifolds with
  boundary.
\newblock {\em Foundations of Data Science}, 2025.

\bibitem{portegies2016embeddings}
Jacobus~W Portegies.
\newblock Embeddings of {Riemannian} manifolds with heat kernels and
  eigenfunctions.
\newblock {\em Communications on Pure and Applied Mathematics}, 69(3):478--518,
  2016.

\bibitem{Qiu:2007}
B.~Qiu, F.~Yue, and J.~Shen.
\newblock {BRIM}: An efficient boundary points detecting algorithm.
\newblock {\em Advances in Knowledge Discovery and Data Mining}, pages
  761--768, 2007.

\bibitem{qiu2016clustering}
Baozhi Qiu and Xiaofeng Cao.
\newblock Clustering boundary detection for high dimensional space based on
  space inversion and {H}opkins statistics.
\newblock {\em Knowledge-Based Systems}, 98:216--225, 2016.

\bibitem{Roweis_Saul:2000}
Sam.~T. Roweis and Lawrence.~K. Saul.
\newblock Nonlinear dimensionality reduction by locally linear embedding.
\newblock {\em Science}, 290(5500):2323--2326, 2000.

\bibitem{shen2020scalability}
Chao Shen and Hau-Tieng Wu.
\newblock Scalability and robustness of spectral embedding: landmark diffusion
  is all you need.
\newblock {\em Information and Inference: A Journal of the IMA},
  11(4):1527--1595, 2022.

\bibitem{singer2012vector}
Amit Singer and Hau-Tieng Wu.
\newblock Vector diffusion maps and the connection {Laplacian}.
\newblock {\em Communications on Pure and Applied Mathematics},
  65(8):1067--1144, 2012.

\bibitem{toth2017handbook}
Csaba~D Toth, Joseph O'Rourke, and Jacob~E Goodman.
\newblock {\em Handbook of discrete and computational geometry}.
\newblock CRC press, 2017.

\bibitem{Tyagi2013}
H.~Tyagi, E.~Vural, and P.~Frossard.
\newblock {Tangent space estimation for smooth embeddings of {R}iemannian
  manifolds}.
\newblock {\em Information and Inference}, 2(1):69--114, 2013.

\bibitem{vaughn2024diffusion}
Ryan Vaughn, Tyrus Berry, and Harbir Antil.
\newblock Diffusion maps for embedded manifolds with boundary with applications
  to {PDE}s.
\newblock {\em Applied and Computational Harmonic Analysis}, 68:101593, 2024.

\bibitem{wormell2021}
Caroline~L Wormell and Sebastian Reich.
\newblock Spectral convergence of diffusion maps: Improved error bounds and an
  alternative normalization.
\newblock {\em SIAM Journal on Numerical Analysis}, 59(3):1687--1734, 2021.

\bibitem{Wu_Wu:2017}
Hau-Tieng Wu and Nan Wu.
\newblock Think globally, fit locally under the manifold setup: Asymptotic
  analysis of locally linear embedding.
\newblock {\em Annals of Statistics}, 46(6B):3805--3837, 2018.

\bibitem{wu2022strong}
Hau-Tieng Wu and Nan Wu.
\newblock Strong uniform consistency with rates for kernel density estimators
  with general kernels on manifolds.
\newblock {\em Information and Inference: A Journal of the IMA},
  11(2):781--799, 2022.

\bibitem{wu2018locally}
Hau-Tieng Wu and Nan Wu.
\newblock When locally linear embedding hits boundary.
\newblock {\em Journal of Machine Learning Research}, 24(69):1--80, 2023.

\bibitem{Xia:2006}
C.~Xia, W.~Hsu, M.-L. Lee, and B.~C. Ooi.
\newblock {BORDER}: efficient computation of boundary points.
\newblock {\em IEEE Transactions on Knowledge and Data Engineering},
  18(3):289--303, 2006.

\bibitem{xue2009boundary}
Lixiang Xue and Baozhi Qiu.
\newblock Boundary points detection algorithm based on coefficient of
  variation.
\newblock {\em Pattern Recognition and Artificial Intelligence},
  22(5):799--802, 2009.

\end{thebibliography}
\end{document}